\documentclass{article}

\usepackage[final]{neurips_2021}

\usepackage{microtype}
\usepackage{graphicx}
\usepackage{booktabs}
\usepackage{hyperref}
\usepackage{microtype}
\usepackage{tikz}
\usetikzlibrary{arrows.meta}
\usepackage{graphicx}
\usepackage{amsmath}
\usepackage{amsfonts}
\usepackage{amssymb}
\usepackage{amsthm}
\usepackage{mathtools, nccmath}
\usepackage{ulem}
\usepackage{float}
\usepackage{bbm}
\usepackage{xcolor}
\usepackage{caption}
\usepackage{subcaption}
\usepackage{lipsum}
\usepackage{ragged2e}
\usepackage{todonotes}
\usepackage{xstring}
\usepackage{url}
\usepackage{makecell}
\usepackage{array,multirow}
\usepackage{wrapfig}

\usepackage{amsmath,amsfonts,bm}

\def\eqref#1{equation~\ref{#1}}

\def\1{\bm{1}}

\DeclareMathAlphabet{\mathsfit}{\encodingdefault}{\sfdefault}{m}{sl}
\SetMathAlphabet{\mathsfit}{bold}{\encodingdefault}{\sfdefault}{bx}{n}

\DeclareMathOperator*{\argmax}{arg\,max}

\renewcommand{\emph}[1]{{\it #1}}
\renewcommand{\ss}{\mathbf{s}}
\newcommand{\ff}{f}
\newcommand{\ssp}{\mathbf{s'}}
\newtheorem{thm}{Theorem}
\newtheorem{lem}{Lemma}

\title{Catalytic Role Of Noise And Necessity Of Inductive Biases In The Emergence Of Compositional Communication}

\author{
  Łukasz Kuciński\thanks{Corresponding author.} \\
  Polish Academy of Sciences \\
  \texttt{lkucinski@impan.pl} \\
  \And
  Tomasz Korbak\\
  University of Sussex\\
   \texttt{tomasz.korbak@gmail.com} \\
  \AND
  Paweł Kołodziej\\
  Polish Academy of Sciences\thanks{Now at Google.} \\
 \texttt{p.kolodziej@gmail.com} \\
  \And
  Piotr Miłoś \\
  Polish Academy of Sciences, \\
  University of Oxford,
  \\
  deepsense.ai\\
  \texttt{pmilos@impan.pl} \\
}

\begin{document}
\maketitle

\begin{abstract}
Communication is compositional if complex signals can be represented as a combination of simpler subparts. In this paper, we theoretically show that inductive biases on both the training framework and the data are needed to develop a compositional communication.  Moreover, we prove that compositionality spontaneously arises in the signaling games, where agents communicate over a \emph{noisy channel}.  We experimentally confirm that a range of noise levels, which depends on the model and the data, indeed promotes compositionality. Finally, we provide a comprehensive study of this dependence and report results in terms of recently studied compositionality metrics: topographical similarity, conflict count, and context independence.
\end{abstract}

\section{Introduction}
In emergent communication studies, one often considers agents who can share information about a set of objects described by the common features. Such a situation is common in multi-agent systems with partial observation (\cite{foerster_learning_2016}, \cite{lazaridou_multi-agent_2016}, \citet{jaques_social_2018}, \citet{raczaszek-leonardi_language_2018}) and it is the major theme in signaling games (\citet{fudenberg1991game}, \citet{lewis_convention:_1969}, \citet{skyrms_signals:_2010}, \citet{lazaridou_emergence_2018}).  In a signaling game, one agent (a sender) conveys information about an object to another agent (a receiver), which then has to infer the object's features. Typically, agents are rewarded if some of the features are correctly identified. During this process, the agents develop a communication protocol. A recent line of work has studied conditions under which compositionality emerges (\citet{batali1998computational,kottur_natural_2017,choi_compositional_2018,korbak2019developmentally,li_ease--teaching_2019,slowik2020exploring,slowik2020towards,guo2020inductive}).

Compositionality is a crucial feature of natural languages and it has been investigated extensively in cognitive science (see e.g. \citet{chomsky_syntactic_1957} \citet{fodor_connectionism_1988}).
It is often measured using dedicated metrics such as topographic similarity (\citet{brighton_understanding_2006,lazaridou_emergence_2018,kriegeskorte_representational_2008,bouchacourt_how_2018}), context independence \citet{bogin_emergence_2018}, conflict count \citet{kucinski2020emergence}, or positional disentanglement (\citet{chaabouni_compositionality_2020}).
In signaling games it bears a strong resemblance to the concept of disentangled representations, see  
(\citet{HigginsMPBGBML17}, \citet{kim2018disentangling}, \citet{locatello2019challenging}). 
In  machine learning context, compositionality is perceived as a generalization mechanism (\citet{lake_building_2016}) and has been used e.g. for goal composition (\citet{jiang2019language}) or knowledge transfer (\citet{li_ease--teaching_2019}). 

In this paper, we theoretically show that inductive biases on both the training framework and the data are needed for compositionality to emerge.
A similar observation has been made by \citet{kottur_natural_2017}; however, our result is more fundamental and points out a common misconception that compositionality can be learned in a purely unsupervised way. 
Such a result can be perceived as a discrete analog of \citet{locatello2019challenging}, applicable in the communication context. 
 
We then prove that adding an inductive bias in the loss function coupled with communication over a noisy channel leads to the spontaneous emergence of compositionality. 
This shows the catalytic role of noise in this process.
Intuitively, this can be attributed to the (partial) robustness of compositional language with respect to message corruption caused by a noisy channel.

We experimentally verify that a certain range of noise levels, dependent on the model and the data, promotes compositionality.  We provide a wide range of experiments that illustrate the influence of different priors. For the inductive biases in the training framework, we look into the impact of the network architecture as well as implementation and temporal variation in noise. On the data side, we study the effect of scrambling visual input or its description. We also study the generalization properties of the proposed training framework. 

\vspace{-.5em}
\section{Related work}\label{sec:related_work}

The topic of communication is actively studied in multi-agent RL, see \citet[Table 2]{Hernandez-LealK20} for a recent survey. Compositionality is often investigated in the context of signaling games (\citet{fudenberg1991game}, \citet{lewis_convention:_1969}, \citet{skyrms_signals:_2010}, \citet{lazaridou_emergence_2018}). 
Recent research has shown that strong inductive biases or grounding of communication protocols are necessary for the  protocol to be compositional
(see e.g. \citet{kottur_natural_2017}, \citet{slowik2020exploring}). The inductive bias can be imposed into the architecture of the agents or the training procedure. For instance,
\citet{das_learning_2017} place pressure on agents, to use symbols consistently across varying contexts, by a frequent reset of the agent's memory. A model-based approach was proposed by \citet{choi_compositional_2018} and \citet{bogin_emergence_2018}, who build upon the obverter algorithm 
(\citet{oliphant_learning_1997}, \citet{batali1998computational}).
\citet{slowik2020towards} explore games with hierarchical inputs and shows how agents implemented as graph convolutional networks obtain good generalization.
\citet{korbak2019developmentally} implemented the idea of template transfer \citep{barrett_self-assembling_2017} by pre-training the agents on simpler subtasks before the target task. 
\citet{918430} studied the iterative learning paradigm, 
where each generation of agents learns the language spoken by the previous generation before starting to communicate.
In the machine learning literature, this idea was explored by \citet{li_ease--teaching_2019}, \citet{cogswell_emergence_2019} and \citet{ren2020compositional} with the generation transfer typically implemented as reinitializing the weights of agents' neural networks. 
Such an approach inevitably introduces noise into the learning process. This naturally leads to a question of whether the noise itself may be a sufficient mechanism of compositionality, which we will try to address in this paper. \citet{guo2020inductive} have shown that the choice of a game has a large impact on the properties of a communication protocol emerging in that game, foreshadowing what we call grounding.

The noisy channel model of communication was famously introduced by \citet{shannon1948mathematical}. The idea of noise as a driving force in the emergence of communication was first proposed by \citet{Nowak8028}, who showed that word-level compositionality is the optimal solution to the problem of communication in a noisy environment under a particular fitness function. 
Noise is also used in deep learning, e.g. as a regularizer
(see e.g. dropout \citep{srivastava2014dropout}) 
or a mechanism allowing backpropagation through a discrete latent (see e.g. \citet{salakhutdinov2009semantic}, \citet{kaiser2018discrete}). 
Noise in the latter context was used in \citet{foerster_learning_2016} in order to learn to communicate. The authors 
observed that it is essential for successful training.

\vspace{-.5em}

\section{Noisy channel method}\label{sec:method}

We discuss the language and compositionality in Section \ref{sec:language_and_compositionality}. The impossibility result and the need for biases in emergent compositionality is the content of Section \ref{sec:grounding}. The communication task considered in this paper as well as the catalytic role of noise is described in Section \ref{sec:noisy_channel_theorem}. Theoretical results from this section hold in a somewhat idealized situation, but experiments in Section \ref{sec:experimental_results} are performed in a more realistic setup. The difference is described in Section \ref{sec:divergence_from_theory}.

\vspace{-0.5em}
\subsection{Language and compositionality} \label{sec:language_and_compositionality}

Consider a set of objects described by some features, and let a set $\mathcal F$ contain a combination of these features' values.
For example, the features could represent shape, say $\mathtt{squares}$ and $\mathtt{circles}$, and color, say $\mathtt{red}$ and $\mathtt{green}$, in which case
$\mathcal F=\{\mathtt{red\ square}, \mathtt{red\ circle}, \mathtt{green\ square}, \mathtt{green\ circle}\}$. The features could be identified with partitions of $\mathcal F$. 
More formally, each feature can be defined via equivalence relation, with features values corresponding to equivalence classes of this relation. 
In our example, the color corresponds to the partition $\{\mathtt{red\ square}, \mathtt{red\ circle}\}$ ('red') and $\{\mathtt{green\ square}, \mathtt{green\ circle}\}$ ('green'), while for the shape corresponds to the partition $\{\mathtt{red\ square}, \mathtt{green\ square}\}$ ('square') and $\{\mathtt{red\ circle}, \mathtt{green\ circle}\}$ ('circle').

We assume that objects with feature space $\mathcal F$  can be defined in  a space $\mathcal X$ and are generated by a two-stage process: first, the feature values are sampled from $f\in\mathcal F$, then an element of $\mathcal X$ is sampled according to a distribution conditioned on $f$.

In general, a language is defined as a mapping from objects to strings over some finite alphabet $\mathcal A$ (sometimes called messages), $\ell:\mathcal X\to\mathcal A^*$.
In this paper, we will study a subset of languages $\ell$, where the range of the language has a fixed length, equal to the number of features.
We say that $\ell$ is compositional with respect to a given feature if a change in $i$-th feature only impacts a corresponding $j$-th index of the message. Continuing the previous example, let $\mathcal A=\{\mathtt{a, b}\}$ and consider a language $\ell$ mapping red squares, red circles, green squares, and green circles to $\mathtt{aa, ba, ab}$, and $\mathtt{bb}$, respectively. Then $\ell$  is compositional with respect to color and shape features since the change of color only impacts the first index in the message (and analogously for the shape), see Figure \ref{fig:language_table}(\subref{subfig:lcompositional}).

Consider a permutation $\pi\colon \mathcal F\to \mathcal F$. In our running example, suppose that $\pi$ is an identity, except that it swaps $\mathtt{red\ circle}$ with $\mathtt{green\ circle}$, i.e. 
$\pi(\mathtt{red\ circle})=\mathtt{green\ circle}$ and 
$\pi(\mathtt{green\ circle})=\mathtt{red\ circle}$. 
The resulting language $\ell_\pi$ would then map red squares, red circles, green squares, and green circles to $\mathtt{aa, bb, ab}$, and $\mathtt{ba}$, respectively. This language is not compositional with respect to color and shape features, since if we  change a shape value in $\mathtt{red\ circle}$, both symbols in the message will change (from $\mathtt{bb}$ to $\mathtt{aa})$, see Figure \ref{fig:language_table}(\subref{subfig:lpinotcompositional}). 
However, $\ell_\pi$ is compositional with respect to a different set of features (shape and 'different color-shape'), see Figure \ref{fig:language_table}(\subref{subfig:lpicompositional}). Consequently, compositionality should be defined together with features, with respect to which it holds. In the next section, we show that this observation has significant implications for learning.

\newcommand\mycircle[1][]{\tikz\node[circle,minimum size=.4cm,fill=#1]{};}
\newcommand\mysquare[1][]{\tikz\node[rectangle,minimum size=.4cm,fill=#1]{};}

\newcommand\mycirclesm[1][]{\tikz\node[circle,minimum size=.3cm,fill=#1]{};}
\newcommand\mysquaresm[1][]{\tikz\node[rectangle,minimum size=.3cm,fill=#1]{};}

\begin{figure}[H]
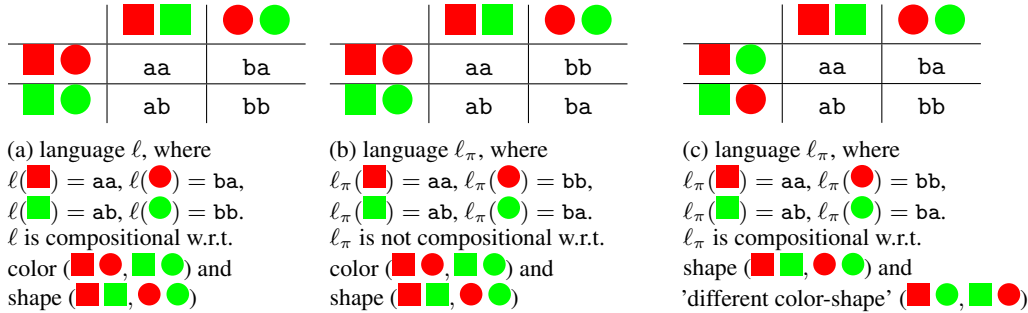

    \centering
    \begin{subfigure}{.3\textwidth}
        \begin{tabular}{l|c|c}
             & \mysquare[red] \mysquare[green] & \mycircle[red] \mycircle[green] \\ \hline
            \mysquare[red] \mycircle[red]  &  $\mathtt{aa}$ & $\mathtt{ba}$ \\ \hline
            \mysquare[green] \mycircle[green]  & $\mathtt{ab}$ & $\mathtt{bb}$ \\ 
        \end{tabular}
        \caption{language $\ell$, where\\
        $\ell(\mysquaresm[red])=\mathtt{aa}$,
        $\ell(\mycirclesm[red])=\mathtt{ba}$,\\
        $\ell(\mysquaresm[green])=\mathtt{ab}$,
        $\ell(\mycirclesm[green])=\mathtt{bb}$.\\ 
        $\ell$ is compositional w.r.t.\\ color 
        ($\mysquaresm[red]\ \mycirclesm[red], \mysquaresm[green]\ \mycirclesm[green]$) and\\
        shape ($\mysquaresm[red]\ \mysquaresm[green], \mycirclesm[red]\ \mycirclesm[green]$)}
        \label{subfig:lcompositional}
    \end{subfigure}
    \begin{subfigure}{.33\textwidth}
        \begin{tabular}{l|c|c}
             & \mysquare[red] \mysquare[green] & \mycircle[red] \mycircle[green] \\ \hline
            \mysquare[red] \mycircle[red]  &  $\mathtt{aa}$  & $\mathtt{bb}$ \\ \hline
            \mysquare[green] \mycircle[green]  & $\mathtt{ab}$ & $\mathtt{ba}$ \\ 
        \end{tabular}
        \caption{
        language $\ell_\pi$, where\\
        $\ell_\pi(\mysquaresm[red])=\mathtt{aa}$, 
        $\ell_\pi(\mycirclesm[red])=\mathtt{bb}$,\\ 
        $\ell_\pi(\mysquaresm[green])=\mathtt{ab}$, $\ell_\pi(\mycirclesm[green])=\mathtt{ba}$.\\
        $\ell_\pi$ is not compositional w.r.t.\\ color 
        ($\mysquaresm[red]\ \mycirclesm[red], \mysquaresm[green]\ \mycirclesm[green]$) and\\
        shape ($\mysquaresm[red]\ \mysquaresm[green], \mycirclesm[red]\ \mycirclesm[green]$)}
        \label{subfig:lpinotcompositional}
    \end{subfigure}
    \begin{subfigure}{.33\textwidth}
        \begin{tabular}{l|c|c}
             & \mysquare[red] \mysquare[green] & \mycircle[red] \mycircle[green] \\ \hline
            \mysquare[red] \mycircle[green]  &  $\mathtt{aa}$ & $\mathtt{ba}$ \\ \hline
            \mysquare[green] \mycircle[red]  & $\mathtt{ab}$ & $\mathtt{bb}$ \\ 
        \end{tabular}
        \caption{
        language $\ell_\pi$, where\\
        $\ell_\pi(\mysquaresm[red])=\mathtt{aa}$, 
        $\ell_\pi(\mycirclesm[red])=\mathtt{bb}$,\\ 
        $\ell_\pi(\mysquaresm[green])=\mathtt{ab}$, 
        $\ell_\pi(\mycirclesm[green])=\mathtt{ba}$.\\
        $\ell_\pi$ is compositional w.r.t.\\ 
        shape ($\mysquaresm[red]\ \mysquaresm[green], \mycirclesm[red]\ \mycirclesm[green]$) and\\
        'different color-shape' ($\mysquaresm[red]\   \mycirclesm[green], \mysquaresm[green]\   \mycirclesm[red]$)}
        \label{subfig:lpicompositional}
    \end{subfigure}
    \caption{Language, features, and compositionality.}
    \label{fig:language_table}
    \vspace{-1em}
\end{figure}

\subsection{Inductive biases and compositionality}\label{sec:grounding}
In this section, we investigate, whether compositionality can spontaneously emerge in an unsupervised fashion.  
The answer to this question is negative since the underlying features cannot be inferred from the data.

\begin{thm}\label{thm:all_is_good}
For a uniform distribution, $\mu$ on $\mathcal F$ and a permutation $\pi\colon\mathcal F\to \mathcal F$, the distribution $\mu\circ\pi^{-1}$ is also uniform.
\end{thm}

While this result is elementary to prove (see Appendix \ref{app:sec:optimality}) it has deep implications for the emergence of compositionality in the learning process.
Suppose that data represent a balanced set of all features (i.e. features are sampled from $\mu$) and recall the two-stage generation process described in Section~\ref{sec:language_and_compositionality}.
If $f\sim \mu\circ\pi^{-1}$ has the same distribution as $f\sim\mu$, then the distribution of the observed data does not depend on $\pi$.
Consequently, any emergent compositional language $\ell$ with respect to some features is not compositional with respect to other (permuted) features. We arrive at the following conclusion: 

\textit{Any learning process which hopes to achieve compositionality must involve some priors related to the data and learning framework.}

Theorem \ref{thm:all_is_good} can be viewed as a discrete version of \citet[Theorem 1]{locatello2019challenging}.
There are multiple ways of grounding the learning process, including imposing inductive biases on the agents and designing loss functions to disentangle features. We study these, together with a new mechanism: injecting noise into the communication channel. In the next section, we show that this mechanism provably achieves compositionality.

\vspace{-0.5em}
\subsection{Compositionality and communication over a noisy channel}\label{sec:noisy_channel_theorem}

Another major conceptual finding is that compositional communication spontaneously emerges when introducing a relatively simple mechanism -- a noisy channel. This is proved in Theorem \ref{thm:optimality}, provided that the loss function penalizes agents' mistakes, but also rewards for (partially) correct guesses.

Recall, that we consider a signaling game, where agents cooperate and develop a communication protocol (a language $\ell$) in order to maximize their joint reward. Here a sender observes a certain object with features $f\in\mathcal F$ and sends a message to the receiver to allow him to infer $f$. The agents are rewarded if some of the features are identified correctly, see \eqref{eq:loss2}.

The above setup is standard and now we augment it with a noisy channel. The noisy channel is located between the sender and the receiver and may scramble messages. A message $\mathbf{s}$ is transformed into a corrupted message  
$\mathbf{s'}$, by replacing each symbol,  independently and with probability $\epsilon\in (0,1)$, with a different, uniformly sampled, symbol.

For the sake of this section and Theorem \ref{thm:optimality} below, we will make the following series of assumptions. Let $\mathcal F=\mathcal F_1\times\ldots\times\mathcal F_K$, that is the feature space is factorized into $K$ features. Furthermore, assume that each feature has the same number of values. We will assume that $\mathcal X=\mathcal F$ and define a language used by the sender as a mapping $\ell: \mathcal F_1 \times\ldots \times\mathcal F_K\to \mathcal A^K$, where $\mathcal A$ is an alphabet and $|\mathcal A|=|\mathcal F_i|$. We will further assume that the 
message $\mathbf{s}=\ell(f)$ is decoded as $\ell^{-1}(\mathbf{s})$.
The corrupted message corresponding to $\ff\in\mathcal F$ is denoted by $\ell(\ff)'$ and the inferred features corresponding to this corrupted message are given by $\ff' = \ell^{-1}(\ell(\ff)')\in\mathcal F$.
The setup for this section and for our experiments (Section \ref{sec:experiment_setup} and Section \ref{sec:experimental_results}) differ, see Section \ref{sec:divergence_from_theory}.

Recall that the definition of compositionality (with respect to given features) connects the change in feature values with the change in message symbols. It is thus reasonable to look for loss functions that are somehow factorized in terms of individual symbols. 
Consider the following loss function:
\begin{equation}\label{eq:loss2}
\begin{split}
J(\ell, \ff) = \mathbb E[H(\rho(\ff', \ff))],
\end{split}
\end{equation}
where $H$ is a non-negative, strictly increasing function and $\rho$ is the Hamming distance\footnotemark{}.
\footnotetext{The Hamming distance 
between two vectors $v, w\in \mathbb R^K$
is defined as  $\rho(v, w)=\sum_{i=1}^K\mathbf{1}(v_i\ne w_i)$.}
Intuitively, $\rho$ measures the number of changed features in the corrupted message and $H$ controls the degree by which we penalize this quantity. 

\begin{thm}\label{thm:optimality}
Assume that $K\ge 2$, $\mathcal F=\mathcal F_1\times \ldots \times\mathcal F_K$, $|\mathcal A|=|\mathcal F_i|\ge 2$, and $\mathcal X=\mathcal F$. 
Suppose additionally that $\epsilon < (|\mathcal A|-1)/|\mathcal A|$. 
Then a language $\ell^*$ minimizes $J$ over all languages $\ell$ which are one-to-one mappings if and only if $\ell^*$ is compositional (with respect to features given by $\mathcal F_i$).
\end{thm}

Informally, Theorem \ref{thm:optimality} states that optimization of loss function $J$ promotes compositionality (assumption on the one-to-one property is technical). 
What makes this possible is the factorized nature of the losses and the introduction of a noisy channel. 
Interestingly, there exist other loss functions with similar properties. 
We postpone the analysis of this point and the proof of Theorem \ref{thm:optimality} to Appendix~\ref{app:sec:optimality}.

It is instructive to discuss some of Theorem \ref{thm:optimality} assumptions. 
Assumption $\mathcal F=\mathcal X$ means that $\ell$ takes semantically meaningful symbolic input. This does not cover many interesting cases, where $\mathcal X$ has representation entangled in terms of features (e.g. an image).
Furthermore, the assertion of Theorem~\ref{thm:optimality} holds for a rather wide spectrum of $\epsilon$ values (up to $0.8$ for $|\mathcal A|=5$, which is what we use in the main experiment in Section \ref{sec:experimental_results}). This stands in contrast to our experimental results, where we observe an interaction between different noise levels and compositionality, see Section \ref{sec:experimental_results}. There is no contradiction here since in Section \ref{sec:experimental_results} we study more realistic setup that extends beyond relatively strict assumptions of Theorem \ref{thm:optimality}, see Section \ref{sec:divergence_from_theory}.

\vspace{-.5em}

\section{Experiments setup}\label{sec:experiment_setup}

\subsection{Differences between experimental and theoretical setups }\label{sec:divergence_from_theory}

The setup of our experiments is more realistic (and common in this area of research) and extends beyond the assumptions of Theorem \ref{thm:optimality}. 
We assume that a dataset $\mathcal X$ contains images, consequently making features entangled in a visual representation (i.e. $\mathcal F\ne \mathcal X$). 
We also allow the alphabet size to differ from the number of feature values, $|\mathcal A|\ne |\mathcal F|$. 
Additionally, the receiver only sees the messages sent by the sender and has to learn to decode the messages.
A further gap stems from the implementation details and the use of neural networks, which may converge to suboptimal solutions. 
We do, however, assume that the feature space $\mathcal F$ is a Cartesian product (see Section \ref{sec:noisy_channel_theorem}).

\vspace{-0.5em}
\subsection{Training pipeline}\label{sec:training_pipeline}

In this section, we sketch the training pipeline and postpone the details to Appendix \ref{app:sec:training_pipeline}.  The dataset of images observed by the sender is denoted by $\mathcal D$. 
Each element of $\mathcal D$ has $K$ independent features $f_1,\ldots, f_K$ (here we consider $K=2$). 
Both the sender and the receiver are modeled as neural networks (for details see Appendix \ref{app:sec:experimental_setup}). 
The sender network takes an image from $\mathcal D$ as input and returns a distribution over the space of messages of length $L$ (here we assume $L=K=2$). We assume that conditionally on the image, the symbols in the message are independent and take values in a finite alphabet $\mathcal A_s=\{1, \ldots, d_s\}$. 
This distribution is then distorted by the noisy channel, which is a function that maps probability vectors into $d_s$-dimensional logits. Unless otherwise stated, we assume that noise is be defined as a dense layer with a specific choice of (not learnable) weights matrix and a $\log$ activation function: \useshortskip
\begin{equation}\label{eq:noise}
\mathtt{noise}(x) = \log (Wx).
\end{equation}
Here we assume that $W\in\mathbb R^{d_s\times d_s}$ is a fixed matrix, which takes a probability vector to a probability vector with positive entries. An example of such $W$ is a stochastic matrix. 
In this paper we use 
\begin{equation}\label{eq:Wuniformfinal}
W_{ij} = 
\begin{cases}
1-\varepsilon, & i = j,\\
\frac{\varepsilon}{d_s-1}, & i \ne j.
\end{cases}
\end{equation}

Alternative implementations of noise architecture are possible, see Appendix \ref{app:sec:training_pipeline}.
The noisy logits are then used to sample a message and pass it to the receiver.
To make this operation differentiable we use Gumbel-Softmax (\cite{jang_categorical_2016}) with Straight-Through mode (\cite{kaiser2018discrete}).\footnotemark{} 
\footnotetext{
We believe that Gumbel-Softmax approach is now an established choice in emergent communication research, see e.g. \citet{lee_emergent_2017}, \citet{mordatch_emergence_2017}. \citet{chaabouni_compositionality_2020} report that Gumbel-Softmax converges to similar solutions as REINFORCE but faster and is more stable.}
Upon receiving the message, the receiver outputs a distribution over possible values of the features, using an alphabet $\mathcal A_r=\{1, \ldots, d_r\}$. 
Finally, the neural networks are trained using a linear combination of cross-entropy loss for the receiver, cross-entropy loss for the sender, and $L_2$-regularization. The loss term for the sender incentives the language to be a one-to-one mapping.

\begin{wrapfigure}{R}{0.4\textwidth}
    \vspace{-3em}
    \centering
    \includegraphics[width=0.4\textwidth]{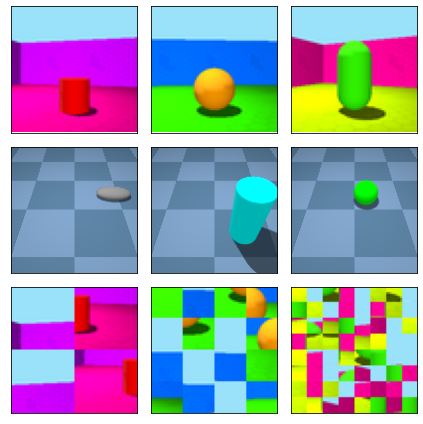}
    \caption{\small Top: shapes3d. Middle: obverter. Bottom: Scrambled shapes3d ($32$, $16$, $8$).}
    \vspace{-4em}
    \label{fig:datasets}
\end{wrapfigure}

\subsection{Datasets}\label{subsec:datasets}
We use two datasets: shapes3d  (used in \cite{3dshapes18} and included in the TensorFlow datasets package) and a dataset used by \citet{choi_compositional_2018}\footnotemark{} and \citet{korbak2019developmentally}, which we will refer to as the obverter dataset, see Figure \ref{fig:datasets}. The datasets are similar but offer different forms of visual variability.  
Shapes3d dataset includes images of 3D shapes. 
Each element is a $(64, 64, 3)$ RGB image, and is characterized by multiple features, such as shape or object hue. We choose images with values for both features ranging in $\{0, 1, 2, 3\}$. 
The obverter dataset contains images of four shapes (box, cylinder, ellipsoid, sphere) in four colors (blue, cyan, gray, green). Each image has dimensions $(128, 128, 3)$, and we use 1000 images for each shape--color pair. 
For details see Appendix \ref{app:datasets}. 
\footnotetext{The dataset is available at \url{https://github.com/benbogin/obverter}. We used the code provided in the repository to generate $1000$ images for each color-shape pair.}

\subsection{Training details}
\vspace{-0.5em}
The sender and the receiver are implemented as feed-forward neural networks. To ensure the diversity of random initializations we run each experiment with 100 random seeds. For each seed, there were 100 evaluation runs, once every $2000$ network updates. We report metrics averaged over the last $20$ evaluation runs. For details on architecture and hyperparameters' choice, see Appendix~\ref{app:sec:experimental_setup}.

\subsection{Compositionality measures}\label{sec:compositionality_metrics}

We measure compositionality in terms of four metrics used in emergent communication literature: topographic similarity, conflict count, context independence, and positional disentanglement.  
For all metrics, the higher values the better, except for conflict count, for which the reverse is true. 
The results presented in Section \ref{sec:experimental_results} indicate that the metrics agree on the assessment of our results. 
Parallel to compositionality metrics, we also report accuracy, which we often refer to as \emph{acc}.

\textbf{Topographic similarity} Topographic similarity \citep{brighton_understanding_2006,lazaridou_emergence_2018}, or \emph{topo} for short, is a popular measure of structural similarity between messages and features. Let $L_f : \mathcal F \times \mathcal F \to \mathbb{R}_+$ be a distance over features and $L_{m} : \mathcal A_s^* \times \mathcal A_s^* \to \mathbb{R}_+$ be a distance over messages. The topographical similarity is the Spearman $\rho$ correlation of $L_f$ and $L_m$ measured over a joint uniform distribution over features and symbols. We choose $L_{m}$ to be the \citet{levenshtein1966binary} distance and treat features $f \in \mathcal{F}$ as ordered pairs or features so we can choose $L_f$ to be the Hamming distance. We use topo as the main metric.

Figure \ref{fig:topo_bound} shows the expected value of topographic similarity for a random bijective language with a message length equal to $2$. For $5$-symbol alphabet, this value does not exceed $0.2$ which sets a point of reference for compositionality results. 
For derivation see Appendix \ref{app:sec:topo}.

\begin{wrapfigure}{R}{0.4\textwidth}
    \vspace{-1em}
    \centering
    \includegraphics[width=.4\textwidth]{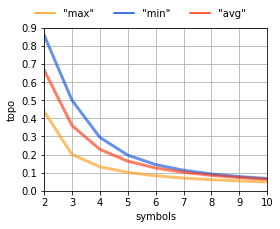}
    \caption{\small Expected value of topographic similarity for a random bijective language with message length $2$, as a function of the alphabet size. "min", "max", and "avg" stand for different ways of computing ranks.} 
    \label{fig:topo_bound}
    \vspace{-2.5em}
\end{wrapfigure}

\textbf{Conflict count}
Let $\phi:\{1,\ldots,K\}\to \{1, \ldots, K\}$ be a permutation. 
The principal meaning of a symbol $s$ at position $j$ is defined as $\mathtt{m}( s,j; \phi)=\argmax_{v}\mathtt{count}(s, j, v;\phi)$, where
$\mathtt{count}(s, j, v; \phi)$ is defined as $\sum_{(\mathtt{img},f)\in\mathcal D} \mathbf{1}(\ell(\mathtt{img})_j=s, f_{\phi(j)}=v)$, $v$ runs over all values of all features, 
and  ties in $\argmax$ are broken arbitrarily. Then, conflict count metric, \emph{conf} for short, is defined as $\mathtt{conf} = \min_\phi \sum_{s,j} \mathtt{score}(s,j; \phi)$,
where $\mathtt{score}(s,j;\phi) = \sum_{v\neq \mathtt{m}(s,j;\phi)} \mathtt{count}(s,j, v; \phi)$. Intuitively, $\mathtt{score}$ measures how many times the feature assigned to a symbol $s$ at a position $j$ diverts from its principal meaning $\mathtt{m}( s,j; \phi)$. $\mathtt{conf}$ sums these errors and takes $\min$ over possible orderings $\phi$. This metric was introduced in  \citet{kucinski2020emergence}.

\textbf{Context independence} Context independence  (\citet{bogin_emergence_2018}), abbreviated here as \emph{cont}, measures the alignment between symbols forming a message and features $f_1, \dots, f_K$. By $p(s|f)$, we mean the probability that the sender maps a feature $f$ to a message containing symbol $s \in \Sigma$. We define the inverse probability $p(f|s)$ similarly. Finally, we define $s^f = \argmax_s p(f \vert s)$; $s^f$ is the symbol most often sent in presence of a feature $f$. Then, context independence is $\mathbb{E} ( p(v^k \vert k) \cdot p(k \vert v^k))$; the expectation is taken with respect to the joint uniform distribution over features and symbols. 

\textbf{Positional disentanglement} Let $s_j$ denote the $j$-th symbol of a message $f(d)$, and $c_1^j$ the feature with the highest mutual information with $s_j$, and $c_2^j$ with the second highest mutual information: $c_1^j = \argmax_c \mathcal{I}(s_j; c)$, $c_2^j = \argmax_{c \neq c_1^j} \mathcal{I}(s_j; c)$
where $\mathcal{I}(\cdot; \cdot)$ is mutual information and $c$ is a feature value. Then, positional disentanglement \citep{chaabouni_compositionality_2020} is defined as
$\frac{1}{L} \sum_{j=1}^L \left(\mathcal{I}(s_j; c^j_1) - \mathcal{I}(s_j; c^j_2)\right)/\mathcal{H}(s_j)$, where $L$ is the maximum message length and $\mathcal{H}(s_j)$ is entropy over the distribution of symbols at $j$-th place in messages for each feature. We ignore positions with zero entropy.\footnote{Positional disentanglement is related to residual entropy proposed by \citet{resnick_capacity_2020}. \citet{chaabouni_compositionality_2020} also proposed 
bag-of-words disentanglement, which assumes order-invariance of messages. Due to our architecture choice, this assumption is not met, hence we decided not to report this metric.}
We will call this measure \emph{pos}, for short.

\section{Experiment results}\label{sec:experimental_results}

In this section, we study how different inductive biases for the model and the data influence compositionality. 
The main experiment is presented in Section~\ref{sec:baseline_experiment} and the experiments in Section~\ref{sec:model_biases}-\ref{sec:generalization} are its variation. 
In the main experiment we assume that the message length is $K=2$, $|\mathcal A_s|=5$, and $|\mathcal A_r|=8$ (see Section \ref{sec:training_pipeline}). Notice that both $\mathcal A_s$ and $\mathcal A_r$ differ from one another and from the number of possible feature values (which equals 4).
It turns out that the qualitative results are similar for both the shapes3d and obverter datasets, and in the interest of brevity we only report the results for both in the main experiment. Supplementary material for this section can be found in Appendix~\ref{sec:app:results}.

Our main findings include the fact that the noise indeed catalyzes the emergence of compositionality.  There is an interesting dependence on the noise level: too high noise may impede learning, while too small noise vanishes within other sources of noise. Importantly, this phenomenon appears consistently across different biases. Using two head output of the network is the strongest bias toward compositionality (amongst studied). Finally, we show that compositionality can generalize to unseen cases when fine-tuning is allowed.

\begin{figure}
    \centering
    \includegraphics[width=\textwidth]{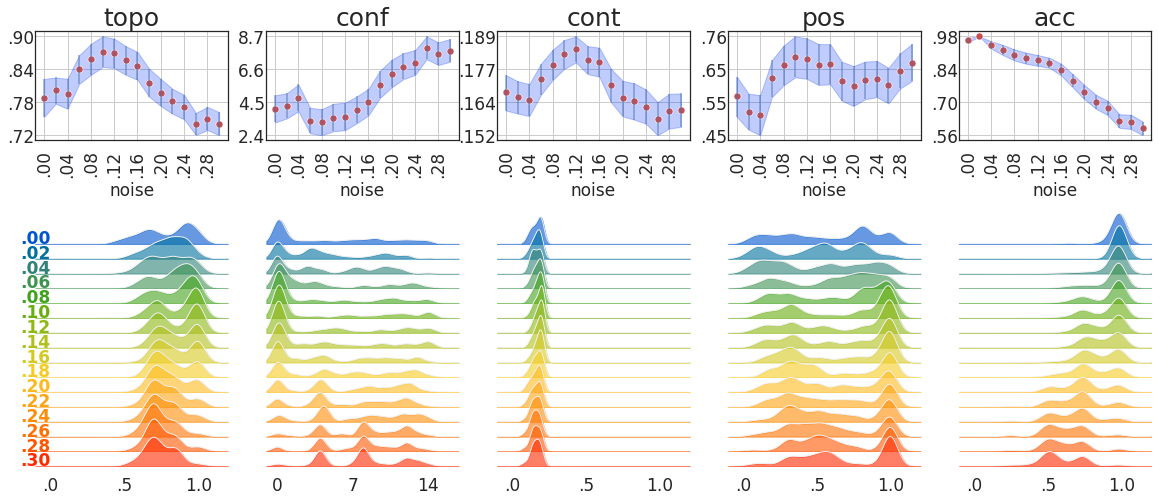}
    \caption{\small Results of the main experiment on shapes3d dataset. Top panel: average value of metrics for various noise levels. The shaded area corresponds to bootstrapped $95\%$-confidence intervals. Bottom panel: kernel density estimators for metrics and noise levels across seeds. Here \emph{topo} stands for topographic similarity, \emph{conf} for conflict count, \emph{cont} for context independence, \emph{pos} for positional disentanglement, and \emph{acc} for accuracy.}
    \label{fig:avg_with_cis_romantic_mcnulty}
	\vspace{-2em}
\end{figure}

\vspace{-.5em}
\subsection{Main experiment: emergence of compositionaliy}
\label{sec:baseline_experiment}

The results, presented in Figure \ref{fig:avg_with_cis_romantic_mcnulty},  illustrate how noise catalyzes the emergence of compositionality.
The top panel of Figure \ref{fig:avg_with_cis_romantic_mcnulty} shows the important patterns for metrics: they improve until an extremal point is reached, and decline afterward. 
Topographic similarity achieves extremum for noise level $0.1$, reaching value $0.87$, which is a significant improvement upon $0.79$ for the lack of noise.
The accuracy drops down with an increase of the noise level, as expected, however the speed of the decline increases. This shows that there is an interesting compositionality-accuracy trade-off.
The bottom panel of Figure \ref{fig:avg_with_cis_romantic_mcnulty} complements the overall picture with a visualization of metrics' distribution. We see interesting dynamics in topographic similarity distribution with respect to change in the noise levels. Namely, it starts by accumulating mass at the higher spectrum of its values, reaching a peak for noise $0.1$, after which it transitions to a bimodal distribution, finally shifting its mass more towards the mediocre end of the spectrum.
An interesting observation is that the results for experiments with achieved high accuracy, are not only better but also the effect of noise is more pronounced (for instance there are 47 seeds with accuracy exceeding $0.9$, for which noise $0.16$ yields $0.96$ topo). For the sake of brevity, we defer the discussion of this phenomenon to Appendix \ref{sec:nostalgic_lovelace}.
The accuracy 
undergoes a similar transformation as a topographic similarity. 
The detailed numerical analysis, as well as corresponding results for the obverter dataset, can be found in Appendix \ref{sec:nostalgic_lovelace}.

\begin{figure}
\centering
\subfloat[symbols 4x4\label{fig:grid:4x4}]{\includegraphics[width=0.25\linewidth]{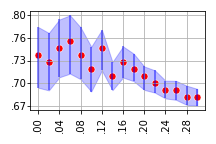}}
\hfil
\subfloat[symbols 8x8\label{fig:grid:8x8}]{\includegraphics[width=0.25\linewidth]{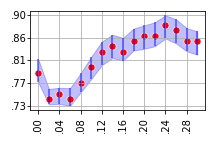}}
\hfil
\subfloat[variable noise $0.1$ \label{fig:grid:variable_noise_0.1}]{\includegraphics[width=0.25\linewidth]{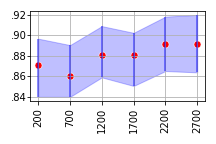}}
\hfil
\subfloat[variable noise $0.15$ \label{fig:grid:variable_noise_0.15}]{\includegraphics[width=0.25\linewidth]{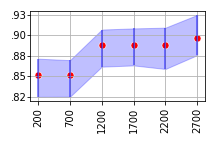}}

\subfloat[bigger CNN\label{fig:grid:big_cnn}]{\includegraphics[width=0.25\linewidth]{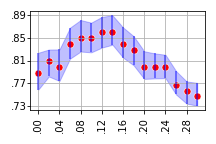}}
\hfil
\subfloat[additional dense layer\label{fig:grid:big_dense}]{\includegraphics[width=0.25\linewidth]{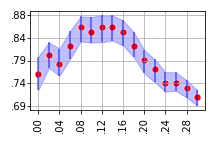}}
\hfil
\subfloat[alternative noise \label{fig:grid:alternative_noise}]{\includegraphics[width=0.25\linewidth]{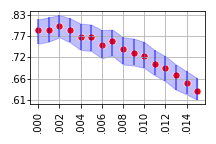}}
\hfil
\subfloat[message len 3\label{fig:grid:message3}]{\includegraphics[width=0.25\linewidth]{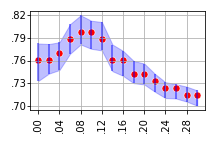}}

\subfloat[scramble 32\label{fig:grid:scrambl32}]{\includegraphics[width=0.25\linewidth]{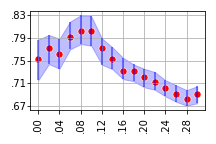}}
\hfil
\subfloat[scramble 16\label{fig:grid:scrambl16}]{\includegraphics[width=0.25\linewidth]{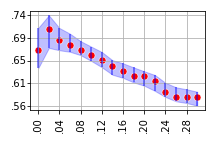}}
\hfil
\subfloat[scramble 8\label{fig:grid:scrambl8}]{\includegraphics[width=0.25\linewidth]{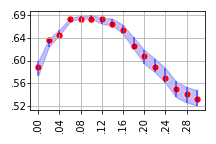}}
\hfil
\subfloat[scrambled labels\label{fig:grid:scramble_labels}]{\includegraphics[width=0.25\linewidth]{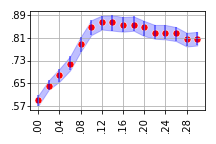}}

\caption{\small Topographic similarity for experiments in Sections \ref{sec:model_biases} and Section \ref{sec:data_biases}. The shaded areas correspond to bootstrapped $95\%$-confidence intervals for average topo.}
\label{fig:grid}
\vspace{-2em}
\end{figure}

\vspace{-0.5em}
\subsection{Influence of model inductive bias}\label{sec:model_biases}
\vspace{-.5em}

\textbf{Different number of symbols}
Here we study the impact of a different number of symbols on compositionality. In our experiments we used the communication channel with $|\mathcal A_s|=5$, giving a total of $25 = (5\times 5)$ possible messages. This is slightly redundant since only $16 = (4\times 4)$ is required, so this is the first case that we study here. 
It turns out, that allowing for only $16$ messages makes the training less stable. 
For topographic similarity, see Figure~\ref{fig:grid}(\subref*{fig:grid:4x4}), small to medium values of noise exhibit wide confidence intervals and it is statistically hard to distinguish between the metric values (this might be attributed to a bimodal distribution of topo in this noise range, see Figure~\ref{fig:ridge_4x4} in Appendix \ref{sec:number_of_symbols_app}). 
For larger values of noise (greater than $0.16$), topo starts to visibly decline. 
As the second experiment, we considered the total of $64 = (8\times 8)$ messages.
Interestingly, the topographic similarity values for the small noise regime (up to $0.08$) do not improve over the baseline value ($0.79$), see Figure~\ref{fig:grid}(\subref*{fig:grid:8x8}). This behavior changes for medium to large values of noise, where we can observe a visible increase in topo, peaking at $0.88$ with a noise level of $0.24$.
For details see Appendix \ref{sec:number_of_symbols_app}.

\textbf{Variable noise}
Understanding what happens under varying noise is an interesting and subtle problem. 
It would most probably arise in more complex situations when a communication channel is a part of a bigger system. 
Additionally, it might be similar to a phenomenon observed in supervised learning, indicating that training can benefit from learning rate warmup (see e.g.   \citet{goyal2017accurate}, \citet{frankle2020linear}). 
In this paragraph, we discuss a simple experiment with increasing noise and leave a more nuanced study for further work.
More precisely, in the initial stage of training the noise is kept at the level $\epsilon_0$, and after $T\in\{200, 700, 1200, 1700, 2200, 2700\}$ network updates it is switched to a different value, $\epsilon_T$, and kept there for the rest of the training.
The results for topographic similarity are presented in
Figures~\ref{fig:grid}(\subref*{fig:grid:variable_noise_0.1})-(\subref{fig:grid:variable_noise_0.15}), where $\epsilon_0=0$ and $\epsilon_T = 0.1$ and $\epsilon_T=0.15$, respectively. In Figure~\ref{fig:grid}(\subref{fig:grid:variable_noise_0.15}) we see that topo increases from $0.85$ for $T=200$, to $0.9$ for $T=2700$. The effect in Figure~\ref{fig:grid}(\subref{fig:grid:variable_noise_0.1}) is weaker and the variance in the results is quite high.
The details can be found in Appendix~\ref{app:sec:variable_noise_experiments}.

\textbf{Sensitivity with respect to small architecture changes}
This set of experiments aims to check the impact of changing the parameters of CNN and the dense layers in the agents' network. 
In the former experiment, we change the number of filters in the sender's CNN architecture from two layers with $8$ filters, to two layers with $16$ filters. This results in the slight change of topographic similarity profile, see Figure~\ref{fig:grid}(\subref*{fig:grid:big_cnn}), with the highest average value  
of $0.86$ for noise range $0.14$ significantly outperforming the zero-noise case ($0.79$). 
For the second experiment, we added a dense layer (with 64 neurons) to the receiver's architecture, see Figure~\ref{fig:grid}(\subref{fig:grid:big_dense}). This again shows improvement of compositionality due to noise, although the noise in the range $[0.08, 0.16]$ performs roughly the same (see also Appendix \ref{sec:architecture_experiments}).

\textbf{Sensitivity to noisy channel implementation}
Here we consider an alternative noisy channel implementation, where the noise permutes the sampled symbols, as opposed to distorting the distribution of the symbols (see Appendix \ref{app:sec:training_pipeline}). 
It turns out that with this change, the scale of noise where the interesting things happen changes as well, see Figure \ref{fig:grid}(\subref*{fig:grid:alternative_noise}). Namely, when compared with the main experiment, 
noise values are in the range smaller by the order of magnitude (we report results for noise levels in $\{0.000, 0.001, \ldots, 0.015\}$.
The results suggest that the small values of noise can help, while the larger noise levels lead to a decline in compositionality, see also Appendix \ref{sec:alternative_noise_design_experiments}.

\textbf{Longer message}
In this section, we discuss the experiment with an additional feature (with floor color acting as the third feature) and a message of length $3$. The result for topographic similarity can be seen in  Figure \ref{fig:grid}(\subref*{fig:grid:message3}), see also Appendix \ref{sec:longer_message}). The overall topographic similarity level is lower than in the main experiment, however, the distinctive peak is visible, here for noise level $0.08$. 

\subsection{Data inductive biases}\label{sec:data_biases}

\textbf{Visual priors} 
This experiment was intended to check how much the CNN-backed input is relevant in the compositionality context. We could conjecture that CNN may facilitate shape recognition and therefore be the driving force in the emergence of languages compositional with respect to the canonical shape, color split. To check this we impair the prior by scrambling images, as depicted in the bottom panel Figure~\ref{fig:datasets}. 
An image is scrambled by splitting it into $(64/x)^2$ disjoint tiles of height and width equal to $x$, and randomly reshuffling them.
This procedure significantly distorts the accuracy profile of the method (see Appendix \ref{sec:scrambled_images_appendix}). 
In particular, each transition from coarser to finer tiles, the accuracy decreases significantly: for zero-noise it drops from $0.97$ for no tiling, to $0.95$ for $x=32$, to $0.79$ for $x=16$ and $0.53$ for $x=8$.
The overall compositionality metrics decrease as well, but the characteristic peak for some positive noise level is still present, see Figure~\ref{fig:grid}(\subref{fig:grid:scrambl32})-(\subref{fig:grid:scrambl8}).
Having said that, the metrics incur a significant boost, when computed for a subset of experiments with high accuracy. We conjecture that the explanation is that the CNN prior is indeed relevant but is not the only one (the output considered in the next section is another). 

\textbf{Scrambled labels}
In this experiment, we aimed to understand how much the overall architecture output is important in the emergence of compositionality. 
The receiver's network has a two-headed output and the training framework uses a factorized loss function. 
In the standard setting, we compare the heads' outputs with 'color' and 'shape' respectively, therefore we reflect human priors from the data. In this experiment, we distort this setting, 
permuting the set of $(\text{color}, \text{shape})$ and factorizing them into new labels (see Appendix \ref{sec:scrambled_features_appendix}). We use a random permutation, so the new labels are abstract and correspond to some joint color-shape concepts. 
In our experiments, we show that languages, which emerge are compositional with respect to these new concepts, see Figure \ref{fig:grid}(\subref*{fig:grid:scramble_labels}). Consequently, they are \emph{not} compositional in the standard color-shape framework. This is in line with the claims of Section~\ref{sec:grounding} and further highlights that the output inductive bias is essential. 

\begin{figure}
    \centering
    \includegraphics[width=\textwidth]{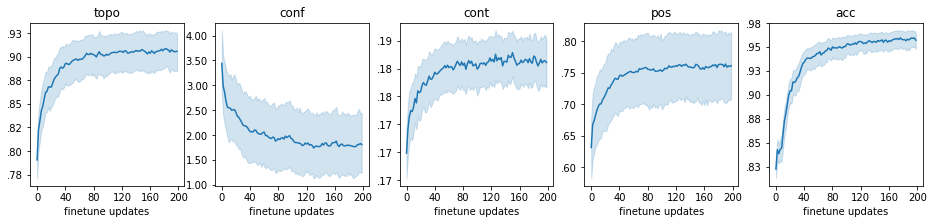}
    \caption{\small Fine-tuning for a baseline agent trained with noise $0.1$ and with several features removed from the training set (the diagonal from the shape-by-color matrix). On the x-axis is the number of finetuning updates.}
    \label{fig:finetuning}  
    \vspace{-2em}
\end{figure}

\begin{wrapfigure}{R}{0.4\textwidth}
    \vspace{-1.5em}
    \centering
    \includegraphics[width=.4\textwidth]{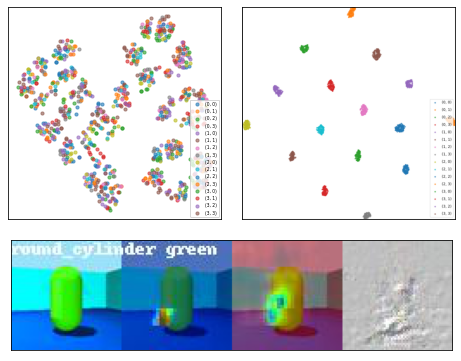}
    \caption{\small t-SNE visualization (with perplexity $50$ and $10000$ steps). Top left: shapes3d, top right: features of the last dense layer. Bottom: occlusion and saliency maps. }
    \label{fig:tsne} 
    \vspace{-1em}
\end{wrapfigure}

\vspace{-0.7em}
\subsection{Generalization}\label{sec:generalization}
\vspace{-0.5em}

\textbf{Network features}
In this paragraph, we present visualizations of the sender network, to gain some insights into what the network is doing, and whether we can observe any obvious overfit. 
In the upper panel of Figure~\ref{fig:tsne}, a t-SNE (\citet{van2008visualizing}) visualization is shown for both the raw shapes3d dataset and the last dense layer of a trained agent with noise $0.1$. 
We observe that the network successfully disentangles features of the data, which in the raw dataset appears to be entangled. This could be a good  sign in the context of generalization\footnotemark{}. 
\footnotetext{
It is known that t-SNE does not necessarily represent cluster sizes, distances, and respective positions (see \citet{wattenberg2016how}). Hence, t-SNE cannot be expected to serve as a compositionality metric.
}
Furthermore, Figure~\ref{fig:tsne} (bottom panel) shows occlusion and saliency maps for a sample image from shapes3d dataset. The network seems to pay attention to parts of the image that could be relevant for the task.

\textbf{Zero-shot and fine-tuning}
In this experiment, we study the generalization properties of the training protocol. 
We trained the baseline agent with noise $0.1$  with several features removed from the training set (a diagonal from the shape-by-color matrix) and used the removed set to analyze the out-of-distribution performance. In the zero-shot setup, the messages on these observations were not in line with the compositional structure acquired on the training set, which is reflected by rather unimpressive outcomes for each metric (initial values for each metric presented in Figure \ref{fig:finetuning}).
This may suggest that for a zero-shot generalization stronger biases might be required (see e.g. \citet{slowik2020towards,slowik2020exploring}). 
However, after a few network updates on the whole dataset (without the removed diagonal) 
we observed a quite significant increase in compositionality metrics, see Figure \ref{fig:finetuning}. The highest improvement can be seen roughly in the first $50$ additional network updates.

\vspace{-.5em}

\section{Limitations of the method }\label{sec:limitations}
Here we summarize a few limitations of the method, that we hope to overcome in future research. 

\textbf{Theory} Theorem \ref{thm:optimality} uses several restrictive assumptions, e.g. fixed message length or alphabet size equal to the range of feature values. In a general setting, it could be the case that noise would promote non-compositional, error-correcting, communication protocols. Stating the general conditions for the emergence of compositionality is an open problem.

\textbf{Choice of datasets} 
Confirmation on non-synthetic datasets is needed to fully underpin our method.

\textbf{Communication protocol} We assume a simple communication protocol with messages of length two and two features, each taking four values (except for experiments in Section \ref{sec:model_biases}). 

\textbf{Compositionality model} 
We assume a rather simple compositionality model based of independent features. Exploring non-trivial compositionality \citep{steinert-threlkeld_towards_2020,korbak2020measuring} might be an important conceptual development.

\textbf{Architecture} Our architecture might implicitly exhibit some unknown biases that increase compositionality. On the other hand, the architecture might be too simple to achieve strong results in tasks such as zero- or few- shot generalization. 

\textbf{Multi-agent interactions} 
We used supervised training setup, it would be natural to test the influence of noise in reinforcement learning scenarios.

\vspace{-.5em}
\section{Conclusions}\label{sec:conclusions}

In this paper, we theoretically show that inductive biases on both the training framework and the data are needed for the compositionality to emerge spontaneously in signaling games. 
We then formulate inductive biases in the loss function and prove that they are sufficient to achieve compositionality when coupled with communication over a noisy channel.
Consequently, we highlight the catalytic role of noise in the emergence of compositionality.
We perform a series of experiments in order to understand different aspects of the proposed framework better. We empirically validate that, indeed, a certain range of noise levels, dependent on the model and the data, promotes compositionality. 
Our work is foundational research and does not lead to any direct negative applications.

\begin{ack}
The work of Piotr Miłoś was supported by the Polish National Science Center grant UMO-2017/26/E/ST6/00622. The work of Tomasz Korbak was supported by the Leverhulme Doctoral Scholarship. We gratefully acknowledge Polish high-performance computing infrastructure PLGrid (HPC Centers: ACK Cyfronet AGH, PCSS) for providing computer facilities and support within computational grant no. PLG/2019/012498. Our experiments were managed using \url{https://neptune.ai}. We would like to thank the Neptune team for providing us access to the team version and technical support.
\end{ack}

\typeout{}
\bibliography{bibliography.bib}
\bibliographystyle{apalike}

\appendix

\clearpage

\clearpage
\appendix

\section{Datasets}\label{app:datasets}

\subsection{Shapes3d}
Shapes3d is a dataset (see \citet{3dshapes18} and the Tensorflow Datasets package) consisting of $64\times 64\times 3$ RGB images of objects having six independent features (floor color, wall color, object color, scale, shape, and orientation), see Figure \ref{fig:datasets}. In this paper, we use four shapes (cube, cylinder, sphere, and rounded cylinder) and four object colors (red, orange, yellow, and green), totaling $192000$ images.

\subsection{Obverter}\label{app:sec:obverter}
The obverter dataset (\citet{bogin_emergence_2018}) is available  at the following address:  \url{https://github.com/benbogin/obverter}. 
The original dataset consists of $128\times 128\times 3$ RGB images of objects having eight colors and five shapes. In this paper, we used four colors (blue, cyan, gray, green) and four shapes (box, cylinder, ellipsoid, sphere), see Figure \ref{fig:datasets}. 
We have generated 1000 samples for each color-shape combination using a generation script available at the dataset repository, hence the total number of images is $160000$. Since the qualitative results were similar, and in the interest of brevity, we report results for the obverter only for the main experiment (Appendix \ref{sec:nostalgic_lovelace}).

\section{Experimental setup}\label{app:sec:experimental_setup}

\subsection{Architecture}\label{app:sec:architecture}

The network consists of three main parts: the sender, the receiver, and the noisy discrete channel between them see Figure \ref{fig:nn_arch}.
The sender network consists of two convolutional layers (with 8 filters, kernel $3\times 3$, stride 1, and elu activation function), each coupled with a $2\times 2$ max pool layer with stride 2.
The last max pool layer's output is passed through two dense layers (with 64 neurons and elu activation) and a linear classifier with softmax for each symbol.
The noisy channel layer consists of a dense layer with $|A_s|$ neurons, a fixed weights matrix, and a $\log$ activation function. This is followed by a Gumbel softmax layer.
The receiver network takes two encoded symbols as input and concatenates them to obtain one input vector $s$.
Consequently, $s$ and $1-s$ are passed to the dense layers, the result is summed up and processed by the elu activation function and two dense layers (similarly to \citet{kaiser2018discrete}).
There are two linear classifiers with a softmax layer at the output: one for the shape and one for the color. Each dense layer in the receiver has 64 neurons.

\subsection{Hyperparameters and training}

For training, we used 
$\lambda_{KL}=0.01$, $\lambda_{l_2}=0.0003$, an Adam optimizer (with $\beta_1=0.9$, $\beta_2=0.999$), 
learning rate $0.0001$, and a batch size of 64. The same set of hyperparameters was used for all the experiments. 
The hyperparameters were chosen on the original obverter dataset available at the repository referenced in Appendix \ref{app:sec:obverter}.
We used a grid search over parameters: learning rate ($1e-2, 1e-3, 1e-4, 3e-4$), 
kl regularization coefficient ($1e-1, 1e-2, 2e-2, 3e-2, 1e-3, 3e-3, 5e-3, 1e-4$), 
the number of CNN's filters ($8, 16$), the CNN's filter sizes ($3\times 3, 5\times 5$),
the sender's embedding size ($32, 64$), 
$l_2$ regularizer weight ($1e-2, 1e-3, 3e-3, 1e-4, 3e-4, 1e-6$), and the number of neurons in receiver's dense layers ($32, 64$).

Each experiment was run on $100$ seeds and had $200000$ network updates. The dataset was split into the training set ($90\%$ of the total) and on the test set (remaining $10\%$ of the dataset). The evaluation was done every $2000$ updates on the test set.

\begin{figure*}[h]
    \centering
    \includegraphics[width=\textwidth]{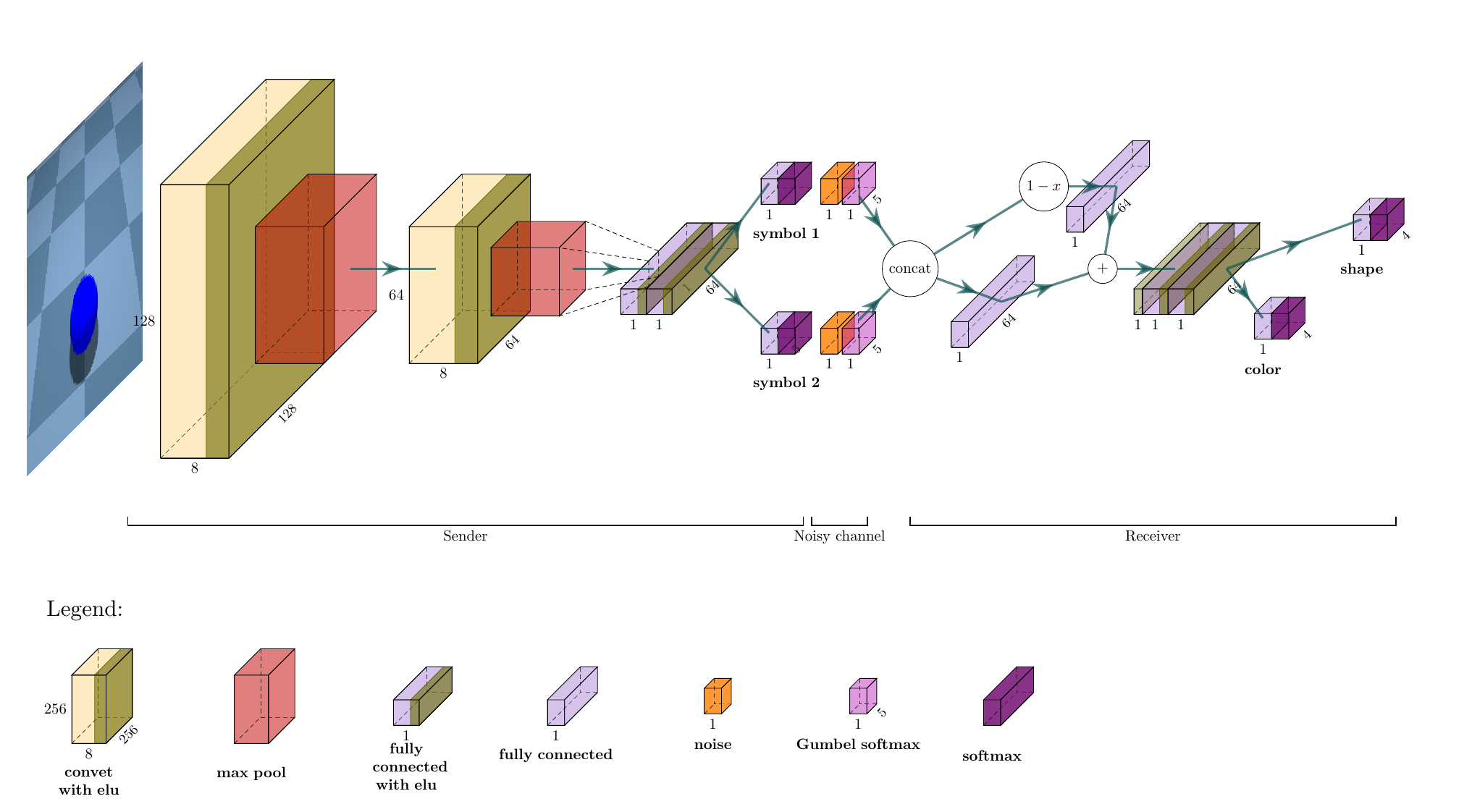}
    \caption{The architecture of the neural network. Here $|A_s|=5$.}
    \label{fig:nn_arch}
\end{figure*}

\subsection{Infrastructure used} \label{sec:infrastructure_used} 

The typical configuration of a computational node used in our experiments was: the Intel Xeon E5-2697 $2.60$GHz processor with $128$GB memory. On a single node, we ran $4$ or $5$ experiments at the same time. A single experiment (single seed) takes about $5$ hours. We did not use GPUs; we found that with the relatively small size of the network (see Appendix \ref{app:sec:architecture})
it offers only slight wall-time improvement while generating substantial additional costs.

\section{Training pipeline}\label{app:sec:training_pipeline}

\subsection{Detailed derivation}

The dataset of images observed by the sender is denoted by $\mathcal D$. 
Each element of $\mathcal D$ has $K$ independent features $f_1,\ldots, f_K$ (here we consider $K=2$). 
Both the sender and the receiver are modeled as neural networks ($s_\theta$ and $r_\psi$, respectively; for details see Appendix \ref{app:sec:experimental_setup}). 
The sender network takes an image from $\mathcal D$ as input and returns a distribution over the space of messages of length $L$ (here we assume $L=K=2$). We assume that conditionally on the image, the symbols in the message are independent and take values in a finite alphabet $\mathcal A_s=\{1, \ldots, d_s\}$. 
We furthermore assume, that features are enumerated with $A_f=\{1,\ldots, d_f\}$ and the receiver's alphabet is $\mathcal A_r=\{1,\ldots, d_r\}$.
Formally, $s_\theta(\mathtt{img}) = (s^i_{\theta}(\mathtt{img}))_{i=1}^ K$, 
where 
$s_\theta^i(\mathtt{img})=(s_{j,\theta}^i(\mathtt{img}))_{j=1}^{d_s}\in \mathcal P(A_s)$\footnotemark{} represents the probability distribution corresponding to the $i$th symbol.
\footnotetext{$\mathcal P(A)=\{p\in\mathbb R^{|A|}:p_i\ge 0, \sum_{i\in A}p_{i}=1\}$.}
Define a function $\mathtt{noise}\colon \mathcal P(A_s)\to \mathbb R^{d_s}$ as follows: 
\begin{equation}\label{app:eq:noise}
\mathtt{noise}(x) = \log (Wx), 
\end{equation}
where $W\in\mathbb R^{d_s\times d_s}$ is a fixed matrix, such that $Wx>0$ and   $Wx\in\mathcal P(A_s)$, for any $x\in\mathcal P(A_s)$\footnotemark{}. 
\footnotetext{
We could also define $\mathtt{noise}$ for all $x\in\mathbb R^m$, for some $m$, by first applying $\mathtt{softmax}$ to $x$, and then using \eqref{app:eq:noise}.}
The second condition on $W$ is satisfied, for instance, by a family of stochastic matrices; several examples are also given at the end of this section.  
In this paper, we use $W$ defined as
\begin{equation*}
W_{ij} = 
\begin{cases}
1-\varepsilon, & i = j,\\
\frac{\varepsilon}{d_s-1}, & i \ne j.
\end{cases}
\end{equation*}
Let $\widehat{s}_\theta^i(\mathtt{img})$ denote the 
logits of $i$th symbol distribution which passes through the noisy channel: 
\[
\widehat{s}_\theta^i(\mathtt{img}) = \mathtt{noise}(s_\theta^i(\mathtt{img})).
\]
Suppose further that $g^i=(g_1^i,\ldots,g_{d_s}^i)$ is a vector of i.i.d. $\text{Gumbel}(0,1)$ random variables and  
define the following functions:
\begin{align*}
\mathtt{gumbel\_sample}(x; g) &= \argmax_i(x_i+g_i), \\
\mathtt{gumbel\_softmax}(x;\tau, g)_i &= \frac{\text{exp}((x_i+g_i)/\tau)}{\sum_{j=1}^k \text{exp}((x_j+g_j)/\tau)}. 
\end{align*}
Let
\begin{align*}
\widehat{\mathfrak{m}}_i&= \mathtt{gumbel\_softmax}(\widehat{s}_\theta^i(\mathtt{img}); \tau, g^i)\in\mathbb R^{d_s}.
\end{align*}
The receiver neural network is denoted as 
$r_\psi(\mathfrak{m}) = (r^i_{\psi}(\mathfrak{m}))_{i=1}^ K$, 
where 
$r_\psi^i(\mathfrak{m})=(r_{j,\psi}^i(\mathfrak{m}))_{j=1}^d\in \mathcal P(A_r)$ represents the probability distribution on $A_r$, corresponding to the $i$th feature. 

In the Straight-Through mode (see \citet{jang_categorical_2016}), 
$r_\psi$ takes $\widehat{\mathfrak{m}}$ as input half of the time, and the remaining half of the time, it takes  $\widetilde{\mathfrak{m}}$. Here 
\begin{align*}
\widetilde{\mathfrak{m}} &= \mathtt{stop\_gradient}(\overline{\mathfrak{m}}-\widehat{\mathfrak{m}})+\widehat{\mathfrak{m}},\\
\overline{\mathfrak{m}}_i&= \mathtt{one\_hot}(\omega_i)\in\mathbb R^{d_s},\\
\omega_i &= \mathtt{gumbel\_sample}(s^i_\theta(\mathtt{img}); g^i) \in A^{d_s}, 
\end{align*}
i.e. $(\omega_1, \ldots, \omega_K)$ is a sampled noisy message.
The neural networks are trained using the following loss function:
\[
\mathcal L = \mathcal L_{xent} + \lambda_{KL}\mathcal L_{KL} + \lambda_{l_2}\mathcal L_{l_2}.
\]
The cross-entropy loss is defined as
\[
\mathcal L_{xent} = -\mathbb E_{(\mathtt{img}, f_1, \ldots, f_K)\sim \mathcal D}
\left[
\sum_{i=1}^K \log r_{f_i, \psi}(\widetilde{\mathfrak{m}}(\mathtt{img})) \right].
\]

Furthermore, 
$
\mathcal L_{KL}=\mathbb E_{x\sim\mathcal D}\left[\sum_{i=1}^K 
\text{KL}(U(\mathcal A_s)||s_\theta^i(x))\right]
$
and 
$
\mathcal L_{l_2}=||\theta||_2 + ||\psi||_2
$, where $U(\mathcal A_s)$ denotes the uniform distribution of $\mathcal A_s$.
The $\text{KL}$ loss incentives the language to be a one-to-one mapping.

\subsection{Alternative noise architecture}
The above implementation of noise is not the only one possible. We can apply noise after the message is formed. Denote the uncorrupted sender's message
\begin{align*}
	{\mathfrak{m}}_i&= \mathtt{gumbel\_softmax}(\log s_\theta^i(\mathtt{img}); \tau, g^i)\in\mathbb R^{d_s}.
\end{align*}
Further, let $\rho_i$ be uniformly sampled random permutations of $A_s$ and $\epsilon_i$ are Bernoulli random variables such that $\mathbb{P}(\epsilon_i=1) = \epsilon = 1- \mathbb{P}(\epsilon_i=0)$, with $\epsilon>0$. We define the noise matrices
\[
	N_i = \epsilon_i P^{\rho_i} + (1 - \epsilon_i) \mathbb{I},
\]
where $P^{\rho_i}$ is the permutation matrix corresponding to $\rho_i$ and $\mathbb{I}$ is the identity matrix. The corrupted message is then given by 
\[
	\widehat{m}_i = N_i {\mathfrak{m}}_i.
\]
Above we assume that $\widehat{m}_i, {\mathfrak{m}}_i$ is encoded in the one-hot vector form $\in \mathbb{R}^{d_s}$. We note that this particular implementation of the noise has the advantage of being differentiable using the standard autograd methods.\footnote{We cannot differentiate the random sampling of $\rho_i, \epsilon_i$ but we can differentiate multiplication with respect to $N_i$.}

\section{Detailed results}\label{sec:app:results}
Each experiment was run on 100 seeds. When presenting results we give $95\%$-confidence intervals, bootstrapped using 4000 resamples.

\subsection{Main experiment}\label{sec:nostalgic_lovelace}
The results for the shapes3d dataset are is visualized in Figure \ref{fig:app:avg_with_cis_shapes3d_repeat} 
(which is the copy of Figure \ref{fig:avg_with_cis_romantic_mcnulty} placed here for convenience), and numerical results are summarized in  
Table  \ref{app:tab:full_baseline_table_shapes3d}. 
Similarly, for the obverter dataset, the results can be found in Figure \ref{fig:avg_with_cis_obverter} and in Table \ref{app:tab:full_baseline_table_obverter}. 
From the qualitative perspective, the results for the obverter dataset are similar to the ones obtained for the shapes3d dataset. Having said that, we can see that overall the metrics are more stable (resulting in narrower confidence intervals), the best performing noise level is slightly different ($0.12$), and the density evolution of topo across different noise levels appears to be smoother.

\begin{figure}[H]
    \centering
    \includegraphics[width=\textwidth]{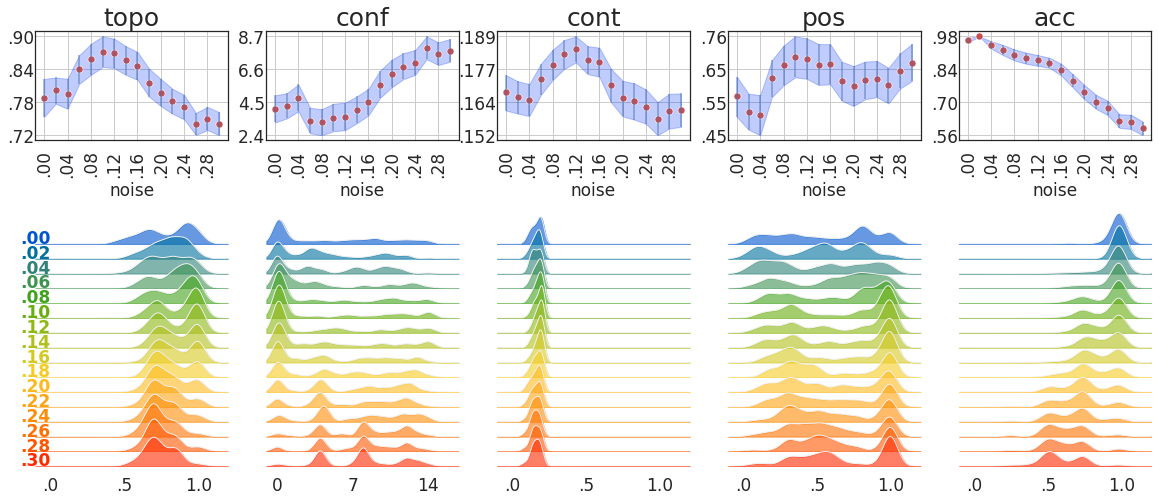}
    \caption{\small Shapes3d dataset. Top panel: average value of metrics for various noise levels. The shaded area corresponds to bootstrapped $95\%$-confidence intervals. Bottom panel: kernel density estimators for metrics and noise levels across seeds. Here \emph{topo} stands for topographic similarity, \emph{conf} for conflict count, \emph{cont} for context independence, \emph{pos} for positional disentanglement and \emph{acc} for accuracy.}
    \label{fig:app:avg_with_cis_shapes3d_repeat}
	\vspace{-1.5em}
\end{figure}

\begin{figure}[H]
    \centering
    \includegraphics[width=\textwidth]{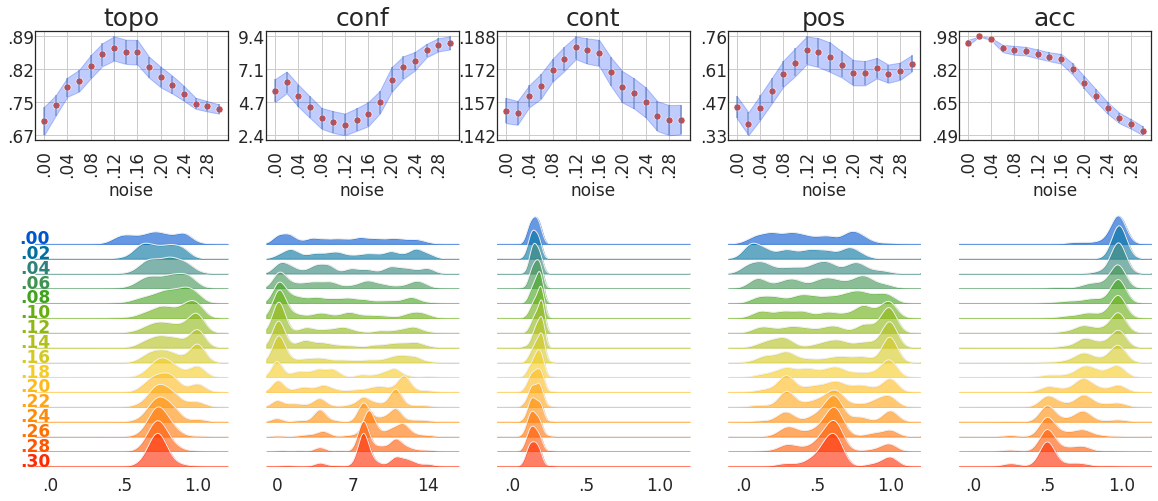}
    \caption{\small Obverter dataset. Top panel: average value of metrics for various noise levels. The shaded area corresponds to bootstrapped $95\%$-confidence intervals for this estimator. Bottom panel: kernel density estimators for metrics and noise levels across seeds.
    Bottom panel: kernel density estimators for metrics and noise levels across seeds. Here \emph{topo} stands for topographic similarity, \emph{conf} for conflict count, \emph{cont} for context independence, \emph{pos} for positional disentanglement and \emph{acc} for accuracy.
    }
    \label{fig:avg_with_cis_obverter}
\end{figure}

As observed in Section \ref{sec:baseline_experiment}, there is an interesting and non-trivial interplay between compositionality and accuracy. In Figure \ref{fig:shapes3d_filter} we visualize topographic similarity behavior when conditioned on experiments with high accuracy (for shapes3d; for the obverter see Figure \ref{fig:obverter_filter}). We can see an increase in the metrics values across all noise levels with the increase of accuracy. For example, for noise $0.1$ and threshold $0.85$, topo equals $0.91$, four percentage points higher than for unconditional case. 
We can see that increasing the threshold also strengthen the impact of noise on compositionality (which can be seen by increasing the profile of topo plots). 
Additionally, the count curves for smaller noises dominate the ones for higher noises. 
Notice, however, that the number of experiments exceeding some accuracy threshold declines as the threshold increases. For example, there are $80$ experiments with noise level $0.1$ exceeding the threshold of $0.85$ accuracy, but only $4$ with noise level $0.3$. 
This implies that conclusions for high accuracy thresholds should be treated with care.

\begin{figure}[H]
    \centering
    \includegraphics[width=\textwidth]{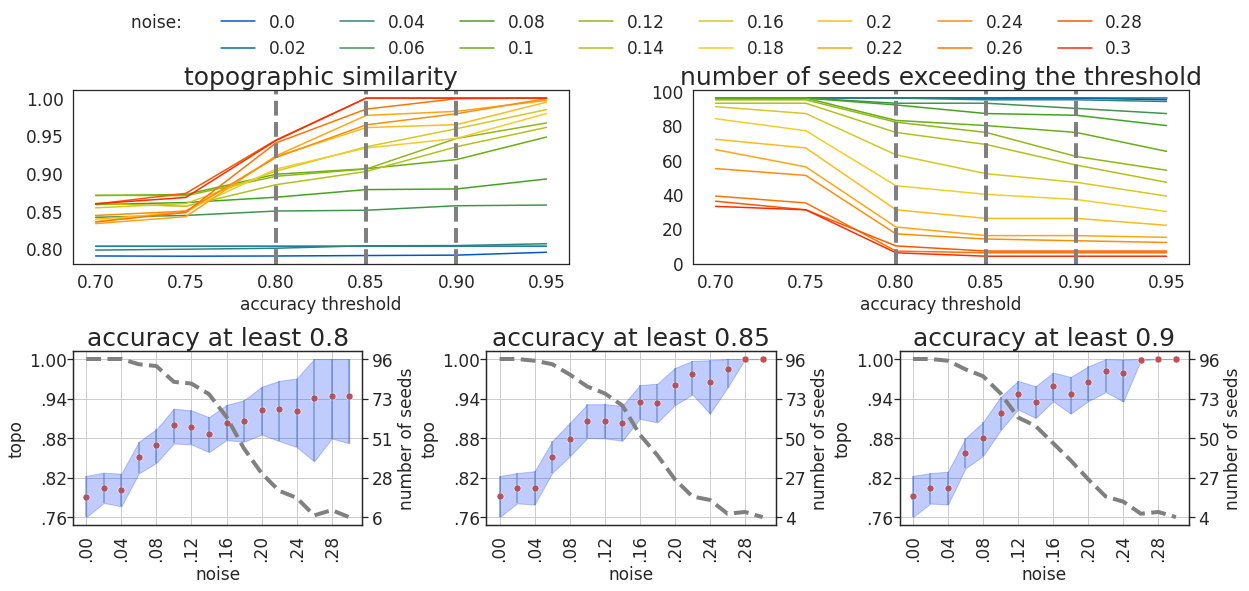}
    \caption{\small Shapes3d dataset. Top left: The values of topo computed for each noise level (hues) and on seeds exceeding a certain accuracy threshold ($x$-axis). Vertical dashed lines represent three cross-sections, visualized in the bottom panel.
    Top right: Similar to the left panel, but instead of topo we visualize the number of seeds with accuracy at least as a given threshold ($x$-axis). Vertical dashed lines represent three cross-sections, visualized in the bottom panel.
    Bottom: Each of the plots represents a cross-section of the plots in the top panel, taken at points $0.80$, $0.85$, and $0.90$, respectively. 
    On the left axis of each figure is the range of topo, whereas on the right axis is the number of seeds with accuracy exceed the corresponding level. On the $x$-axis are the noise levels. The scatter plot with 95\%-confidence intervals represents the values of topo. The gray dashed line represents the number of seeds with accuracy exceeding a given threshold, for each of the noise levels.}
    \label{fig:shapes3d_filter}
\end{figure}

\begin{figure}[H]
    \centering
    \includegraphics[width=\textwidth]{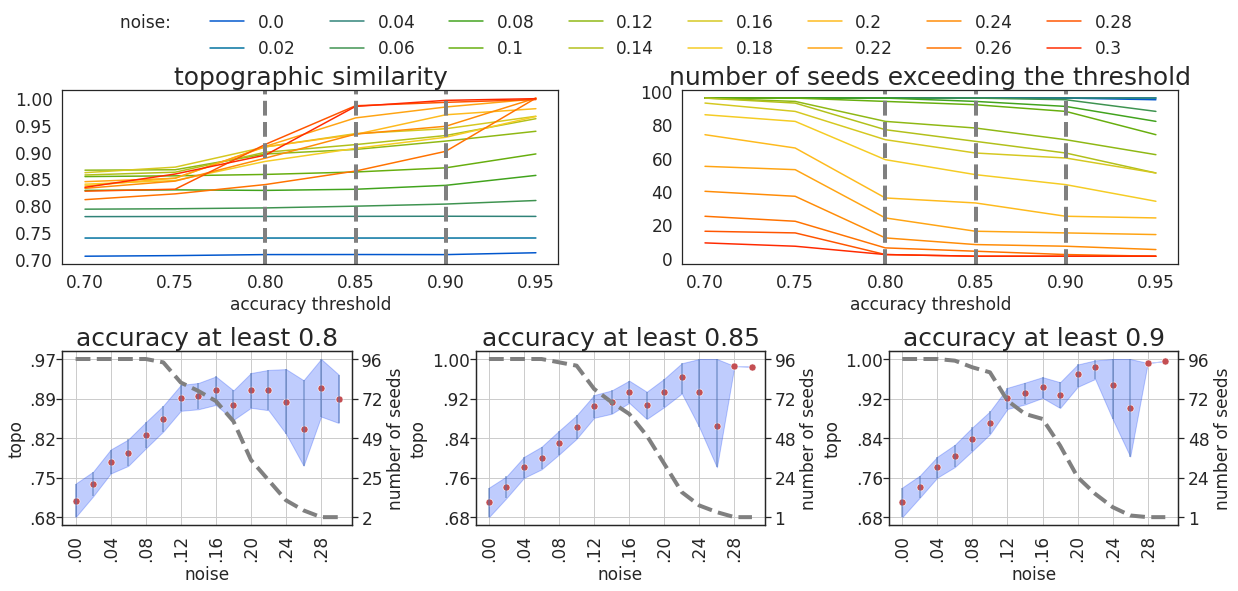}
    \caption{\small Obverter dataset. Top left: The values of topo computed for each noise level (hues) and on seeds exceeding a certain accuracy threshold ($x$-axis). Vertical dashed lines represent three cross-sections, visualized in the bottom panel.
    Top right: Similar to the left panel, but instead of topo we visualize the number of seeds with accuracy at least as a given threshold ($x$-axis). Vertical dashed lines represent three cross-sections, visualized in the bottom panel.
    Bottom: Each of the plots represents a cross-section of the plots in the top panel, taken at points $0.80$, $0.85$, and $0.90$, respectively. 
    On the left axis of each figure is the range of topo, whereas on the right axis is the number of seeds with accuracy exceed the corresponding level. On the $x$-axis are the noise levels. The scatter plot with 95\%-confidence intervals represents the values of topo. The gray dashed line represents the number of seeds with accuracy exceeding a given threshold, for each of the noise levels.}
    \label{fig:obverter_filter}
\end{figure}

For the main experiment, we also provide pair-plots for all metrics and three noise levels: $0.0$, $0.1$, and $0.2$,  see Figure \ref{fig:shapes3d_pairplot} and Figure \ref{fig:obverter_pairplot} for the shapes3d and the obverter datasets, respectively. 
It shows the Spearman correlation between metrics, broken down to noise levels (color-coded circles in the upper triangle of the grid) as well as for the entire group (white circle in the upper triangle of the grid). We can see in Figure \ref{fig:shapes3d_pairplot} that the metrics are highly correlated (the negative correlation with conflict count follows by definition of the metric, see Section \ref{sec:compositionality_metrics}). 
For zero-noise (blue color), we see that accuracy is high irrespective of the compositionality metrics, resulting in an almost vertical line. Looking at the topo-acc cell, we can see that for mediocre accuracy values, the noise level $0.2$ (yellow color) tends to score higher in topo metrics than the noise level $0.1$ (green color). This relation reverses for accuracy values closer to $1.0$. 
For topo-conf and topo-pos cells, we see a visible linear correlation, with the effect weakening slightly for higher noise levels.
At the topo-cont cell, a lot of the mass of all noise levels is occupied in the center, but the noise level $0.1$ more frequently stays in the upper right corner. Similar observations can be done for Figure \ref{fig:obverter_pairplot}.

\begin{figure}[H]
    \centering
    \includegraphics[width=\textwidth]{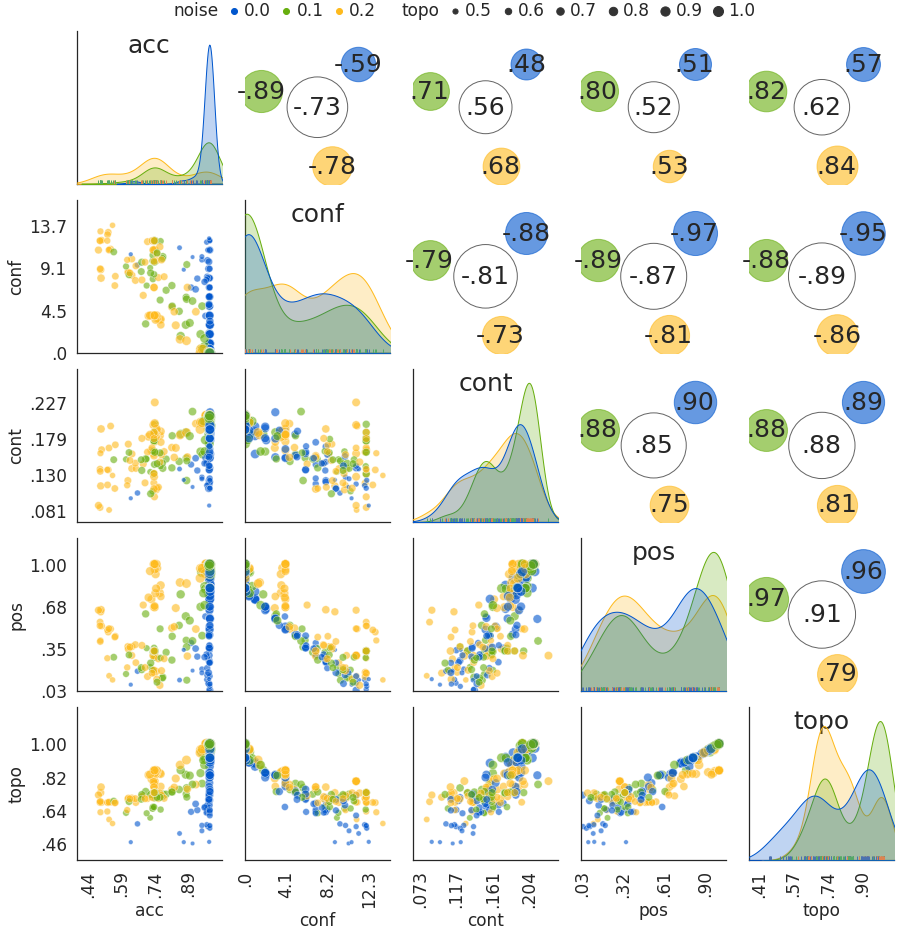}
    \caption{\small Shapes3d dataset. The lower triangle of the grid: color-coded scatter plot for experiments with different noise levels.  The upper triangle of the grid: visualization of Spearman correlation between metrics. The large circle with white fill shows the correlation of metrics value without a split into noise levels. The smaller color-coded circles represent the in-group correlation. Diagonal: kernel density estimators for each metric and noise level.
    Here \emph{topo} stands for topographic similarity, \emph{conf} for conflict count, \emph{cont} for context independence, \emph{pos} for positional disentanglement and \emph{acc} for accuracy.}
    \label{fig:shapes3d_pairplot}
\end{figure}

\begin{figure}[H]
    \centering
    \includegraphics[width=\textwidth]{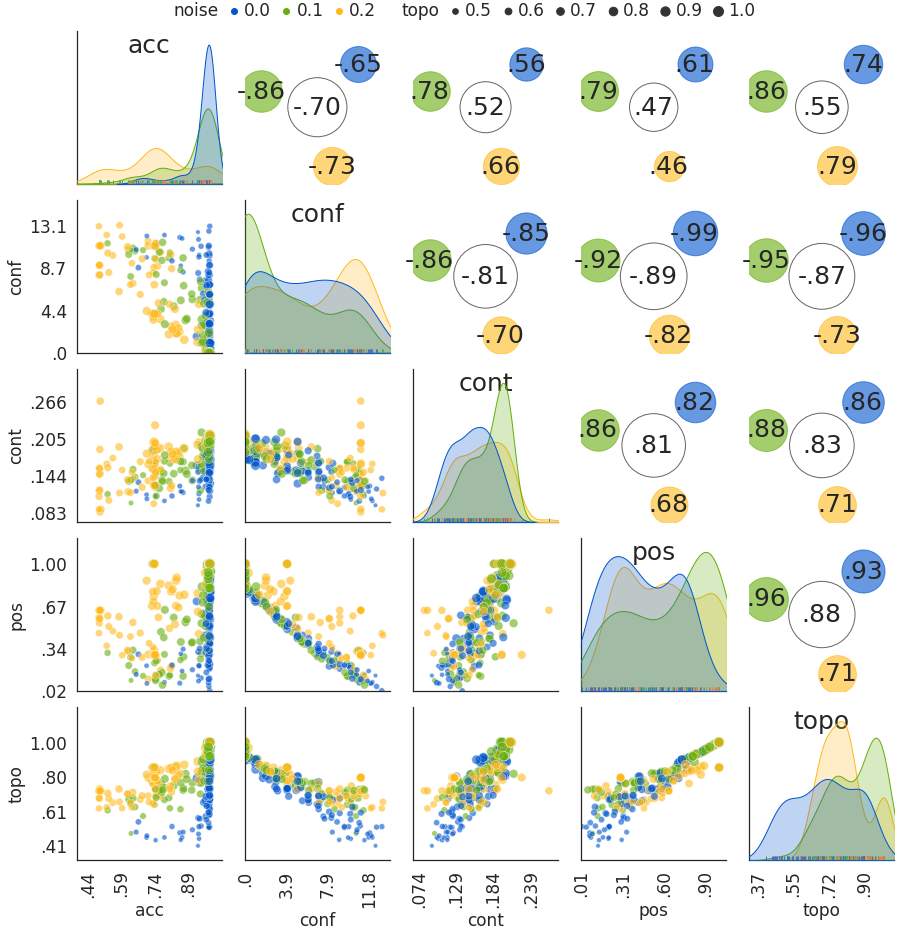}
    \caption{\small Obverter dataset. The lower triangle of the grid: color-coded scatter plot for experiments with different noise levels.  The upper triangle of the grid: visualization of Spearman correlation between metrics. The large circle with white fill shows the correlation of metrics value without a split into noise levels. The smaller color-coded circles represent the in-group correlation. Diagonal: kernel density estimators for each metric and noise level.
    Here \emph{topo} stands for topographic similarity, \emph{conf} for conflict count, \emph{cont} for context independence, \emph{pos} for positional disentanglement and \emph{acc} for accuracy.}
    \label{fig:obverter_pairplot}
\end{figure}

\begin{table}[H]
    \centering
    \begin{tabular}{l|ccccc}
\toprule
noise &                                              topo &                                              conf &                                                 cont &                                               pos &                                               acc \\ \midrule
0.00  &  \makecell{$0.79$ \\ \scriptsize{$[0.76, 0.82]$}} &  \makecell{$4.08$ \\ \scriptsize{$[3.21, 4.94]$}} &  \makecell{$0.168$ \\ \scriptsize{$[0.161, 0.175]$}} &  \makecell{$0.57$ \\ \scriptsize{$[0.51, 0.63]$}} &  \makecell{$0.97$ \\ \scriptsize{$[0.96, 0.98]$}} \\ \midrule
0.02  &  \makecell{$0.80$ \\ \scriptsize{$[0.78, 0.82]$}} &  \makecell{$4.26$ \\ \scriptsize{$[3.46, 5.08]$}} &  \makecell{$0.166$ \\ \scriptsize{$[0.160, 0.172]$}} &  \makecell{$0.52$ \\ \scriptsize{$[0.47, 0.58]$}} &  \makecell{$0.98$ \\ \scriptsize{$[0.98, 0.98]$}} \\ \midrule
0.04  &  \makecell{$0.80$ \\ \scriptsize{$[0.77, 0.82]$}} &  \makecell{$4.78$ \\ \scriptsize{$[3.92, 5.64]$}} &  \makecell{$0.165$ \\ \scriptsize{$[0.159, 0.171]$}} &  \makecell{$0.51$ \\ \scriptsize{$[0.45, 0.57]$}} &  \makecell{$0.95$ \\ \scriptsize{$[0.93, 0.96]$}} \\ \midrule
0.06  &  \makecell{$0.84$ \\ \scriptsize{$[0.81, 0.86]$}} &  \makecell{$3.31$ \\ \scriptsize{$[2.52, 4.15]$}} &  \makecell{$0.173$ \\ \scriptsize{$[0.167, 0.179]$}} &  \makecell{$0.62$ \\ \scriptsize{$[0.57, 0.68]$}} &  \makecell{$0.93$ \\ \scriptsize{$[0.90, 0.95]$}} \\ \midrule
0.08  &  \makecell{$0.86$ \\ \scriptsize{$[0.83, 0.88]$}} &  \makecell{$3.22$ \\ \scriptsize{$[2.40, 4.06]$}} &  \makecell{$0.179$ \\ \scriptsize{$[0.173, 0.184]$}} &  \makecell{$0.66$ \\ \scriptsize{$[0.60, 0.73]$}} &  \makecell{$0.91$ \\ \scriptsize{$[0.88, 0.93]$}} \\ \midrule
0.10  &  \makecell{$0.87$ \\ \scriptsize{$[0.84, 0.90]$}} &  \makecell{$3.49$ \\ \scriptsize{$[2.64, 4.41]$}} &  \makecell{$0.183$ \\ \scriptsize{$[0.177, 0.188]$}} &  \makecell{$0.69$ \\ \scriptsize{$[0.63, 0.76]$}} &  \makecell{$0.89$ \\ \scriptsize{$[0.87, 0.92]$}} \\ \midrule
0.12  &  \makecell{$0.87$ \\ \scriptsize{$[0.84, 0.89]$}} &  \makecell{$3.58$ \\ \scriptsize{$[2.71, 4.47]$}} &  \makecell{$0.184$ \\ \scriptsize{$[0.179, 0.189]$}} &  \makecell{$0.69$ \\ \scriptsize{$[0.62, 0.75]$}} &  \makecell{$0.88$ \\ \scriptsize{$[0.86, 0.90]$}} \\ \midrule
0.14  &  \makecell{$0.85$ \\ \scriptsize{$[0.83, 0.88]$}} &  \makecell{$3.99$ \\ \scriptsize{$[3.16, 4.87]$}} &  \makecell{$0.180$ \\ \scriptsize{$[0.175, 0.186]$}} &  \makecell{$0.67$ \\ \scriptsize{$[0.60, 0.73]$}} &  \makecell{$0.87$ \\ \scriptsize{$[0.85, 0.89]$}} \\ \midrule
0.16  &  \makecell{$0.84$ \\ \scriptsize{$[0.82, 0.87]$}} &  \makecell{$4.49$ \\ \scriptsize{$[3.66, 5.36]$}} &  \makecell{$0.180$ \\ \scriptsize{$[0.174, 0.185]$}} &  \makecell{$0.67$ \\ \scriptsize{$[0.61, 0.73]$}} &  \makecell{$0.84$ \\ \scriptsize{$[0.82, 0.87]$}} \\ \midrule
0.18  &  \makecell{$0.81$ \\ \scriptsize{$[0.79, 0.84]$}} &  \makecell{$5.63$ \\ \scriptsize{$[4.78, 6.52]$}} &  \makecell{$0.171$ \\ \scriptsize{$[0.164, 0.177]$}} &  \makecell{$0.62$ \\ \scriptsize{$[0.56, 0.68]$}} &  \makecell{$0.80$ \\ \scriptsize{$[0.77, 0.82]$}} \\ \midrule
0.20  &  \makecell{$0.80$ \\ \scriptsize{$[0.77, 0.82]$}} &  \makecell{$6.32$ \\ \scriptsize{$[5.45, 7.22]$}} &  \makecell{$0.166$ \\ \scriptsize{$[0.159, 0.173]$}} &  \makecell{$0.60$ \\ \scriptsize{$[0.54, 0.66]$}} &  \makecell{$0.75$ \\ \scriptsize{$[0.72, 0.78]$}} \\ \midrule
0.22  &  \makecell{$0.78$ \\ \scriptsize{$[0.76, 0.81]$}} &  \makecell{$6.78$ \\ \scriptsize{$[5.97, 7.64]$}} &  \makecell{$0.165$ \\ \scriptsize{$[0.158, 0.172]$}} &  \makecell{$0.62$ \\ \scriptsize{$[0.56, 0.68]$}} &  \makecell{$0.71$ \\ \scriptsize{$[0.68, 0.74]$}} \\ \midrule
0.24  &  \makecell{$0.77$ \\ \scriptsize{$[0.75, 0.79]$}} &  \makecell{$7.01$ \\ \scriptsize{$[6.24, 7.83]$}} &  \makecell{$0.163$ \\ \scriptsize{$[0.156, 0.169]$}} &  \makecell{$0.62$ \\ \scriptsize{$[0.57, 0.68]$}} &  \makecell{$0.68$ \\ \scriptsize{$[0.65, 0.71]$}} \\ \midrule
0.26  &  \makecell{$0.74$ \\ \scriptsize{$[0.72, 0.76]$}} &  \makecell{$7.97$ \\ \scriptsize{$[7.25, 8.71]$}} &  \makecell{$0.158$ \\ \scriptsize{$[0.152, 0.164]$}} &  \makecell{$0.60$ \\ \scriptsize{$[0.55, 0.66]$}} &  \makecell{$0.62$ \\ \scriptsize{$[0.59, 0.65]$}} \\ \midrule
0.28  &  \makecell{$0.75$ \\ \scriptsize{$[0.73, 0.77]$}} &  \makecell{$7.58$ \\ \scriptsize{$[6.88, 8.32]$}} &  \makecell{$0.161$ \\ \scriptsize{$[0.154, 0.168]$}} &  \makecell{$0.65$ \\ \scriptsize{$[0.59, 0.70]$}} &  \makecell{$0.62$ \\ \scriptsize{$[0.59, 0.65]$}} \\ \midrule
0.30  &  \makecell{$0.74$ \\ \scriptsize{$[0.72, 0.76]$}} &  \makecell{$7.80$ \\ \scriptsize{$[7.08, 8.54]$}} &  \makecell{$0.161$ \\ \scriptsize{$[0.155, 0.168]$}} &  \makecell{$0.67$ \\ \scriptsize{$[0.62, 0.73]$}} &  \makecell{$0.59$ \\ \scriptsize{$[0.56, 0.62]$}} \\ 
\bottomrule
\end{tabular}

    \caption{\small Shapes3d dataset. Results for the metrics for selected noise levels. Shown in square brackets are bootstrapped $95\%$-confidence intervals.
    Here \emph{topo} stands for topographic similarity, \emph{conf} for conflict count, \emph{cont} for context independence, \emph{pos} for positional disentanglement and \emph{acc} for accuracy.
    }
    \label{app:tab:full_baseline_table_shapes3d}
\end{table}

\begin{table}
    \centering
    \begin{tabular}{l|ccccc}
\toprule
noise &                                              topo &                                              conf &                                                 cont &                                               pos &                                               acc \\ \midrule
0.00  &  \makecell{$0.71$ \\ \scriptsize{$[0.67, 0.74]$}} &  \makecell{$5.54$ \\ \scriptsize{$[4.73, 6.37]$}} &  \makecell{$0.153$ \\ \scriptsize{$[0.147, 0.159]$}} &  \makecell{$0.45$ \\ \scriptsize{$[0.40, 0.50]$}} &  \makecell{$0.95$ \\ \scriptsize{$[0.94, 0.97]$}} \\ \midrule
0.02  &  \makecell{$0.74$ \\ \scriptsize{$[0.72, 0.76]$}} &  \makecell{$6.15$ \\ \scriptsize{$[5.36, 6.91]$}} &  \makecell{$0.152$ \\ \scriptsize{$[0.146, 0.157]$}} &  \makecell{$0.38$ \\ \scriptsize{$[0.33, 0.43]$}} &  \makecell{$0.98$ \\ \scriptsize{$[0.98, 0.98]$}} \\ \midrule
0.04  &  \makecell{$0.78$ \\ \scriptsize{$[0.76, 0.80]$}} &  \makecell{$5.19$ \\ \scriptsize{$[4.42, 5.97]$}} &  \makecell{$0.160$ \\ \scriptsize{$[0.155, 0.165]$}} &  \makecell{$0.45$ \\ \scriptsize{$[0.39, 0.50]$}} &  \makecell{$0.97$ \\ \scriptsize{$[0.97, 0.98]$}} \\ \midrule
0.06  &  \makecell{$0.79$ \\ \scriptsize{$[0.77, 0.82]$}} &  \makecell{$4.36$ \\ \scriptsize{$[3.61, 5.14]$}} &  \makecell{$0.164$ \\ \scriptsize{$[0.158, 0.170]$}} &  \makecell{$0.52$ \\ \scriptsize{$[0.46, 0.58]$}} &  \makecell{$0.92$ \\ \scriptsize{$[0.90, 0.94]$}} \\ \midrule
0.08  &  \makecell{$0.83$ \\ \scriptsize{$[0.80, 0.85]$}} &  \makecell{$3.58$ \\ \scriptsize{$[2.84, 4.32]$}} &  \makecell{$0.172$ \\ \scriptsize{$[0.166, 0.177]$}} &  \makecell{$0.59$ \\ \scriptsize{$[0.54, 0.65]$}} &  \makecell{$0.91$ \\ \scriptsize{$[0.89, 0.94]$}} \\ \midrule
0.10  &  \makecell{$0.85$ \\ \scriptsize{$[0.82, 0.88]$}} &  \makecell{$3.30$ \\ \scriptsize{$[2.54, 4.09]$}} &  \makecell{$0.177$ \\ \scriptsize{$[0.171, 0.183]$}} &  \makecell{$0.64$ \\ \scriptsize{$[0.58, 0.70]$}} &  \makecell{$0.91$ \\ \scriptsize{$[0.89, 0.93]$}} \\ \midrule
0.12  &  \makecell{$0.86$ \\ \scriptsize{$[0.84, 0.89]$}} &  \makecell{$3.14$ \\ \scriptsize{$[2.39, 3.93]$}} &  \makecell{$0.182$ \\ \scriptsize{$[0.177, 0.188]$}} &  \makecell{$0.70$ \\ \scriptsize{$[0.64, 0.76]$}} &  \makecell{$0.90$ \\ \scriptsize{$[0.87, 0.92]$}} \\ \midrule
0.14  &  \makecell{$0.86$ \\ \scriptsize{$[0.83, 0.88]$}} &  \makecell{$3.50$ \\ \scriptsize{$[2.69, 4.31]$}} &  \makecell{$0.181$ \\ \scriptsize{$[0.175, 0.187]$}} &  \makecell{$0.69$ \\ \scriptsize{$[0.62, 0.75]$}} &  \makecell{$0.88$ \\ \scriptsize{$[0.86, 0.90]$}} \\ \midrule
0.16  &  \makecell{$0.86$ \\ \scriptsize{$[0.83, 0.88]$}} &  \makecell{$3.87$ \\ \scriptsize{$[3.02, 4.74]$}} &  \makecell{$0.180$ \\ \scriptsize{$[0.174, 0.186]$}} &  \makecell{$0.67$ \\ \scriptsize{$[0.60, 0.73]$}} &  \makecell{$0.87$ \\ \scriptsize{$[0.85, 0.90]$}} \\ \midrule
0.18  &  \makecell{$0.82$ \\ \scriptsize{$[0.80, 0.85]$}} &  \makecell{$4.73$ \\ \scriptsize{$[3.93, 5.55]$}} &  \makecell{$0.171$ \\ \scriptsize{$[0.164, 0.178]$}} &  \makecell{$0.63$ \\ \scriptsize{$[0.58, 0.69]$}} &  \makecell{$0.82$ \\ \scriptsize{$[0.79, 0.85]$}} \\ \midrule
0.20  &  \makecell{$0.80$ \\ \scriptsize{$[0.78, 0.82]$}} &  \makecell{$6.34$ \\ \scriptsize{$[5.46, 7.22]$}} &  \makecell{$0.164$ \\ \scriptsize{$[0.157, 0.171]$}} &  \makecell{$0.60$ \\ \scriptsize{$[0.54, 0.65]$}} &  \makecell{$0.75$ \\ \scriptsize{$[0.72, 0.78]$}} \\ \midrule
0.22  &  \makecell{$0.78$ \\ \scriptsize{$[0.76, 0.80]$}} &  \makecell{$7.22$ \\ \scriptsize{$[6.39, 7.99]$}} &  \makecell{$0.161$ \\ \scriptsize{$[0.154, 0.168]$}} &  \makecell{$0.60$ \\ \scriptsize{$[0.54, 0.65]$}} &  \makecell{$0.69$ \\ \scriptsize{$[0.65, 0.72]$}} \\ \midrule
0.24  &  \makecell{$0.76$ \\ \scriptsize{$[0.75, 0.78]$}} &  \makecell{$7.66$ \\ \scriptsize{$[6.98, 8.30]$}} &  \makecell{$0.157$ \\ \scriptsize{$[0.150, 0.164]$}} &  \makecell{$0.62$ \\ \scriptsize{$[0.57, 0.66]$}} &  \makecell{$0.62$ \\ \scriptsize{$[0.59, 0.65]$}} \\ \midrule
0.26  &  \makecell{$0.74$ \\ \scriptsize{$[0.73, 0.76]$}} &  \makecell{$8.42$ \\ \scriptsize{$[7.91, 8.88]$}} &  \makecell{$0.151$ \\ \scriptsize{$[0.144, 0.157]$}} &  \makecell{$0.59$ \\ \scriptsize{$[0.55, 0.63]$}} &  \makecell{$0.57$ \\ \scriptsize{$[0.55, 0.60]$}} \\ \midrule
0.28  &  \makecell{$0.74$ \\ \scriptsize{$[0.73, 0.75]$}} &  \makecell{$8.81$ \\ \scriptsize{$[8.33, 9.27]$}} &  \makecell{$0.149$ \\ \scriptsize{$[0.142, 0.156]$}} &  \makecell{$0.60$ \\ \scriptsize{$[0.57, 0.64]$}} &  \makecell{$0.54$ \\ \scriptsize{$[0.52, 0.57]$}} \\ \midrule
0.30  &  \makecell{$0.73$ \\ \scriptsize{$[0.72, 0.74]$}} &  \makecell{$8.94$ \\ \scriptsize{$[8.47, 9.41]$}} &  \makecell{$0.149$ \\ \scriptsize{$[0.142, 0.156]$}} &  \makecell{$0.64$ \\ \scriptsize{$[0.60, 0.67]$}} &  \makecell{$0.51$ \\ \scriptsize{$[0.49, 0.53]$}} \\ 
\bottomrule
\end{tabular}

    \caption{\small Obverter dataset. Results for the metrics for selected noise levels. Shown in square brackets are bootstrapped $95\%$-confidence intervals.
    Here \emph{topo} stands for topographic similarity, \emph{conf} for conflict count, \emph{cont} for context independence, \emph{pos} for positional disentanglement and \emph{acc} for accuracy.
    }
    \label{app:tab:full_baseline_table_obverter}
\end{table}

\pagebreak
\subsection{Different number of symbols}\label{sec:number_of_symbols_app}

This section presents detailed results for the case when the message space has $16=4\times 4$ elements (4-symbol alphabet; see Figure \ref{fig:ridge_4x4} and Table \ref{app:tab:4x4}) 
 and $64=8\times 8$ elements (8-symbol alphabet;
 see Figure \ref{fig:ridge_8x8} and Table \ref{app:tab:8x8}).
In the former case, the results are more variable. For topo, this can be seen from a bimodal shape of its distribution and wide confidence intervals,  particularly for small to medium values of noise. This makes it statistically hard to distinguish values of topo in this noise range. For larger values of noise (greater than $0.16$), topo starts to visibly decline. Similar behavior can be seen for other metrics. 

Interestingly, for the latter experiment, with $64=8\times 8$ messages, the story is different. While topographic similarity values for the small noise regime (up to $0.08$) do not improve over the baseline value, this behavior changes for medium to large values of noise. In this range, we can observe a visible increase in the topo, peaking at $0.88$ with a noise level of $0.24$.

\begin{figure}[H]
    \centering
    \includegraphics[width=\textwidth]{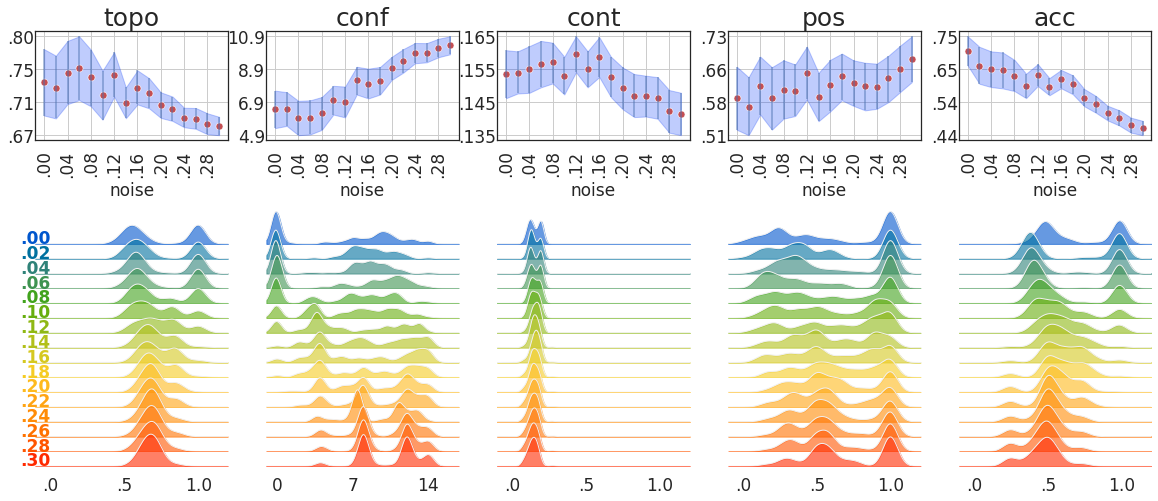}
    \caption{$4\times 4$. Top panel: average value of metrics for various noise levels (on shapes3d dataset). The shaded area corresponds to bootstrapped $95\%$-confidence intervals for this estimator. Bottom panel: kernel density estimators for metrics and noise levels across seeds.
    Here \emph{topo} stands for topographic similarity, \emph{conf} for conflict count, \emph{cont} for context independence, \emph{pos} for positional disentanglement and \emph{acc} for accuracy.
    }
    \label{fig:ridge_4x4}
\end{figure}

\begin{table}
    \centering
    \begin{tabular}{l|ccccc}
\toprule
noise &                                              topo &                                                conf &                                                 cont &                                               pos &                                               acc \\ \midrule
0.00  &  \makecell{$0.74$ \\ \scriptsize{$[0.69, 0.78]$}} &    \makecell{$6.45$ \\ \scriptsize{$[5.30, 7.57]$}} &  \makecell{$0.154$ \\ \scriptsize{$[0.146, 0.161]$}} &  \makecell{$0.59$ \\ \scriptsize{$[0.52, 0.67]$}} &  \makecell{$0.71$ \\ \scriptsize{$[0.66, 0.75]$}} \\ \midrule
0.02  &  \makecell{$0.73$ \\ \scriptsize{$[0.69, 0.77]$}} &    \makecell{$6.48$ \\ \scriptsize{$[5.44, 7.48]$}} &  \makecell{$0.154$ \\ \scriptsize{$[0.148, 0.160]$}} &  \makecell{$0.57$ \\ \scriptsize{$[0.51, 0.64]$}} &  \makecell{$0.66$ \\ \scriptsize{$[0.60, 0.72]$}} \\ \midrule
0.04  &  \makecell{$0.75$ \\ \scriptsize{$[0.71, 0.79]$}} &    \makecell{$5.93$ \\ \scriptsize{$[4.88, 6.96]$}} &  \makecell{$0.155$ \\ \scriptsize{$[0.148, 0.162]$}} &  \makecell{$0.62$ \\ \scriptsize{$[0.56, 0.69]$}} &  \makecell{$0.65$ \\ \scriptsize{$[0.59, 0.71]$}} \\ \midrule
0.06  &  \makecell{$0.76$ \\ \scriptsize{$[0.71, 0.80]$}} &    \makecell{$5.94$ \\ \scriptsize{$[4.92, 6.99]$}} &  \makecell{$0.156$ \\ \scriptsize{$[0.149, 0.163]$}} &  \makecell{$0.59$ \\ \scriptsize{$[0.52, 0.66]$}} &  \makecell{$0.65$ \\ \scriptsize{$[0.59, 0.70]$}} \\ \midrule
0.08  &  \makecell{$0.74$ \\ \scriptsize{$[0.71, 0.78]$}} &    \makecell{$6.19$ \\ \scriptsize{$[5.20, 7.19]$}} &  \makecell{$0.157$ \\ \scriptsize{$[0.151, 0.163]$}} &  \makecell{$0.61$ \\ \scriptsize{$[0.55, 0.68]$}} &  \makecell{$0.63$ \\ \scriptsize{$[0.58, 0.68]$}} \\ \midrule
0.10  &  \makecell{$0.72$ \\ \scriptsize{$[0.69, 0.75]$}} &    \makecell{$7.00$ \\ \scriptsize{$[6.12, 7.87]$}} &  \makecell{$0.153$ \\ \scriptsize{$[0.147, 0.159]$}} &  \makecell{$0.61$ \\ \scriptsize{$[0.55, 0.67]$}} &  \makecell{$0.60$ \\ \scriptsize{$[0.56, 0.63]$}} \\ \midrule
0.12  &  \makecell{$0.75$ \\ \scriptsize{$[0.72, 0.78]$}} &    \makecell{$6.90$ \\ \scriptsize{$[5.95, 7.82]$}} &  \makecell{$0.159$ \\ \scriptsize{$[0.154, 0.165]$}} &  \makecell{$0.65$ \\ \scriptsize{$[0.59, 0.71]$}} &  \makecell{$0.63$ \\ \scriptsize{$[0.60, 0.67]$}} \\ \midrule
0.14  &  \makecell{$0.71$ \\ \scriptsize{$[0.69, 0.73]$}} &    \makecell{$8.20$ \\ \scriptsize{$[7.34, 9.02]$}} &  \makecell{$0.155$ \\ \scriptsize{$[0.149, 0.160]$}} &  \makecell{$0.60$ \\ \scriptsize{$[0.54, 0.65]$}} &  \makecell{$0.59$ \\ \scriptsize{$[0.57, 0.62]$}} \\ \midrule
0.16  &  \makecell{$0.73$ \\ \scriptsize{$[0.71, 0.75]$}} &    \makecell{$8.01$ \\ \scriptsize{$[7.10, 8.91]$}} &  \makecell{$0.159$ \\ \scriptsize{$[0.153, 0.165]$}} &  \makecell{$0.62$ \\ \scriptsize{$[0.56, 0.68]$}} &  \makecell{$0.62$ \\ \scriptsize{$[0.59, 0.65]$}} \\ \midrule
0.18  &  \makecell{$0.72$ \\ \scriptsize{$[0.70, 0.74]$}} &    \makecell{$8.18$ \\ \scriptsize{$[7.34, 9.01]$}} &  \makecell{$0.153$ \\ \scriptsize{$[0.147, 0.159]$}} &  \makecell{$0.64$ \\ \scriptsize{$[0.59, 0.70]$}} &  \makecell{$0.60$ \\ \scriptsize{$[0.57, 0.63]$}} \\ \midrule
0.20  &  \makecell{$0.71$ \\ \scriptsize{$[0.69, 0.72]$}} &    \makecell{$8.93$ \\ \scriptsize{$[8.23, 9.63]$}} &  \makecell{$0.149$ \\ \scriptsize{$[0.143, 0.155]$}} &  \makecell{$0.63$ \\ \scriptsize{$[0.58, 0.68]$}} &  \makecell{$0.56$ \\ \scriptsize{$[0.53, 0.59]$}} \\ \midrule
0.22  &  \makecell{$0.70$ \\ \scriptsize{$[0.69, 0.72]$}} &   \makecell{$9.38$ \\ \scriptsize{$[8.67, 10.09]$}} &  \makecell{$0.147$ \\ \scriptsize{$[0.140, 0.153]$}} &  \makecell{$0.62$ \\ \scriptsize{$[0.57, 0.67]$}} &  \makecell{$0.54$ \\ \scriptsize{$[0.51, 0.56]$}} \\ \midrule
0.24  &  \makecell{$0.69$ \\ \scriptsize{$[0.68, 0.70]$}} &   \makecell{$9.86$ \\ \scriptsize{$[9.23, 10.46]$}} &  \makecell{$0.147$ \\ \scriptsize{$[0.141, 0.153]$}} &  \makecell{$0.62$ \\ \scriptsize{$[0.57, 0.67]$}} &  \makecell{$0.51$ \\ \scriptsize{$[0.49, 0.53]$}} \\ \midrule
0.26  &  \makecell{$0.69$ \\ \scriptsize{$[0.68, 0.70]$}} &   \makecell{$9.87$ \\ \scriptsize{$[9.29, 10.45]$}} &  \makecell{$0.146$ \\ \scriptsize{$[0.140, 0.153]$}} &  \makecell{$0.64$ \\ \scriptsize{$[0.58, 0.69]$}} &  \makecell{$0.49$ \\ \scriptsize{$[0.47, 0.52]$}} \\ \midrule
0.28  &  \makecell{$0.68$ \\ \scriptsize{$[0.67, 0.70]$}} &  \makecell{$10.15$ \\ \scriptsize{$[9.60, 10.72]$}} &  \makecell{$0.142$ \\ \scriptsize{$[0.136, 0.149]$}} &  \makecell{$0.66$ \\ \scriptsize{$[0.61, 0.71]$}} &  \makecell{$0.47$ \\ \scriptsize{$[0.45, 0.49]$}} \\ \midrule
0.30  &  \makecell{$0.68$ \\ \scriptsize{$[0.67, 0.69]$}} &  \makecell{$10.32$ \\ \scriptsize{$[9.78, 10.86]$}} &  \makecell{$0.141$ \\ \scriptsize{$[0.135, 0.148]$}} &  \makecell{$0.68$ \\ \scriptsize{$[0.63, 0.73]$}} &  \makecell{$0.46$ \\ \scriptsize{$[0.44, 0.48]$}} \\ 
\bottomrule
\end{tabular}

    \caption{4x4. Results for the metrics for selected noise levels. Shown in square brackets are bootstrapped $95\%$-confidence intervals.
    Here \emph{topo} stands for topographic similarity, \emph{conf} for conflict count, \emph{cont} for context independence, \emph{pos} for positional disentanglement and \emph{acc} for accuracy.
    }
    \label{app:tab:4x4}
\end{table}

\begin{figure}[H]
    \centering
    \includegraphics[width=\textwidth]{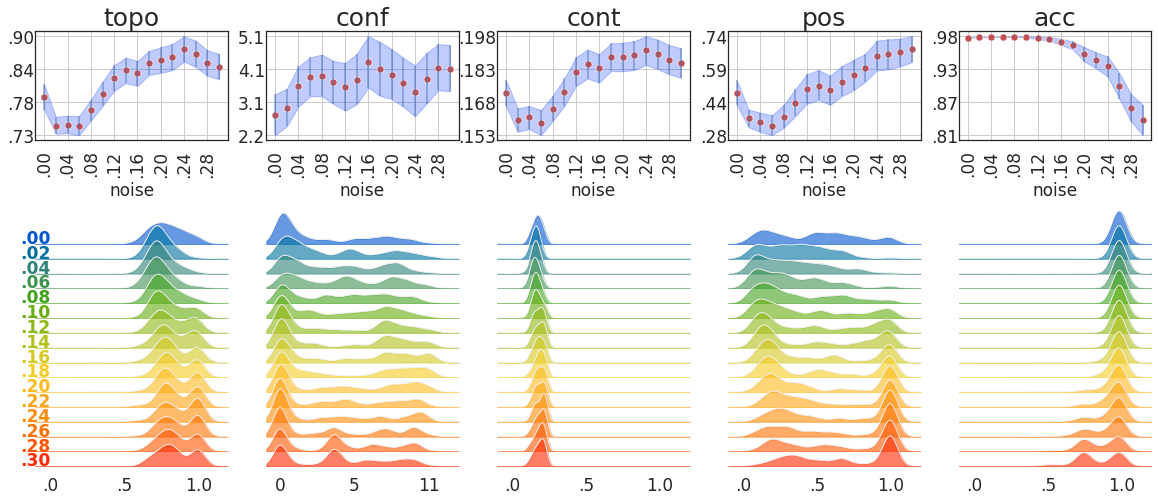}
    \caption{8x8. Top panel: average value of metrics for various noise levels. The shaded area corresponds to bootstrapped $95\%$-confidence intervals for this estimator. Bottom panel: kernel density estimators for metrics and noise levels across seeds.
    Here \emph{topo} stands for topographic similarity, \emph{conf} for conflict count, \emph{cont} for context independence, \emph{pos} for positional disentanglement and \emph{acc} for accuracy.}
    \label{fig:ridge_8x8}
\end{figure}

\begin{table}
    \centering
    \begin{tabular}{l|ccccc}
\toprule
noise &                                              topo &                                              conf &                                                 cont &                                               pos &                                               acc \\ \midrule
0.00  &  \makecell{$0.79$ \\ \scriptsize{$[0.77, 0.82]$}} &  \makecell{$2.77$ \\ \scriptsize{$[2.17, 3.39]$}} &  \makecell{$0.173$ \\ \scriptsize{$[0.166, 0.178]$}} &  \makecell{$0.48$ \\ \scriptsize{$[0.42, 0.54]$}} &  \makecell{$0.98$ \\ \scriptsize{$[0.98, 0.98]$}} \\ \midrule
0.02  &  \makecell{$0.74$ \\ \scriptsize{$[0.73, 0.76]$}} &  \makecell{$2.99$ \\ \scriptsize{$[2.43, 3.55]$}} &  \makecell{$0.160$ \\ \scriptsize{$[0.155, 0.166]$}} &  \makecell{$0.36$ \\ \scriptsize{$[0.32, 0.40]$}} &  \makecell{$0.98$ \\ \scriptsize{$[0.98, 0.98]$}} \\ \midrule
0.04  &  \makecell{$0.75$ \\ \scriptsize{$[0.73, 0.76]$}} &  \makecell{$3.62$ \\ \scriptsize{$[3.02, 4.22]$}} &  \makecell{$0.161$ \\ \scriptsize{$[0.156, 0.167]$}} &  \makecell{$0.34$ \\ \scriptsize{$[0.30, 0.39]$}} &  \makecell{$0.98$ \\ \scriptsize{$[0.98, 0.98]$}} \\ \midrule
0.06  &  \makecell{$0.74$ \\ \scriptsize{$[0.73, 0.76]$}} &  \makecell{$3.90$ \\ \scriptsize{$[3.33, 4.49]$}} &  \makecell{$0.159$ \\ \scriptsize{$[0.153, 0.165]$}} &  \makecell{$0.33$ \\ \scriptsize{$[0.28, 0.37]$}} &  \makecell{$0.98$ \\ \scriptsize{$[0.98, 0.98]$}} \\ \midrule
0.08  &  \makecell{$0.77$ \\ \scriptsize{$[0.75, 0.79]$}} &  \makecell{$3.94$ \\ \scriptsize{$[3.33, 4.57]$}} &  \makecell{$0.165$ \\ \scriptsize{$[0.160, 0.171]$}} &  \makecell{$0.37$ \\ \scriptsize{$[0.32, 0.42]$}} &  \makecell{$0.98$ \\ \scriptsize{$[0.98, 0.98]$}} \\ \midrule
0.10  &  \makecell{$0.80$ \\ \scriptsize{$[0.78, 0.82]$}} &  \makecell{$3.74$ \\ \scriptsize{$[3.10, 4.38]$}} &  \makecell{$0.173$ \\ \scriptsize{$[0.167, 0.179]$}} &  \makecell{$0.43$ \\ \scriptsize{$[0.37, 0.50]$}} &  \makecell{$0.98$ \\ \scriptsize{$[0.98, 0.98]$}} \\ \midrule
0.12  &  \makecell{$0.83$ \\ \scriptsize{$[0.80, 0.85]$}} &  \makecell{$3.60$ \\ \scriptsize{$[2.90, 4.32]$}} &  \makecell{$0.182$ \\ \scriptsize{$[0.176, 0.189]$}} &  \makecell{$0.50$ \\ \scriptsize{$[0.43, 0.56]$}} &  \makecell{$0.98$ \\ \scriptsize{$[0.98, 0.98]$}} \\ \midrule
0.14  &  \makecell{$0.84$ \\ \scriptsize{$[0.82, 0.86]$}} &  \makecell{$3.82$ \\ \scriptsize{$[3.09, 4.59]$}} &  \makecell{$0.186$ \\ \scriptsize{$[0.179, 0.192]$}} &  \makecell{$0.51$ \\ \scriptsize{$[0.44, 0.58]$}} &  \makecell{$0.98$ \\ \scriptsize{$[0.98, 0.98]$}} \\ \midrule
0.16  &  \makecell{$0.83$ \\ \scriptsize{$[0.81, 0.86]$}} &  \makecell{$4.35$ \\ \scriptsize{$[3.59, 5.11]$}} &  \makecell{$0.184$ \\ \scriptsize{$[0.177, 0.190]$}} &  \makecell{$0.49$ \\ \scriptsize{$[0.42, 0.56]$}} &  \makecell{$0.97$ \\ \scriptsize{$[0.97, 0.98]$}} \\ \midrule
0.18  &  \makecell{$0.85$ \\ \scriptsize{$[0.83, 0.87]$}} &  \makecell{$4.13$ \\ \scriptsize{$[3.34, 4.95]$}} &  \makecell{$0.189$ \\ \scriptsize{$[0.182, 0.195]$}} &  \makecell{$0.53$ \\ \scriptsize{$[0.46, 0.60]$}} &  \makecell{$0.97$ \\ \scriptsize{$[0.96, 0.97]$}} \\ \midrule
0.20  &  \makecell{$0.86$ \\ \scriptsize{$[0.84, 0.88]$}} &  \makecell{$3.97$ \\ \scriptsize{$[3.24, 4.73]$}} &  \makecell{$0.189$ \\ \scriptsize{$[0.182, 0.195]$}} &  \makecell{$0.56$ \\ \scriptsize{$[0.49, 0.63]$}} &  \makecell{$0.95$ \\ \scriptsize{$[0.94, 0.96]$}} \\ \midrule
0.22  &  \makecell{$0.86$ \\ \scriptsize{$[0.84, 0.89]$}} &  \makecell{$3.72$ \\ \scriptsize{$[3.01, 4.48]$}} &  \makecell{$0.190$ \\ \scriptsize{$[0.183, 0.196]$}} &  \makecell{$0.59$ \\ \scriptsize{$[0.53, 0.66]$}} &  \makecell{$0.94$ \\ \scriptsize{$[0.93, 0.96]$}} \\ \midrule
0.24  &  \makecell{$0.88$ \\ \scriptsize{$[0.85, 0.90]$}} &  \makecell{$3.47$ \\ \scriptsize{$[2.73, 4.25]$}} &  \makecell{$0.192$ \\ \scriptsize{$[0.185, 0.198]$}} &  \makecell{$0.65$ \\ \scriptsize{$[0.58, 0.72]$}} &  \makecell{$0.93$ \\ \scriptsize{$[0.91, 0.95]$}} \\ \midrule
0.26  &  \makecell{$0.87$ \\ \scriptsize{$[0.85, 0.89]$}} &  \makecell{$3.84$ \\ \scriptsize{$[3.11, 4.62]$}} &  \makecell{$0.190$ \\ \scriptsize{$[0.183, 0.197]$}} &  \makecell{$0.66$ \\ \scriptsize{$[0.59, 0.72]$}} &  \makecell{$0.90$ \\ \scriptsize{$[0.87, 0.92]$}} \\ \midrule
0.28  &  \makecell{$0.85$ \\ \scriptsize{$[0.83, 0.87]$}} &  \makecell{$4.18$ \\ \scriptsize{$[3.51, 4.89]$}} &  \makecell{$0.188$ \\ \scriptsize{$[0.181, 0.194]$}} &  \makecell{$0.66$ \\ \scriptsize{$[0.60, 0.73]$}} &  \makecell{$0.86$ \\ \scriptsize{$[0.83, 0.88]$}} \\ \midrule
0.30  &  \makecell{$0.85$ \\ \scriptsize{$[0.82, 0.87]$}} &  \makecell{$4.15$ \\ \scriptsize{$[3.48, 4.85]$}} &  \makecell{$0.186$ \\ \scriptsize{$[0.179, 0.193]$}} &  \makecell{$0.68$ \\ \scriptsize{$[0.62, 0.74]$}} &  \makecell{$0.84$ \\ \scriptsize{$[0.81, 0.86]$}} \\ 
\bottomrule
\end{tabular}

    \caption{8x8. Results for the metrics for selected noise levels. Shown in square brackets are bootstrapped $95\%$-confidence intervals.
    Here \emph{topo} stands for topographic similarity, \emph{conf} for conflict count, \emph{cont} for context independence, \emph{pos} for positional disentanglement and \emph{acc} for accuracy.
    }
    \label{app:tab:8x8}
\end{table}

\pagebreak
\subsection{Variable noise }\label{app:sec:variable_noise_experiments}

In this section, we present detailed results for variable noise experiments. 
More precisely, at the beginning of training the noise is kept at some initial level $\epsilon_0\in\{0.0, 0.15\}$ 
and after $T\in\{200, 700, 1200, 1700, 2200, 2700\}$ warmup network updates it is changed to a value, $\epsilon_T$, and kept there for the rest of the training.  
The results are presented in
Figure~\ref{fig:ridge_variable_noise_0.15} and Table \ref{app:tab:variable_noise_0.15} ($\epsilon_0=0.0, \epsilon_T=0.15$), 
Figure~\ref*{fig:ridge_variable_noise_0_0.1} and Table \ref{app:tab:variable_noise_0_0.1} ($\epsilon_0 = 0.0, \epsilon_T=0.1$), 
and
Figure~\ref{fig:ridge_REAL_0.15_0.1} and Table \ref{app:tab:REAL_0.15_0.1} ($\epsilon_0 = 0.15, \epsilon_T=0.1$).

For $\epsilon_0=0.0,\epsilon_T=0.15$ case, we see that topo increases from $0.85$ for $T=200$, to $0.9$ for $T=2700$, and the transition is reflected both in the density profile as well as the confidence intervals. The effect for $\epsilon_0=0.0, \epsilon_T=0.1$ is weaker and the variance in the results is quite high. There seems to be a negligable effect for the case $\epsilon_0 = 0.15, \epsilon_T=0.1$.

\begin{figure}[H]
    \centering
    \includegraphics[width=\textwidth]{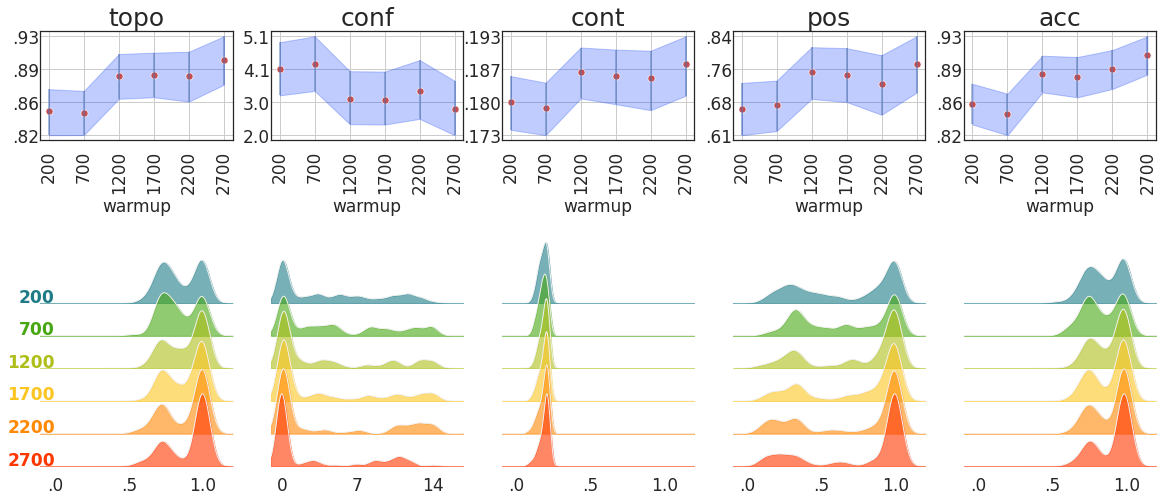}
    \caption{Variable noise with $\epsilon_0=0.0,\epsilon_T=0.15$. Top panel: average value of metrics for various warmup levels ($T$). The shaded area corresponds to bootstrapped $95\%$-confidence intervals for this estimator. Bottom panel: kernel density estimators for metrics and noise levels across seeds.
    Here \emph{topo} stands for topographic similarity, \emph{conf} for conflict count, \emph{cont} for context independence, \emph{pos} for positional disentanglement and \emph{acc} for accuracy.
    }
    \label{fig:ridge_variable_noise_0.15}
\end{figure}

\begin{figure}[H]
    \centering
    \includegraphics[width=\textwidth]{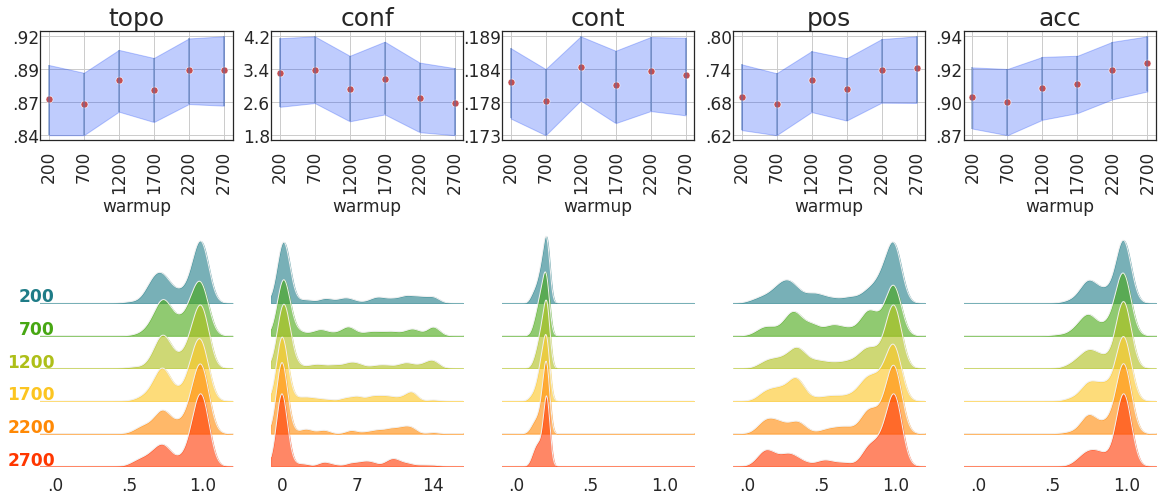}
    \caption{Variable noise $\epsilon_0=0.0, \epsilon_T=0.1$. Top panel: average value of metrics for various warmup levels ($T$). The shaded area corresponds to bootstrapped $95\%$-confidence intervals for this estimator. Bottom panel: kernel density estimators for metrics and noise levels across seeds.
    Here \emph{topo} stands for topographic similarity, \emph{conf} for conflict count, \emph{cont} for context independence, \emph{pos} for positional disentanglement and \emph{acc} for accuracy.
    }
    \label{fig:ridge_variable_noise_0_0.1}
\end{figure}

\begin{figure}[H]
    \centering
    \includegraphics[width=\textwidth]{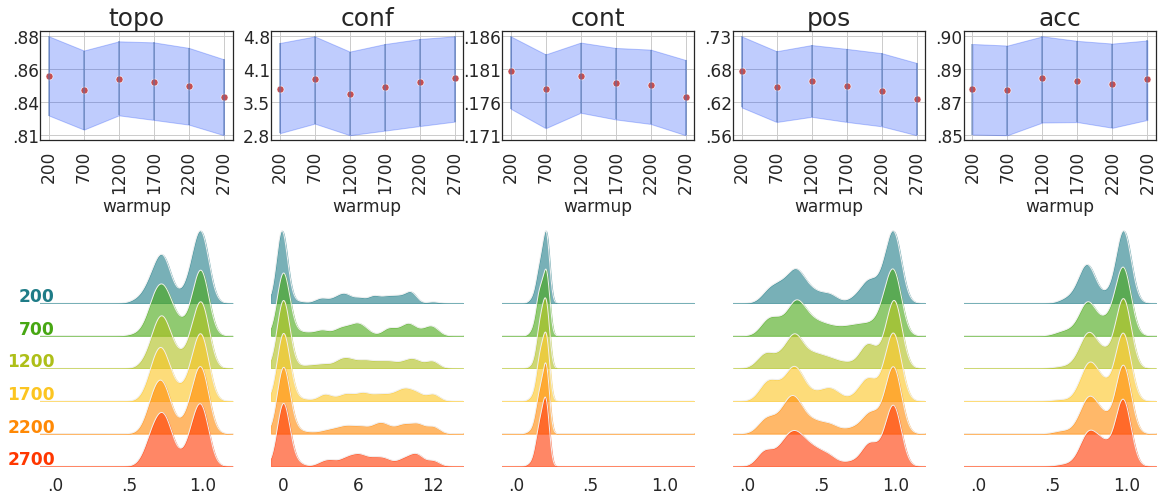}
    \caption{Variable noise $\epsilon_0=0.15, \epsilon_T=0.1$. Top panel: average value of metrics for various warmup levels ($T$). The shaded area corresponds to bootstrapped $95\%$-confidence intervals for this estimator. Bottom panel: kernel density estimators for metrics and noise levels across seeds.
    Here \emph{topo} stands for topographic similarity, \emph{conf} for conflict count, \emph{cont} for context independence, \emph{pos} for positional disentanglement and \emph{acc} for accuracy.
    }
    \label{fig:ridge_REAL_0.15_0.1}
\end{figure}

\begin{table}[H]
    \centering
    \begin{tabular}{l|ccccc}
\toprule
warmup &                                              topo &                                              conf &                                                 cont &                                               pos &                                               acc \\ \midrule
200    &  \makecell{$0.85$ \\ \scriptsize{$[0.82, 0.87]$}} &  \makecell{$4.06$ \\ \scriptsize{$[3.25, 4.89]$}} &  \makecell{$0.180$ \\ \scriptsize{$[0.175, 0.185]$}} &  \makecell{$0.67$ \\ \scriptsize{$[0.61, 0.73]$}} &  \makecell{$0.85$ \\ \scriptsize{$[0.83, 0.88]$}} \\ \midrule
700    &  \makecell{$0.85$ \\ \scriptsize{$[0.82, 0.87]$}} &  \makecell{$4.21$ \\ \scriptsize{$[3.39, 5.07]$}} &  \makecell{$0.179$ \\ \scriptsize{$[0.173, 0.184]$}} &  \makecell{$0.68$ \\ \scriptsize{$[0.62, 0.73]$}} &  \makecell{$0.84$ \\ \scriptsize{$[0.82, 0.87]$}} \\ \midrule
1200   &  \makecell{$0.89$ \\ \scriptsize{$[0.86, 0.91]$}} &  \makecell{$3.14$ \\ \scriptsize{$[2.37, 4.00]$}} &  \makecell{$0.186$ \\ \scriptsize{$[0.181, 0.191]$}} &  \makecell{$0.75$ \\ \scriptsize{$[0.69, 0.81]$}} &  \makecell{$0.89$ \\ \scriptsize{$[0.87, 0.91]$}} \\ \midrule
1700   &  \makecell{$0.89$ \\ \scriptsize{$[0.86, 0.91]$}} &  \makecell{$3.12$ \\ \scriptsize{$[2.36, 3.99]$}} &  \makecell{$0.185$ \\ \scriptsize{$[0.180, 0.190]$}} &  \makecell{$0.75$ \\ \scriptsize{$[0.68, 0.81]$}} &  \makecell{$0.88$ \\ \scriptsize{$[0.86, 0.91]$}} \\ \midrule
2200   &  \makecell{$0.89$ \\ \scriptsize{$[0.86, 0.91]$}} &  \makecell{$3.40$ \\ \scriptsize{$[2.53, 4.34]$}} &  \makecell{$0.185$ \\ \scriptsize{$[0.178, 0.190]$}} &  \makecell{$0.73$ \\ \scriptsize{$[0.65, 0.79]$}} &  \makecell{$0.89$ \\ \scriptsize{$[0.87, 0.92]$}} \\ \midrule
2700   &  \makecell{$0.90$ \\ \scriptsize{$[0.88, 0.93]$}} &  \makecell{$2.82$ \\ \scriptsize{$[2.03, 3.70]$}} &  \makecell{$0.188$ \\ \scriptsize{$[0.181, 0.193]$}} &  \makecell{$0.77$ \\ \scriptsize{$[0.71, 0.84]$}} &  \makecell{$0.91$ \\ \scriptsize{$[0.89, 0.93]$}} \\ 
\bottomrule
\end{tabular}

    \caption{Variable noise with $\epsilon_0=0.0,\epsilon_T=0.15$. Results for the metrics for various warmup levels ($T$). Shown in square brackets are bootstrapped $95\%$-confidence intervals.
    Here \emph{topo} stands for topographic similarity, \emph{conf} for conflict count, \emph{cont} for context independence, \emph{pos} for positional disentanglement and \emph{acc} for accuracy. }
    \label{app:tab:variable_noise_0.15}
\end{table}

\begin{table}
    \centering
    \begin{tabular}{l|ccccc}
\toprule
warmup &                                              topo &                                              conf &                                                 cont &                                               pos &                                               acc \\ \midrule
200    &  \makecell{$0.87$ \\ \scriptsize{$[0.84, 0.90]$}} &  \makecell{$3.28$ \\ \scriptsize{$[2.47, 4.12]$}} &  \makecell{$0.182$ \\ \scriptsize{$[0.176, 0.187]$}} &  \makecell{$0.69$ \\ \scriptsize{$[0.63, 0.75]$}} &  \makecell{$0.90$ \\ \scriptsize{$[0.88, 0.92]$}} \\ \midrule
700    &  \makecell{$0.86$ \\ \scriptsize{$[0.84, 0.89]$}} &  \makecell{$3.35$ \\ \scriptsize{$[2.56, 4.17]$}} &  \makecell{$0.179$ \\ \scriptsize{$[0.173, 0.184]$}} &  \makecell{$0.68$ \\ \scriptsize{$[0.62, 0.73]$}} &  \makecell{$0.90$ \\ \scriptsize{$[0.87, 0.92]$}} \\ \midrule
1200   &  \makecell{$0.88$ \\ \scriptsize{$[0.86, 0.91]$}} &  \makecell{$2.89$ \\ \scriptsize{$[2.12, 3.70]$}} &  \makecell{$0.184$ \\ \scriptsize{$[0.179, 0.189]$}} &  \makecell{$0.72$ \\ \scriptsize{$[0.66, 0.78]$}} &  \makecell{$0.91$ \\ \scriptsize{$[0.88, 0.93]$}} \\ \midrule
1700   &  \makecell{$0.88$ \\ \scriptsize{$[0.85, 0.90]$}} &  \makecell{$3.14$ \\ \scriptsize{$[2.28, 4.04]$}} &  \makecell{$0.181$ \\ \scriptsize{$[0.175, 0.187]$}} &  \makecell{$0.70$ \\ \scriptsize{$[0.65, 0.76]$}} &  \makecell{$0.91$ \\ \scriptsize{$[0.89, 0.93]$}} \\ \midrule
2200   &  \makecell{$0.89$ \\ \scriptsize{$[0.86, 0.92]$}} &  \makecell{$2.69$ \\ \scriptsize{$[1.85, 3.54]$}} &  \makecell{$0.183$ \\ \scriptsize{$[0.177, 0.189]$}} &  \makecell{$0.74$ \\ \scriptsize{$[0.68, 0.80]$}} &  \makecell{$0.92$ \\ \scriptsize{$[0.90, 0.94]$}} \\ \midrule
2700   &  \makecell{$0.89$ \\ \scriptsize{$[0.86, 0.92]$}} &  \makecell{$2.56$ \\ \scriptsize{$[1.78, 3.40]$}} &  \makecell{$0.183$ \\ \scriptsize{$[0.176, 0.189]$}} &  \makecell{$0.74$ \\ \scriptsize{$[0.68, 0.80]$}} &  \makecell{$0.92$ \\ \scriptsize{$[0.90, 0.94]$}} \\ 
\bottomrule
\end{tabular}

    \caption{Variable noise with $\epsilon_0=0.0,\epsilon_T=0.1$. Results for the metrics for various warmup levels ($T$). Shown in square brackets are bootstrapped $95\%$-confidence intervals.
    Here \emph{topo} stands for topographic similarity, \emph{conf} for conflict count, \emph{cont} for context independence, \emph{pos} for positional disentanglement and \emph{acc} for accuracy.
    }
    \label{app:tab:variable_noise_0_0.1}
\end{table}

\begin{table}
    \centering
    \begin{tabular}{l|ccccc}
\toprule
warmup &                                              topo &                                              conf &                                                 cont &                                               pos &                                               acc \\ \midrule
200    &  \makecell{$0.86$ \\ \scriptsize{$[0.83, 0.88]$}} &  \makecell{$3.73$ \\ \scriptsize{$[2.85, 4.63]$}} &  \makecell{$0.181$ \\ \scriptsize{$[0.175, 0.186]$}} &  \makecell{$0.67$ \\ \scriptsize{$[0.61, 0.73]$}} &  \makecell{$0.88$ \\ \scriptsize{$[0.85, 0.90]$}} \\ \midrule
700    &  \makecell{$0.85$ \\ \scriptsize{$[0.82, 0.87]$}} &  \makecell{$3.91$ \\ \scriptsize{$[3.04, 4.76]$}} &  \makecell{$0.178$ \\ \scriptsize{$[0.172, 0.183]$}} &  \makecell{$0.65$ \\ \scriptsize{$[0.59, 0.71]$}} &  \makecell{$0.88$ \\ \scriptsize{$[0.85, 0.90]$}} \\ \midrule
1200   &  \makecell{$0.85$ \\ \scriptsize{$[0.83, 0.88]$}} &  \makecell{$3.63$ \\ \scriptsize{$[2.81, 4.46]$}} &  \makecell{$0.180$ \\ \scriptsize{$[0.175, 0.185]$}} &  \makecell{$0.66$ \\ \scriptsize{$[0.59, 0.72]$}} &  \makecell{$0.88$ \\ \scriptsize{$[0.86, 0.90]$}} \\ \midrule
1700   &  \makecell{$0.85$ \\ \scriptsize{$[0.82, 0.88]$}} &  \makecell{$3.77$ \\ \scriptsize{$[2.90, 4.61]$}} &  \makecell{$0.179$ \\ \scriptsize{$[0.174, 0.184]$}} &  \makecell{$0.65$ \\ \scriptsize{$[0.59, 0.71]$}} &  \makecell{$0.88$ \\ \scriptsize{$[0.86, 0.90]$}} \\ \midrule
2200   &  \makecell{$0.85$ \\ \scriptsize{$[0.82, 0.87]$}} &  \makecell{$3.86$ \\ \scriptsize{$[2.99, 4.71]$}} &  \makecell{$0.179$ \\ \scriptsize{$[0.173, 0.184]$}} &  \makecell{$0.64$ \\ \scriptsize{$[0.58, 0.70]$}} &  \makecell{$0.88$ \\ \scriptsize{$[0.85, 0.90]$}} \\ \midrule
2700   &  \makecell{$0.84$ \\ \scriptsize{$[0.81, 0.87]$}} &  \makecell{$3.94$ \\ \scriptsize{$[3.08, 4.76]$}} &  \makecell{$0.177$ \\ \scriptsize{$[0.171, 0.183]$}} &  \makecell{$0.62$ \\ \scriptsize{$[0.56, 0.69]$}} &  \makecell{$0.88$ \\ \scriptsize{$[0.86, 0.90]$}} \\ 
\bottomrule
\end{tabular}

    \caption{Variable noise with $\epsilon_0=0.15, \epsilon_T=0.1$. Results for the metrics for various warmup levels ($T$). Shown in square brackets are bootstrapped $95\%$-confidence intervals.
    Here \emph{topo} stands for topographic similarity, \emph{conf} for conflict count, \emph{cont} for context independence, \emph{pos} for positional disentanglement and \emph{acc} for accuracy.
    }
    \label{app:tab:REAL_0.15_0.1}
\end{table}

\clearpage
\newpage
\subsection{Sensitivity with respect to small architecture changes}\label{sec:architecture_experiments}

This section provides details for experiments aiming to check the impact of small architecture change in CNN (Figure \ref{fig:ridge_big_cnn}, Table \ref{app:tab:big_cnn}) and the dense layers in the agents' network (Figure \ref{fig:ridge_big_dense}, Table \ref{app:tab:big_dense}). 
In the former experiment, we change the number of filters in the sender's CNN architecture from two layers with $8$ filters, to two layers with $16$ filters. This results in the slight change of topographic similarity profile with the highest average value of $0.86$ for noise range $0.14$ significantly outperforming the zero-noise case ($0.79$). 
For the second experiment, we added a dense layer (with 64 neurons) to the receiver's architecture. This again shows improvement of compositionality due to noise, although the noise in the range $[0.08, 0.16]$ performs roughly the same.

\begin{figure}[H]
    \centering
    \includegraphics[width=\textwidth]{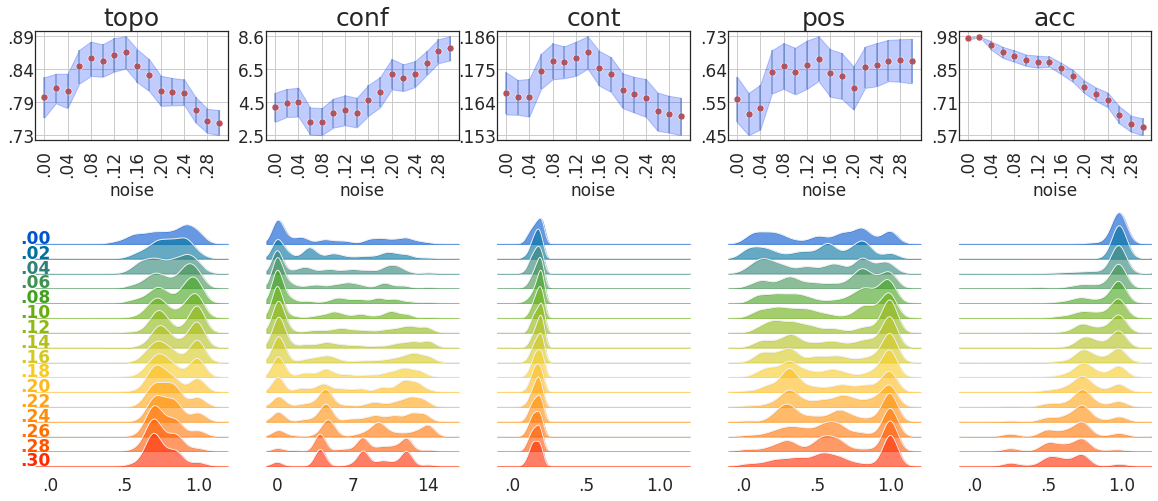}
    \caption{Bigger CNN. Top panel: average value of metrics for various noise levels. The shaded area corresponds to bootstrapped $95\%$-confidence intervals for this estimator. Bottom panel: kernel density estimators for metrics and noise levels across seeds.
    Here \emph{topo} stands for topographic similarity, \emph{conf} for conflict count, \emph{cont} for context independence, \emph{pos} for positional disentanglement and \emph{acc} for accuracy.
    }
    \label{fig:ridge_big_cnn}
\end{figure}

\begin{figure}[H]
    \centering
    \includegraphics[width=\textwidth]{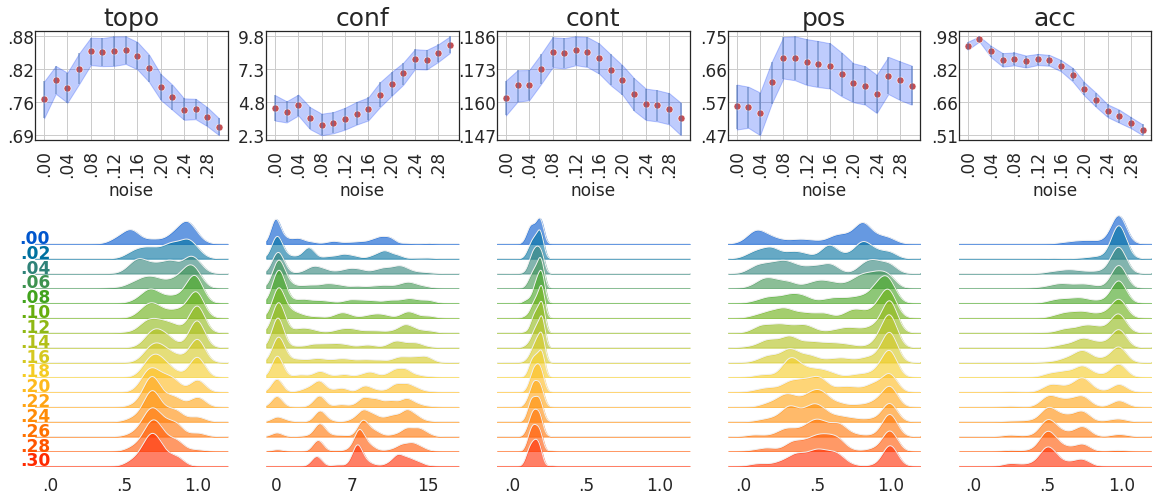}
    \caption{Bigger dense layer. Top panel: average value of metrics for various noise levels. The shaded area corresponds to bootstrapped $95\%$-confidence intervals for this estimator. Bottom panel: kernel density estimators for metrics and noise levels across seeds.
    Here \emph{topo} stands for topographic similarity, \emph{conf} for conflict count, \emph{cont} for context independence, \emph{pos} for positional disentanglement and \emph{acc} for accuracy.
    }
    \label{fig:ridge_big_dense}
\end{figure}

\begin{table}
    \centering
    \begin{tabular}{l|ccccc}
\toprule
noise &                                              topo &                                              conf &                                                 cont &                                               pos &                                               acc \\ \midrule
0.00  &  \makecell{$0.79$ \\ \scriptsize{$[0.76, 0.82]$}} &  \makecell{$4.23$ \\ \scriptsize{$[3.34, 5.11]$}} &  \makecell{$0.168$ \\ \scriptsize{$[0.160, 0.174]$}} &  \makecell{$0.56$ \\ \scriptsize{$[0.49, 0.62]$}} &  \makecell{$0.98$ \\ \scriptsize{$[0.97, 0.98]$}} \\ \midrule
0.02  &  \makecell{$0.81$ \\ \scriptsize{$[0.78, 0.83]$}} &  \makecell{$4.47$ \\ \scriptsize{$[3.63, 5.34]$}} &  \makecell{$0.166$ \\ \scriptsize{$[0.160, 0.172]$}} &  \makecell{$0.51$ \\ \scriptsize{$[0.45, 0.57]$}} &  \makecell{$0.98$ \\ \scriptsize{$[0.98, 0.98]$}} \\ \midrule
0.04  &  \makecell{$0.80$ \\ \scriptsize{$[0.78, 0.83]$}} &  \makecell{$4.53$ \\ \scriptsize{$[3.66, 5.42]$}} &  \makecell{$0.166$ \\ \scriptsize{$[0.160, 0.172]$}} &  \makecell{$0.53$ \\ \scriptsize{$[0.47, 0.60]$}} &  \makecell{$0.95$ \\ \scriptsize{$[0.93, 0.96]$}} \\ \midrule
0.06  &  \makecell{$0.84$ \\ \scriptsize{$[0.81, 0.87]$}} &  \makecell{$3.35$ \\ \scriptsize{$[2.56, 4.22]$}} &  \makecell{$0.175$ \\ \scriptsize{$[0.169, 0.181]$}} &  \makecell{$0.63$ \\ \scriptsize{$[0.57, 0.69]$}} &  \makecell{$0.92$ \\ \scriptsize{$[0.89, 0.94]$}} \\ \midrule
0.08  &  \makecell{$0.85$ \\ \scriptsize{$[0.83, 0.88]$}} &  \makecell{$3.32$ \\ \scriptsize{$[2.52, 4.18]$}} &  \makecell{$0.178$ \\ \scriptsize{$[0.172, 0.184]$}} &  \makecell{$0.65$ \\ \scriptsize{$[0.58, 0.71]$}} &  \makecell{$0.90$ \\ \scriptsize{$[0.88, 0.93]$}} \\ \midrule
0.10  &  \makecell{$0.85$ \\ \scriptsize{$[0.83, 0.88]$}} &  \makecell{$3.86$ \\ \scriptsize{$[2.99, 4.76]$}} &  \makecell{$0.178$ \\ \scriptsize{$[0.172, 0.183]$}} &  \makecell{$0.63$ \\ \scriptsize{$[0.57, 0.70]$}} &  \makecell{$0.89$ \\ \scriptsize{$[0.86, 0.91]$}} \\ \midrule
0.12  &  \makecell{$0.86$ \\ \scriptsize{$[0.83, 0.89]$}} &  \makecell{$4.06$ \\ \scriptsize{$[3.13, 4.95]$}} &  \makecell{$0.179$ \\ \scriptsize{$[0.173, 0.185]$}} &  \makecell{$0.65$ \\ \scriptsize{$[0.58, 0.72]$}} &  \makecell{$0.88$ \\ \scriptsize{$[0.85, 0.90]$}} \\ \midrule
0.14  &  \makecell{$0.86$ \\ \scriptsize{$[0.84, 0.89]$}} &  \makecell{$3.90$ \\ \scriptsize{$[3.00, 4.81]$}} &  \makecell{$0.181$ \\ \scriptsize{$[0.176, 0.186]$}} &  \makecell{$0.67$ \\ \scriptsize{$[0.61, 0.73]$}} &  \makecell{$0.88$ \\ \scriptsize{$[0.86, 0.90]$}} \\ \midrule
0.16  &  \makecell{$0.84$ \\ \scriptsize{$[0.82, 0.87]$}} &  \makecell{$4.64$ \\ \scriptsize{$[3.70, 5.59]$}} &  \makecell{$0.176$ \\ \scriptsize{$[0.170, 0.182]$}} &  \makecell{$0.63$ \\ \scriptsize{$[0.56, 0.70]$}} &  \makecell{$0.85$ \\ \scriptsize{$[0.83, 0.87]$}} \\ \midrule
0.18  &  \makecell{$0.83$ \\ \scriptsize{$[0.80, 0.85]$}} &  \makecell{$5.17$ \\ \scriptsize{$[4.22, 6.08]$}} &  \makecell{$0.174$ \\ \scriptsize{$[0.168, 0.180]$}} &  \makecell{$0.62$ \\ \scriptsize{$[0.56, 0.69]$}} &  \makecell{$0.82$ \\ \scriptsize{$[0.80, 0.85]$}} \\ \midrule
0.20  &  \makecell{$0.80$ \\ \scriptsize{$[0.78, 0.83]$}} &  \makecell{$6.26$ \\ \scriptsize{$[5.32, 7.17]$}} &  \makecell{$0.168$ \\ \scriptsize{$[0.162, 0.174]$}} &  \makecell{$0.59$ \\ \scriptsize{$[0.53, 0.65]$}} &  \makecell{$0.77$ \\ \scriptsize{$[0.75, 0.80]$}} \\ \midrule
0.22  &  \makecell{$0.80$ \\ \scriptsize{$[0.78, 0.82]$}} &  \makecell{$6.04$ \\ \scriptsize{$[5.21, 6.85]$}} &  \makecell{$0.167$ \\ \scriptsize{$[0.161, 0.173]$}} &  \makecell{$0.65$ \\ \scriptsize{$[0.59, 0.70]$}} &  \makecell{$0.74$ \\ \scriptsize{$[0.72, 0.77]$}} \\ \midrule
0.24  &  \makecell{$0.80$ \\ \scriptsize{$[0.78, 0.82]$}} &  \makecell{$6.26$ \\ \scriptsize{$[5.46, 7.04]$}} &  \makecell{$0.166$ \\ \scriptsize{$[0.160, 0.172]$}} &  \makecell{$0.65$ \\ \scriptsize{$[0.59, 0.71]$}} &  \makecell{$0.72$ \\ \scriptsize{$[0.69, 0.75]$}} \\ \midrule
0.26  &  \makecell{$0.77$ \\ \scriptsize{$[0.75, 0.79]$}} &  \makecell{$6.95$ \\ \scriptsize{$[6.19, 7.66]$}} &  \makecell{$0.161$ \\ \scriptsize{$[0.155, 0.168]$}} &  \makecell{$0.66$ \\ \scriptsize{$[0.61, 0.72]$}} &  \makecell{$0.66$ \\ \scriptsize{$[0.62, 0.69]$}} \\ \midrule
0.28  &  \makecell{$0.76$ \\ \scriptsize{$[0.74, 0.78]$}} &  \makecell{$7.65$ \\ \scriptsize{$[6.91, 8.38]$}} &  \makecell{$0.161$ \\ \scriptsize{$[0.154, 0.167]$}} &  \makecell{$0.66$ \\ \scriptsize{$[0.60, 0.72]$}} &  \makecell{$0.62$ \\ \scriptsize{$[0.59, 0.66]$}} \\ \midrule
0.30  &  \makecell{$0.75$ \\ \scriptsize{$[0.73, 0.77]$}} &  \makecell{$7.81$ \\ \scriptsize{$[7.08, 8.55]$}} &  \makecell{$0.160$ \\ \scriptsize{$[0.153, 0.166]$}} &  \makecell{$0.66$ \\ \scriptsize{$[0.60, 0.73]$}} &  \makecell{$0.61$ \\ \scriptsize{$[0.57, 0.64]$}} \\ 
\bottomrule
\end{tabular}

    \caption{Bigger CNN. Results for the metrics for selected noise levels. Shown in square brackets are bootstrapped $95\%$-confidence intervals.
    Here \emph{topo} stands for topographic similarity, \emph{conf} for conflict count, \emph{cont} for context independence, \emph{pos} for positional disentanglement and \emph{acc} for accuracy.
    }
    \label{app:tab:big_cnn}
\end{table}

\begin{table}
    \centering
    \begin{tabular}{l|ccccc}
\toprule
noise &                                              topo &                                              conf &                                                 cont &                                               pos &                                               acc \\ \midrule
0.00  &  \makecell{$0.76$ \\ \scriptsize{$[0.72, 0.80]$}} &  \makecell{$4.42$ \\ \scriptsize{$[3.45, 5.38]$}} &  \makecell{$0.162$ \\ \scriptsize{$[0.155, 0.169]$}} &  \makecell{$0.56$ \\ \scriptsize{$[0.49, 0.62]$}} &  \makecell{$0.93$ \\ \scriptsize{$[0.91, 0.95]$}} \\ \midrule
0.02  &  \makecell{$0.80$ \\ \scriptsize{$[0.77, 0.83]$}} &  \makecell{$4.10$ \\ \scriptsize{$[3.30, 4.92]$}} &  \makecell{$0.167$ \\ \scriptsize{$[0.160, 0.173]$}} &  \makecell{$0.55$ \\ \scriptsize{$[0.50, 0.61]$}} &  \makecell{$0.96$ \\ \scriptsize{$[0.95, 0.98]$}} \\ \midrule
0.04  &  \makecell{$0.78$ \\ \scriptsize{$[0.76, 0.81]$}} &  \makecell{$4.65$ \\ \scriptsize{$[3.83, 5.42]$}} &  \makecell{$0.167$ \\ \scriptsize{$[0.161, 0.173]$}} &  \makecell{$0.54$ \\ \scriptsize{$[0.47, 0.59]$}} &  \makecell{$0.90$ \\ \scriptsize{$[0.88, 0.93]$}} \\ \midrule
0.06  &  \makecell{$0.82$ \\ \scriptsize{$[0.79, 0.85]$}} &  \makecell{$3.62$ \\ \scriptsize{$[2.80, 4.48]$}} &  \makecell{$0.173$ \\ \scriptsize{$[0.167, 0.179]$}} &  \makecell{$0.62$ \\ \scriptsize{$[0.56, 0.68]$}} &  \makecell{$0.86$ \\ \scriptsize{$[0.83, 0.90]$}} \\ \midrule
0.08  &  \makecell{$0.86$ \\ \scriptsize{$[0.83, 0.88]$}} &  \makecell{$3.14$ \\ \scriptsize{$[2.33, 3.97]$}} &  \makecell{$0.180$ \\ \scriptsize{$[0.174, 0.185]$}} &  \makecell{$0.69$ \\ \scriptsize{$[0.63, 0.75]$}} &  \makecell{$0.87$ \\ \scriptsize{$[0.84, 0.90]$}} \\ \midrule
0.10  &  \makecell{$0.85$ \\ \scriptsize{$[0.83, 0.88]$}} &  \makecell{$3.27$ \\ \scriptsize{$[2.48, 4.09]$}} &  \makecell{$0.179$ \\ \scriptsize{$[0.173, 0.185]$}} &  \makecell{$0.69$ \\ \scriptsize{$[0.63, 0.75]$}} &  \makecell{$0.86$ \\ \scriptsize{$[0.83, 0.88]$}} \\ \midrule
0.12  &  \makecell{$0.86$ \\ \scriptsize{$[0.83, 0.88]$}} &  \makecell{$3.56$ \\ \scriptsize{$[2.76, 4.42]$}} &  \makecell{$0.180$ \\ \scriptsize{$[0.175, 0.186]$}} &  \makecell{$0.68$ \\ \scriptsize{$[0.62, 0.74]$}} &  \makecell{$0.87$ \\ \scriptsize{$[0.84, 0.89]$}} \\ \midrule
0.14  &  \makecell{$0.86$ \\ \scriptsize{$[0.83, 0.88]$}} &  \makecell{$3.95$ \\ \scriptsize{$[3.14, 4.87]$}} &  \makecell{$0.180$ \\ \scriptsize{$[0.174, 0.186]$}} &  \makecell{$0.67$ \\ \scriptsize{$[0.61, 0.74]$}} &  \makecell{$0.86$ \\ \scriptsize{$[0.84, 0.89]$}} \\ \midrule
0.16  &  \makecell{$0.85$ \\ \scriptsize{$[0.82, 0.87]$}} &  \makecell{$4.31$ \\ \scriptsize{$[3.46, 5.21]$}} &  \makecell{$0.177$ \\ \scriptsize{$[0.171, 0.183]$}} &  \makecell{$0.67$ \\ \scriptsize{$[0.61, 0.73]$}} &  \makecell{$0.84$ \\ \scriptsize{$[0.81, 0.86]$}} \\ \midrule
0.18  &  \makecell{$0.82$ \\ \scriptsize{$[0.80, 0.85]$}} &  \makecell{$5.42$ \\ \scriptsize{$[4.46, 6.36]$}} &  \makecell{$0.173$ \\ \scriptsize{$[0.167, 0.179]$}} &  \makecell{$0.65$ \\ \scriptsize{$[0.58, 0.71]$}} &  \makecell{$0.79$ \\ \scriptsize{$[0.76, 0.83]$}} \\ \midrule
0.20  &  \makecell{$0.79$ \\ \scriptsize{$[0.76, 0.81]$}} &  \makecell{$6.25$ \\ \scriptsize{$[5.33, 7.14]$}} &  \makecell{$0.169$ \\ \scriptsize{$[0.163, 0.175]$}} &  \makecell{$0.62$ \\ \scriptsize{$[0.57, 0.68]$}} &  \makecell{$0.73$ \\ \scriptsize{$[0.69, 0.76]$}} \\ \midrule
0.22  &  \makecell{$0.77$ \\ \scriptsize{$[0.74, 0.79]$}} &  \makecell{$7.04$ \\ \scriptsize{$[6.17, 7.85]$}} &  \makecell{$0.163$ \\ \scriptsize{$[0.157, 0.170]$}} &  \makecell{$0.61$ \\ \scriptsize{$[0.56, 0.67]$}} &  \makecell{$0.67$ \\ \scriptsize{$[0.64, 0.71]$}} \\ \midrule
0.24  &  \makecell{$0.74$ \\ \scriptsize{$[0.72, 0.76]$}} &  \makecell{$8.14$ \\ \scriptsize{$[7.36, 8.87]$}} &  \makecell{$0.159$ \\ \scriptsize{$[0.153, 0.166]$}} &  \makecell{$0.59$ \\ \scriptsize{$[0.54, 0.64]$}} &  \makecell{$0.62$ \\ \scriptsize{$[0.59, 0.65]$}} \\ \midrule
0.26  &  \makecell{$0.74$ \\ \scriptsize{$[0.72, 0.76]$}} &  \makecell{$8.08$ \\ \scriptsize{$[7.33, 8.82]$}} &  \makecell{$0.159$ \\ \scriptsize{$[0.153, 0.165]$}} &  \makecell{$0.64$ \\ \scriptsize{$[0.59, 0.69]$}} &  \makecell{$0.60$ \\ \scriptsize{$[0.57, 0.63]$}} \\ \midrule
0.28  &  \makecell{$0.73$ \\ \scriptsize{$[0.71, 0.75]$}} &  \makecell{$8.58$ \\ \scriptsize{$[7.90, 9.24]$}} &  \makecell{$0.158$ \\ \scriptsize{$[0.152, 0.164]$}} &  \makecell{$0.63$ \\ \scriptsize{$[0.57, 0.68]$}} &  \makecell{$0.57$ \\ \scriptsize{$[0.54, 0.59]$}} \\ \midrule
0.30  &  \makecell{$0.71$ \\ \scriptsize{$[0.69, 0.73]$}} &  \makecell{$9.21$ \\ \scriptsize{$[8.58, 9.85]$}} &  \makecell{$0.154$ \\ \scriptsize{$[0.147, 0.160]$}} &  \makecell{$0.61$ \\ \scriptsize{$[0.56, 0.67]$}} &  \makecell{$0.53$ \\ \scriptsize{$[0.51, 0.56]$}} \\ 
\bottomrule
\end{tabular}

    \caption{Bigger dense layer. Results for the metrics for selected noise levels. Shown in square brackets are bootstrapped $95\%$-confidence intervals.
    Here \emph{topo} stands for topographic similarity, \emph{conf} for conflict count, \emph{cont} for context independence, \emph{pos} for positional disentanglement and \emph{acc} for accuracy.
    }
    \label{app:tab:big_dense}
\end{table}

\clearpage
\newpage
\subsection{Sensitivity to noisy channel implementation}
\label{sec:alternative_noise_design_experiments}
In this section, we provide details for 
experiments with alternative noisy channel implementation, see Figure \ref{fig:ridge_alternative_noise} and Table \ref{app:tab:alternative_noise}. 
Notice that the range of noise values is different from the main experiment (smaller by the order of magnitude).
The results suggest that the small values of noise can help, while the larger noise levels lead to a decline in compositionality.

\begin{figure}[H]
    \centering
    \includegraphics[width=\textwidth]{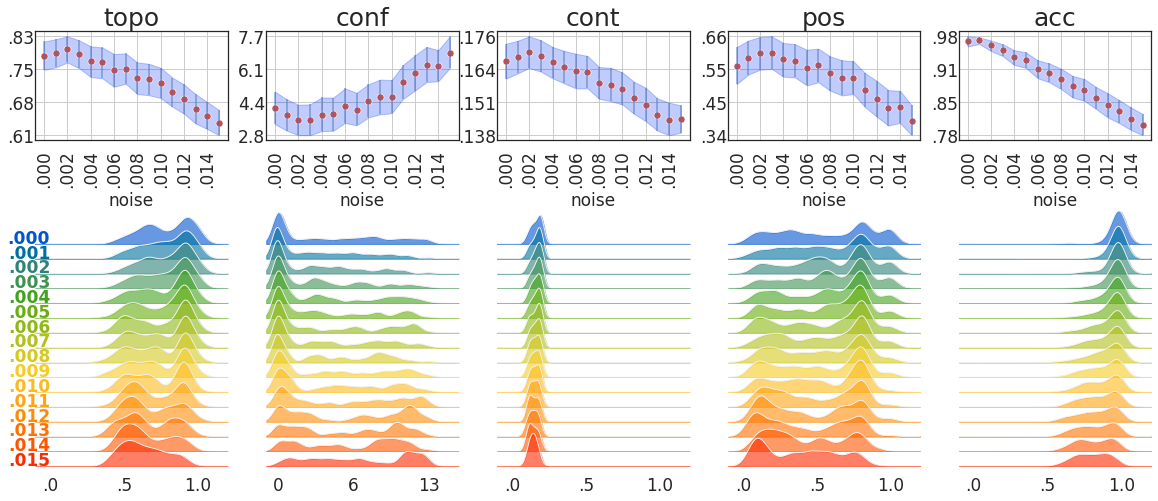}
    \caption{Alternative noise design. Top panel: average value of metrics for various noise levels. The shaded area corresponds to bootstrapped $95\%$-confidence intervals for this estimator. Bottom panel: kernel density estimators for metrics and noise levels across seeds.
    Here \emph{topo} stands for topographic similarity, \emph{conf} for conflict count, \emph{cont} for context independence, \emph{pos} for positional disentanglement and \emph{acc} for accuracy.
    }
    \label{fig:ridge_alternative_noise}
\end{figure}

\begin{table}
    \centering
    \begin{tabular}{l|ccccc}
\toprule
noise &                                              topo &                                              conf &                                                 cont &                                               pos &                                               acc \\ \midrule
0.000 &  \makecell{$0.79$ \\ \scriptsize{$[0.75, 0.82]$}} &  \makecell{$4.13$ \\ \scriptsize{$[3.31, 4.95]$}} &  \makecell{$0.167$ \\ \scriptsize{$[0.160, 0.173]$}} &  \makecell{$0.56$ \\ \scriptsize{$[0.50, 0.62]$}} &  \makecell{$0.97$ \\ \scriptsize{$[0.96, 0.98]$}} \\ \midrule
0.001 &  \makecell{$0.79$ \\ \scriptsize{$[0.76, 0.82]$}} &  \makecell{$3.77$ \\ \scriptsize{$[3.00, 4.57]$}} &  \makecell{$0.168$ \\ \scriptsize{$[0.162, 0.174]$}} &  \makecell{$0.59$ \\ \scriptsize{$[0.53, 0.64]$}} &  \makecell{$0.97$ \\ \scriptsize{$[0.96, 0.98]$}} \\ \midrule
0.002 &  \makecell{$0.80$ \\ \scriptsize{$[0.77, 0.83]$}} &  \makecell{$3.51$ \\ \scriptsize{$[2.77, 4.29]$}} &  \makecell{$0.170$ \\ \scriptsize{$[0.164, 0.176]$}} &  \makecell{$0.60$ \\ \scriptsize{$[0.55, 0.66]$}} &  \makecell{$0.96$ \\ \scriptsize{$[0.95, 0.97]$}} \\ \midrule
0.003 &  \makecell{$0.79$ \\ \scriptsize{$[0.76, 0.82]$}} &  \makecell{$3.52$ \\ \scriptsize{$[2.77, 4.31]$}} &  \makecell{$0.169$ \\ \scriptsize{$[0.163, 0.175]$}} &  \makecell{$0.61$ \\ \scriptsize{$[0.55, 0.66]$}} &  \makecell{$0.95$ \\ \scriptsize{$[0.94, 0.96]$}} \\ \midrule
0.004 &  \makecell{$0.77$ \\ \scriptsize{$[0.74, 0.81]$}} &  \makecell{$3.80$ \\ \scriptsize{$[2.97, 4.66]$}} &  \makecell{$0.166$ \\ \scriptsize{$[0.160, 0.172]$}} &  \makecell{$0.58$ \\ \scriptsize{$[0.53, 0.64]$}} &  \makecell{$0.94$ \\ \scriptsize{$[0.92, 0.95]$}} \\ \midrule
0.005 &  \makecell{$0.77$ \\ \scriptsize{$[0.73, 0.80]$}} &  \makecell{$3.81$ \\ \scriptsize{$[2.99, 4.65]$}} &  \makecell{$0.165$ \\ \scriptsize{$[0.158, 0.171]$}} &  \makecell{$0.58$ \\ \scriptsize{$[0.53, 0.63]$}} &  \makecell{$0.93$ \\ \scriptsize{$[0.91, 0.95]$}} \\ \midrule
0.006 &  \makecell{$0.75$ \\ \scriptsize{$[0.72, 0.79]$}} &  \makecell{$4.21$ \\ \scriptsize{$[3.38, 5.11]$}} &  \makecell{$0.163$ \\ \scriptsize{$[0.156, 0.170]$}} &  \makecell{$0.56$ \\ \scriptsize{$[0.50, 0.61]$}} &  \makecell{$0.91$ \\ \scriptsize{$[0.89, 0.93]$}} \\ \midrule
0.007 &  \makecell{$0.76$ \\ \scriptsize{$[0.72, 0.79]$}} &  \makecell{$4.05$ \\ \scriptsize{$[3.26, 4.89]$}} &  \makecell{$0.163$ \\ \scriptsize{$[0.156, 0.169]$}} &  \makecell{$0.56$ \\ \scriptsize{$[0.51, 0.62]$}} &  \makecell{$0.91$ \\ \scriptsize{$[0.89, 0.92]$}} \\ \midrule
0.008 &  \makecell{$0.74$ \\ \scriptsize{$[0.70, 0.77]$}} &  \makecell{$4.48$ \\ \scriptsize{$[3.65, 5.31]$}} &  \makecell{$0.158$ \\ \scriptsize{$[0.152, 0.164]$}} &  \makecell{$0.54$ \\ \scriptsize{$[0.48, 0.59]$}} &  \makecell{$0.89$ \\ \scriptsize{$[0.87, 0.91]$}} \\ \midrule
0.009 &  \makecell{$0.73$ \\ \scriptsize{$[0.70, 0.77]$}} &  \makecell{$4.68$ \\ \scriptsize{$[3.81, 5.55]$}} &  \makecell{$0.158$ \\ \scriptsize{$[0.151, 0.164]$}} &  \makecell{$0.52$ \\ \scriptsize{$[0.47, 0.58]$}} &  \makecell{$0.88$ \\ \scriptsize{$[0.85, 0.90]$}} \\ \midrule
0.010 &  \makecell{$0.72$ \\ \scriptsize{$[0.69, 0.76]$}} &  \makecell{$4.70$ \\ \scriptsize{$[3.87, 5.53]$}} &  \makecell{$0.156$ \\ \scriptsize{$[0.150, 0.162]$}} &  \makecell{$0.52$ \\ \scriptsize{$[0.47, 0.58]$}} &  \makecell{$0.87$ \\ \scriptsize{$[0.85, 0.89]$}} \\ \midrule
0.011 &  \makecell{$0.70$ \\ \scriptsize{$[0.67, 0.74]$}} &  \makecell{$5.43$ \\ \scriptsize{$[4.59, 6.30]$}} &  \makecell{$0.153$ \\ \scriptsize{$[0.147, 0.159]$}} &  \makecell{$0.48$ \\ \scriptsize{$[0.43, 0.54]$}} &  \makecell{$0.85$ \\ \scriptsize{$[0.83, 0.88]$}} \\ \midrule
0.012 &  \makecell{$0.69$ \\ \scriptsize{$[0.65, 0.72]$}} &  \makecell{$5.87$ \\ \scriptsize{$[5.00, 6.76]$}} &  \makecell{$0.150$ \\ \scriptsize{$[0.144, 0.157]$}} &  \makecell{$0.46$ \\ \scriptsize{$[0.40, 0.51]$}} &  \makecell{$0.84$ \\ \scriptsize{$[0.81, 0.86]$}} \\ \midrule
0.013 &  \makecell{$0.67$ \\ \scriptsize{$[0.63, 0.70]$}} &  \makecell{$6.29$ \\ \scriptsize{$[5.42, 7.19]$}} &  \makecell{$0.146$ \\ \scriptsize{$[0.140, 0.152]$}} &  \makecell{$0.43$ \\ \scriptsize{$[0.37, 0.48]$}} &  \makecell{$0.83$ \\ \scriptsize{$[0.80, 0.85]$}} \\ \midrule
0.014 &  \makecell{$0.65$ \\ \scriptsize{$[0.62, 0.68]$}} &  \makecell{$6.23$ \\ \scriptsize{$[5.47, 7.03]$}} &  \makecell{$0.144$ \\ \scriptsize{$[0.138, 0.151]$}} &  \makecell{$0.43$ \\ \scriptsize{$[0.38, 0.48]$}} &  \makecell{$0.81$ \\ \scriptsize{$[0.79, 0.83]$}} \\ \midrule
0.015 &  \makecell{$0.63$ \\ \scriptsize{$[0.61, 0.66]$}} &  \makecell{$6.91$ \\ \scriptsize{$[6.13, 7.72]$}} &  \makecell{$0.145$ \\ \scriptsize{$[0.139, 0.150]$}} &  \makecell{$0.39$ \\ \scriptsize{$[0.34, 0.44]$}} &  \makecell{$0.80$ \\ \scriptsize{$[0.78, 0.82]$}} \\ 
\bottomrule
\end{tabular}

    \caption{Alternative noise design. Results for the metrics for selected noise levels. Shown in square brackets are bootstrapped $95\%$-confidence intervals.
    Here \emph{topo} stands for topographic similarity, \emph{conf} for conflict count, \emph{cont} for context independence, \emph{pos} for positional disentanglement and \emph{acc} for accuracy.
    }
    \label{app:tab:alternative_noise}
\end{table}

\clearpage
\newpage
\subsection{Longer message}\label{sec:longer_message}
Here we provide details for longer message experiments, see Figure~\ref{fig:ridge_message_len_3}, Table~\ref{app:tab:message_len_3} and Figure~\ref{fig:ridge_message_len_4}, Table~\ref{app:tab:message_len_4} for message lengths three and four, respectively. The setup expands upon the main experiment by including floor color
\footnote{We took the floor color feature to take values $0, 1, 2, 3$ (out of $10$ possible). The resulting dataset contains $76800$ images. } 
for the former and, additionally, wall color
\footnote{We took the wall color feature to take values $0, 1, 2, 3$ (out of $10$ possible). The resulting dataset contains $30720$ images. } 
for the latter. The overall levels of compositionality metrics decline when compared with the main experiment, however, the general picture that noise improves compositionality remains intact.

\begin{figure}[H]
    \centering
    \includegraphics[width=\textwidth]{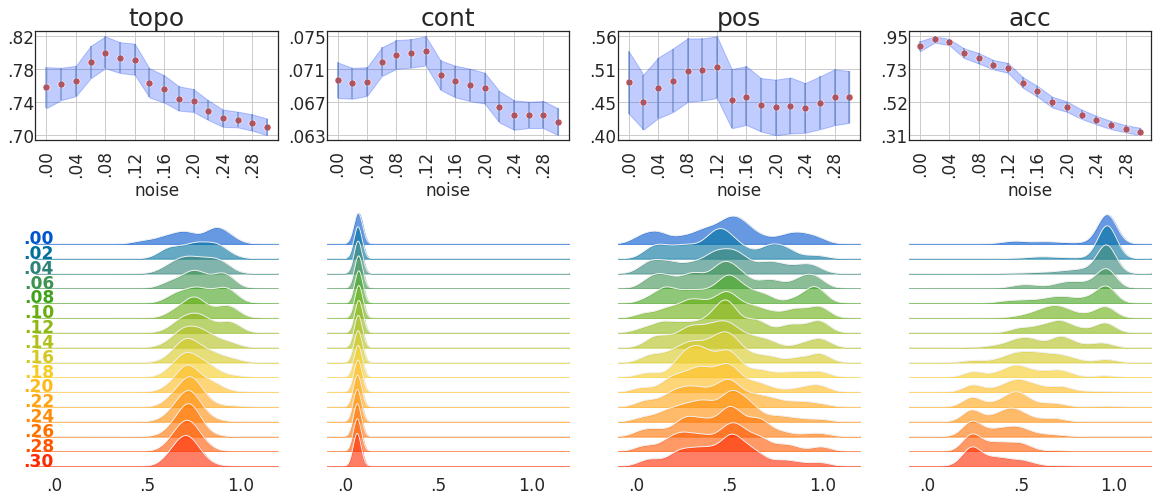}
    \caption{Message length equals 3. Top panel: average value of metrics for various noise levels. The shaded area corresponds to bootstrapped $95\%$-confidence intervals for this estimator. Bottom panel: kernel density estimators for metrics and noise levels across seeds.
    Here \emph{topo} stands for topographic similarity, \emph{pos} for positional disentanglement and \emph{acc} for accuracy.
    }
    \label{fig:ridge_message_len_3}
\end{figure}

\begin{figure}[H]
    \centering
    \includegraphics[width=\textwidth]{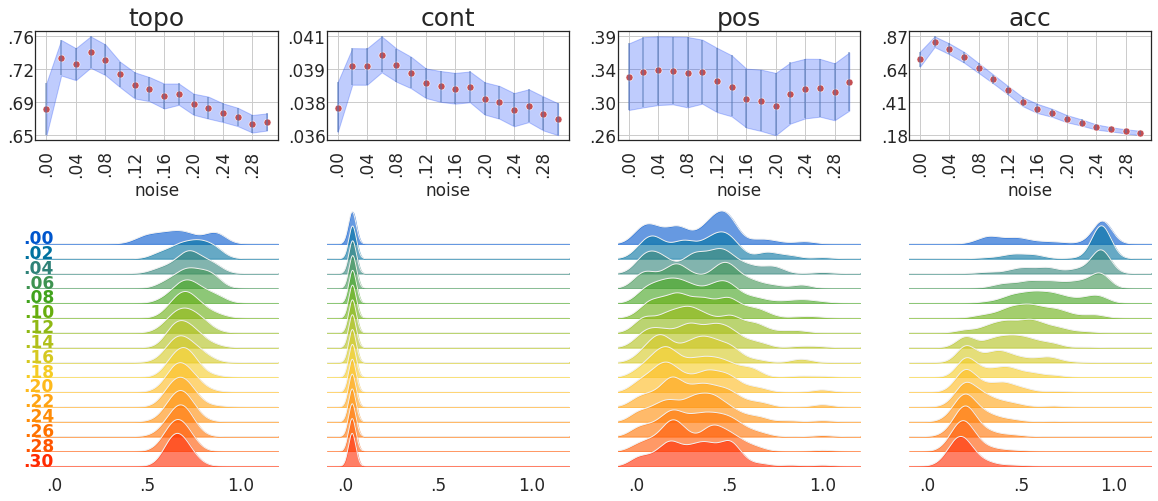}
    \caption{Message length equals 4. Top panel: average value of metrics for various noise levels. The shaded area corresponds to bootstrapped $95\%$-confidence intervals for this estimator. Bottom panel: kernel density estimators for metrics and noise levels across seeds.
    Here \emph{topo} stands for topographic similarity, \emph{pos} for positional disentanglement and \emph{acc} for accuracy.
    }
    \label{fig:ridge_message_len_4}
\end{figure}

\begin{table}
    \centering
    \begin{tabular}{l|cccc}
\toprule
noise &                                              topo &                                                 cont &                                               pos &                                               acc \\ \midrule
0.00  &  \makecell{$0.76$ \\ \scriptsize{$[0.73, 0.78]$}} &  \makecell{$0.069$ \\ \scriptsize{$[0.067, 0.072]$}} &  \makecell{$0.49$ \\ \scriptsize{$[0.44, 0.54]$}} &  \makecell{$0.89$ \\ \scriptsize{$[0.85, 0.92]$}} \\ \midrule
0.02  &  \makecell{$0.76$ \\ \scriptsize{$[0.74, 0.78]$}} &  \makecell{$0.069$ \\ \scriptsize{$[0.067, 0.071]$}} &  \makecell{$0.45$ \\ \scriptsize{$[0.41, 0.50]$}} &  \makecell{$0.93$ \\ \scriptsize{$[0.91, 0.95]$}} \\ \midrule
0.04  &  \makecell{$0.77$ \\ \scriptsize{$[0.75, 0.79]$}} &  \makecell{$0.069$ \\ \scriptsize{$[0.067, 0.071]$}} &  \makecell{$0.48$ \\ \scriptsize{$[0.43, 0.53]$}} &  \makecell{$0.91$ \\ \scriptsize{$[0.89, 0.93]$}} \\ \midrule
0.06  &  \makecell{$0.79$ \\ \scriptsize{$[0.77, 0.81]$}} &  \makecell{$0.072$ \\ \scriptsize{$[0.070, 0.074]$}} &  \makecell{$0.49$ \\ \scriptsize{$[0.44, 0.54]$}} &  \makecell{$0.84$ \\ \scriptsize{$[0.81, 0.87]$}} \\ \midrule
0.08  &  \makecell{$0.80$ \\ \scriptsize{$[0.78, 0.82]$}} &  \makecell{$0.073$ \\ \scriptsize{$[0.071, 0.074]$}} &  \makecell{$0.51$ \\ \scriptsize{$[0.46, 0.56]$}} &  \makecell{$0.81$ \\ \scriptsize{$[0.77, 0.84]$}} \\ \midrule
0.10  &  \makecell{$0.80$ \\ \scriptsize{$[0.78, 0.82]$}} &  \makecell{$0.073$ \\ \scriptsize{$[0.071, 0.075]$}} &  \makecell{$0.51$ \\ \scriptsize{$[0.46, 0.56]$}} &  \makecell{$0.76$ \\ \scriptsize{$[0.73, 0.79]$}} \\ \midrule
0.12  &  \makecell{$0.79$ \\ \scriptsize{$[0.77, 0.81]$}} &  \makecell{$0.073$ \\ \scriptsize{$[0.071, 0.075]$}} &  \makecell{$0.51$ \\ \scriptsize{$[0.46, 0.56]$}} &  \makecell{$0.74$ \\ \scriptsize{$[0.71, 0.78]$}} \\ \midrule
0.14  &  \makecell{$0.76$ \\ \scriptsize{$[0.74, 0.78]$}} &  \makecell{$0.070$ \\ \scriptsize{$[0.068, 0.072]$}} &  \makecell{$0.46$ \\ \scriptsize{$[0.41, 0.51]$}} &  \makecell{$0.65$ \\ \scriptsize{$[0.62, 0.68]$}} \\ \midrule
0.16  &  \makecell{$0.76$ \\ \scriptsize{$[0.74, 0.77]$}} &  \makecell{$0.069$ \\ \scriptsize{$[0.067, 0.071]$}} &  \makecell{$0.46$ \\ \scriptsize{$[0.42, 0.51]$}} &  \makecell{$0.59$ \\ \scriptsize{$[0.56, 0.63]$}} \\ \midrule
0.18  &  \makecell{$0.74$ \\ \scriptsize{$[0.73, 0.76]$}} &  \makecell{$0.069$ \\ \scriptsize{$[0.067, 0.071]$}} &  \makecell{$0.45$ \\ \scriptsize{$[0.41, 0.49]$}} &  \makecell{$0.52$ \\ \scriptsize{$[0.49, 0.56]$}} \\ \midrule
0.20  &  \makecell{$0.74$ \\ \scriptsize{$[0.73, 0.76]$}} &  \makecell{$0.069$ \\ \scriptsize{$[0.067, 0.070]$}} &  \makecell{$0.45$ \\ \scriptsize{$[0.40, 0.49]$}} &  \makecell{$0.49$ \\ \scriptsize{$[0.46, 0.53]$}} \\ \midrule
0.22  &  \makecell{$0.73$ \\ \scriptsize{$[0.72, 0.74]$}} &  \makecell{$0.066$ \\ \scriptsize{$[0.064, 0.068]$}} &  \makecell{$0.45$ \\ \scriptsize{$[0.40, 0.50]$}} &  \makecell{$0.44$ \\ \scriptsize{$[0.41, 0.47]$}} \\ \midrule
0.24  &  \makecell{$0.72$ \\ \scriptsize{$[0.71, 0.73]$}} &  \makecell{$0.065$ \\ \scriptsize{$[0.063, 0.067]$}} &  \makecell{$0.45$ \\ \scriptsize{$[0.40, 0.49]$}} &  \makecell{$0.41$ \\ \scriptsize{$[0.38, 0.43]$}} \\ \midrule
0.26  &  \makecell{$0.72$ \\ \scriptsize{$[0.71, 0.73]$}} &  \makecell{$0.065$ \\ \scriptsize{$[0.064, 0.067]$}} &  \makecell{$0.45$ \\ \scriptsize{$[0.41, 0.50]$}} &  \makecell{$0.37$ \\ \scriptsize{$[0.35, 0.40]$}} \\ \midrule
0.28  &  \makecell{$0.71$ \\ \scriptsize{$[0.70, 0.72]$}} &  \makecell{$0.065$ \\ \scriptsize{$[0.064, 0.067]$}} &  \makecell{$0.46$ \\ \scriptsize{$[0.42, 0.51]$}} &  \makecell{$0.35$ \\ \scriptsize{$[0.33, 0.37]$}} \\ \midrule
0.30  &  \makecell{$0.71$ \\ \scriptsize{$[0.70, 0.72]$}} &  \makecell{$0.064$ \\ \scriptsize{$[0.063, 0.066]$}} &  \makecell{$0.46$ \\ \scriptsize{$[0.42, 0.51]$}} &  \makecell{$0.33$ \\ \scriptsize{$[0.31, 0.36]$}} \\ 
\bottomrule
\end{tabular}

    \caption{Message length equals 3. Results for the metrics for selected noise levels. Shown in square brackets are bootstrapped $95\%$-confidence intervals.
    Here \emph{topo} stands for topographic similarity, \emph{conf} for conflict count, \emph{cont} for context independence, \emph{pos} for positional disentanglement and \emph{acc} for accuracy.
    }
    \label{app:tab:message_len_3}
\end{table}

\begin{table}
    \centering
    \begin{tabular}{l|cccc}
\toprule
noise &                                              topo &                                                 cont &                                               pos &                                               acc \\ \midrule
0.00  &  \makecell{$0.68$ \\ \scriptsize{$[0.65, 0.71]$}} &  \makecell{$0.038$ \\ \scriptsize{$[0.036, 0.039]$}} &  \makecell{$0.33$ \\ \scriptsize{$[0.29, 0.38]$}} &  \makecell{$0.71$ \\ \scriptsize{$[0.66, 0.76]$}} \\ \midrule
0.02  &  \makecell{$0.74$ \\ \scriptsize{$[0.72, 0.76]$}} &  \makecell{$0.040$ \\ \scriptsize{$[0.039, 0.040]$}} &  \makecell{$0.34$ \\ \scriptsize{$[0.30, 0.39]$}} &  \makecell{$0.83$ \\ \scriptsize{$[0.79, 0.87]$}} \\ \midrule
0.04  &  \makecell{$0.73$ \\ \scriptsize{$[0.71, 0.75]$}} &  \makecell{$0.040$ \\ \scriptsize{$[0.039, 0.040]$}} &  \makecell{$0.34$ \\ \scriptsize{$[0.30, 0.39]$}} &  \makecell{$0.78$ \\ \scriptsize{$[0.74, 0.82]$}} \\ \midrule
0.06  &  \makecell{$0.74$ \\ \scriptsize{$[0.73, 0.76]$}} &  \makecell{$0.040$ \\ \scriptsize{$[0.039, 0.041]$}} &  \makecell{$0.34$ \\ \scriptsize{$[0.30, 0.39]$}} &  \makecell{$0.72$ \\ \scriptsize{$[0.69, 0.76]$}} \\ \midrule
0.08  &  \makecell{$0.73$ \\ \scriptsize{$[0.72, 0.75]$}} &  \makecell{$0.040$ \\ \scriptsize{$[0.039, 0.040]$}} &  \makecell{$0.34$ \\ \scriptsize{$[0.30, 0.39]$}} &  \makecell{$0.65$ \\ \scriptsize{$[0.61, 0.68]$}} \\ \midrule
0.10  &  \makecell{$0.72$ \\ \scriptsize{$[0.70, 0.73]$}} &  \makecell{$0.039$ \\ \scriptsize{$[0.039, 0.040]$}} &  \makecell{$0.34$ \\ \scriptsize{$[0.30, 0.38]$}} &  \makecell{$0.57$ \\ \scriptsize{$[0.54, 0.60]$}} \\ \midrule
0.12  &  \makecell{$0.71$ \\ \scriptsize{$[0.69, 0.72]$}} &  \makecell{$0.039$ \\ \scriptsize{$[0.038, 0.039]$}} &  \makecell{$0.33$ \\ \scriptsize{$[0.29, 0.37]$}} &  \makecell{$0.49$ \\ \scriptsize{$[0.47, 0.52]$}} \\ \midrule
0.14  &  \makecell{$0.70$ \\ \scriptsize{$[0.69, 0.72]$}} &  \makecell{$0.039$ \\ \scriptsize{$[0.038, 0.039]$}} &  \makecell{$0.32$ \\ \scriptsize{$[0.28, 0.36]$}} &  \makecell{$0.41$ \\ \scriptsize{$[0.38, 0.45]$}} \\ \midrule
0.16  &  \makecell{$0.69$ \\ \scriptsize{$[0.68, 0.71]$}} &  \makecell{$0.038$ \\ \scriptsize{$[0.038, 0.039]$}} &  \makecell{$0.31$ \\ \scriptsize{$[0.27, 0.35]$}} &  \makecell{$0.37$ \\ \scriptsize{$[0.34, 0.40]$}} \\ \midrule
0.18  &  \makecell{$0.70$ \\ \scriptsize{$[0.69, 0.71]$}} &  \makecell{$0.039$ \\ \scriptsize{$[0.038, 0.039]$}} &  \makecell{$0.30$ \\ \scriptsize{$[0.27, 0.34]$}} &  \makecell{$0.34$ \\ \scriptsize{$[0.31, 0.37]$}} \\ \midrule
0.20  &  \makecell{$0.69$ \\ \scriptsize{$[0.67, 0.70]$}} &  \makecell{$0.038$ \\ \scriptsize{$[0.037, 0.039]$}} &  \makecell{$0.30$ \\ \scriptsize{$[0.26, 0.34]$}} &  \makecell{$0.30$ \\ \scriptsize{$[0.27, 0.32]$}} \\ \midrule
0.22  &  \makecell{$0.68$ \\ \scriptsize{$[0.67, 0.69]$}} &  \makecell{$0.038$ \\ \scriptsize{$[0.037, 0.039]$}} &  \makecell{$0.31$ \\ \scriptsize{$[0.28, 0.35]$}} &  \makecell{$0.27$ \\ \scriptsize{$[0.25, 0.29]$}} \\ \midrule
0.24  &  \makecell{$0.68$ \\ \scriptsize{$[0.67, 0.69]$}} &  \makecell{$0.038$ \\ \scriptsize{$[0.037, 0.038]$}} &  \makecell{$0.32$ \\ \scriptsize{$[0.28, 0.36]$}} &  \makecell{$0.24$ \\ \scriptsize{$[0.22, 0.26]$}} \\ \midrule
0.26  &  \makecell{$0.67$ \\ \scriptsize{$[0.66, 0.68]$}} &  \makecell{$0.038$ \\ \scriptsize{$[0.037, 0.038]$}} &  \makecell{$0.32$ \\ \scriptsize{$[0.28, 0.36]$}} &  \makecell{$0.22$ \\ \scriptsize{$[0.21, 0.24]$}} \\ \midrule
0.28  &  \makecell{$0.66$ \\ \scriptsize{$[0.66, 0.67]$}} &  \makecell{$0.037$ \\ \scriptsize{$[0.037, 0.038]$}} &  \makecell{$0.32$ \\ \scriptsize{$[0.28, 0.35]$}} &  \makecell{$0.21$ \\ \scriptsize{$[0.20, 0.23]$}} \\ \midrule
0.30  &  \makecell{$0.67$ \\ \scriptsize{$[0.66, 0.68]$}} &  \makecell{$0.037$ \\ \scriptsize{$[0.036, 0.038]$}} &  \makecell{$0.33$ \\ \scriptsize{$[0.29, 0.37]$}} &  \makecell{$0.20$ \\ \scriptsize{$[0.18, 0.21]$}} \\ 
\bottomrule
\end{tabular}

    \caption{Message length equals 4. Results for the metrics for selected noise levels. Shown in square brackets are bootstrapped $95\%$-confidence intervals.
    Here \emph{topo} stands for topographic similarity, \emph{conf} for conflict count, \emph{cont} for context independence, \emph{pos} for positional disentanglement and \emph{acc} for accuracy.
    }
    \label{app:tab:message_len_4}
\end{table}

\clearpage
\newpage
\subsection{Visual priors}\label{sec:scrambled_images_appendix}

In this section we give details for visual priors experiments, see Figure \ref{fig:ridge_scramble_32}, Table \ref{app:tab:scramble_32} (tile 32), Figure \ref{fig:ridge_scramble_16}, Table \ref{app:tab:scramble_16} (tile 16), and Figure \ref{fig:ridge_scramble_8}, Table \ref{app:tab:scramble_8} (tile 8).
The transition from coarser to finer tiles has a significant impact both on the metric's profiles and on their overall levels. The characteristic peak for some positive noise levels is still present.

We complement the picture with an analysis of the interplay between variable noise and accuracy, see Figure \ref{fig:scramble_32_filter} (tile $32$) and Figure \ref{fig:scramble_16_filter} (tile $16$). It shows that the overall metrics level, conditioned on seeds with high accuracy, increases significantly. 
Notice, however, that the number of experiments with high accuracy decrease as the threshold increases. 
In particular, we did not include visualization for a tile size 8, since there were too few experiments exceeding the accuracy threshold of $0.80$.

\begin{figure}[H]
    \centering
    \includegraphics[width=\textwidth]{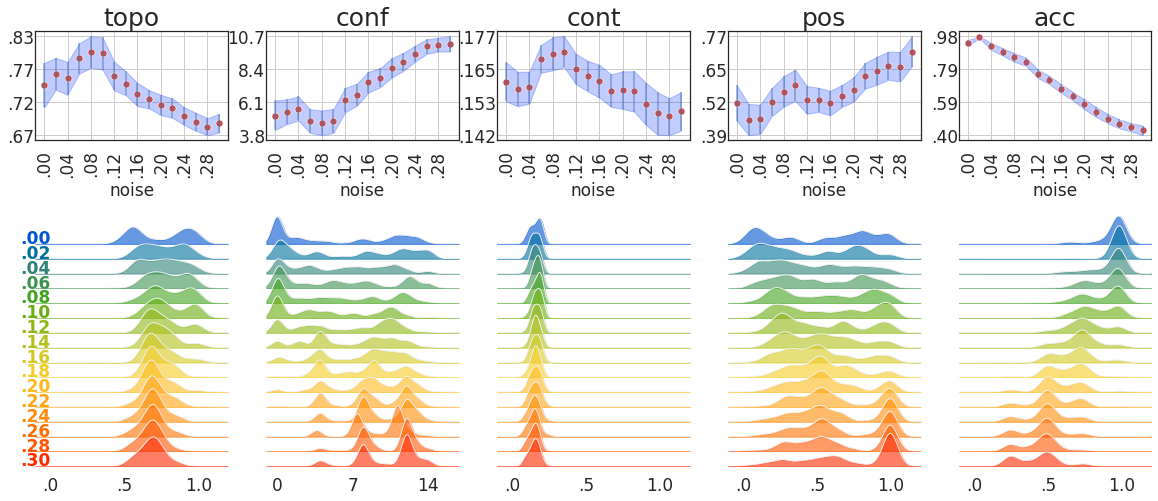}
    \caption{Scramble with tile 32. Top panel: average value of metrics for various noise levels. The shaded area corresponds to bootstrapped $95\%$-confidence intervals for this estimator. Bottom panel: kernel density estimators for metrics and noise levels across seeds.
    Here \emph{topo} stands for topographic similarity, \emph{conf} for conflict count, \emph{cont} for context independence, \emph{pos} for positional disentanglement and \emph{acc} for accuracy.
    }
    \label{fig:ridge_scramble_32}
\end{figure}

\begin{figure}[H]
    \centering
    \includegraphics[width=\textwidth]{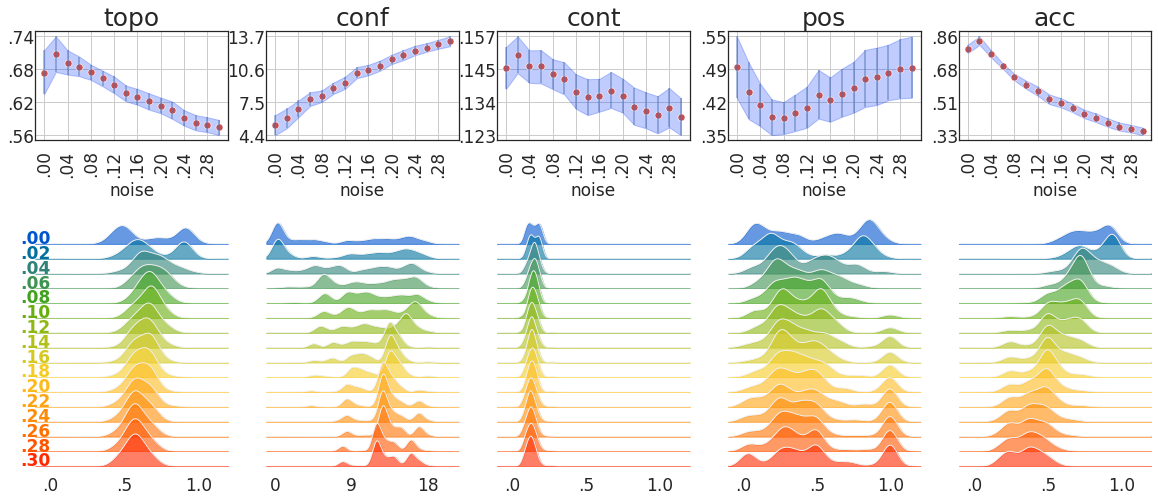}
    \caption{Scramble with tile 16. Top panel: average value of metrics for various noise levels. The shaded area corresponds to bootstrapped $95\%$-confidence intervals for this estimator. Bottom panel: kernel density estimators for metrics and noise levels across seeds.
    Here \emph{topo} stands for topographic similarity, \emph{conf} for conflict count, \emph{cont} for context independence, \emph{pos} for positional disentanglement and \emph{acc} for accuracy.
    }
    \label{fig:ridge_scramble_16}
\end{figure}

\begin{figure}[H]
    \centering
    \includegraphics[width=\textwidth]{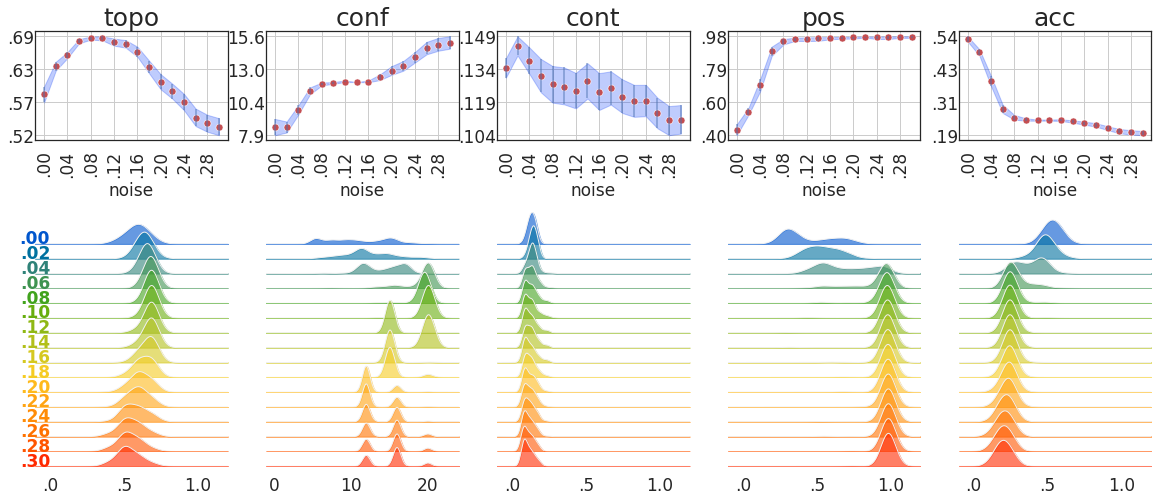}
    \caption{Scramble with tile 8. Top panel: average value of metrics for various noise levels. The shaded area corresponds to bootstrapped $95\%$-confidence intervals for this estimator. Bottom panel: kernel density estimators for metrics and noise levels across seeds.
    Here \emph{topo} stands for topographic similarity, \emph{conf} for conflict count, \emph{cont} for context independence, \emph{pos} for positional disentanglement and \emph{acc} for accuracy.
    }
    \label{fig:ridge_scramble_8}
\end{figure}

\begin{figure}[H]
    \centering
    \includegraphics[width=\textwidth]{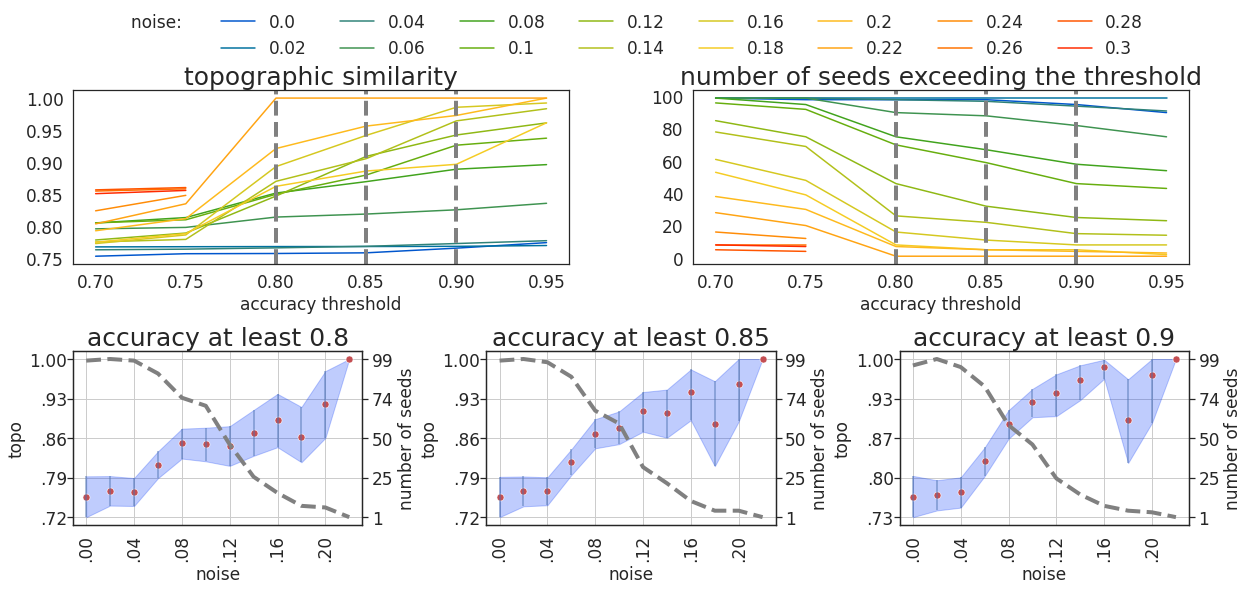}
    \caption{\small Tile size 32. Top left: The values of topo computed for each noise level (hues) and on seeds exceeding a certain accuracy threshold ($x$-axis). Vertical dashed lines represent three cross-sections, visualized in the bottom panel.
    Top right: Similar to the left panel, but instead of topo we visualize the number of seeds with accuracy at least as a given threshold ($x$-axis). Vertical dashed lines represent three cross-sections, visualized in the bottom panel.
    Bottom: Each of the plots represents a cross-section of the plots in the top panel, taken at points $0.80$, $0.85$, and $0.90$, respectively. 
    On the left axis of each figure is the range of topo, whereas on the right axis is the number of seeds with accuracy exceed the corresponding level. On the $x$-axis are the noise levels. The scatter plot with 95\%-confidence intervals represents the values of topo. The gray dashed line represents the number of seeds with accuracy exceeding a given threshold, for each of the noise levels.}
    \label{fig:scramble_32_filter}
\end{figure}

\begin{figure}[H]
    \centering
    \includegraphics[width=\textwidth]{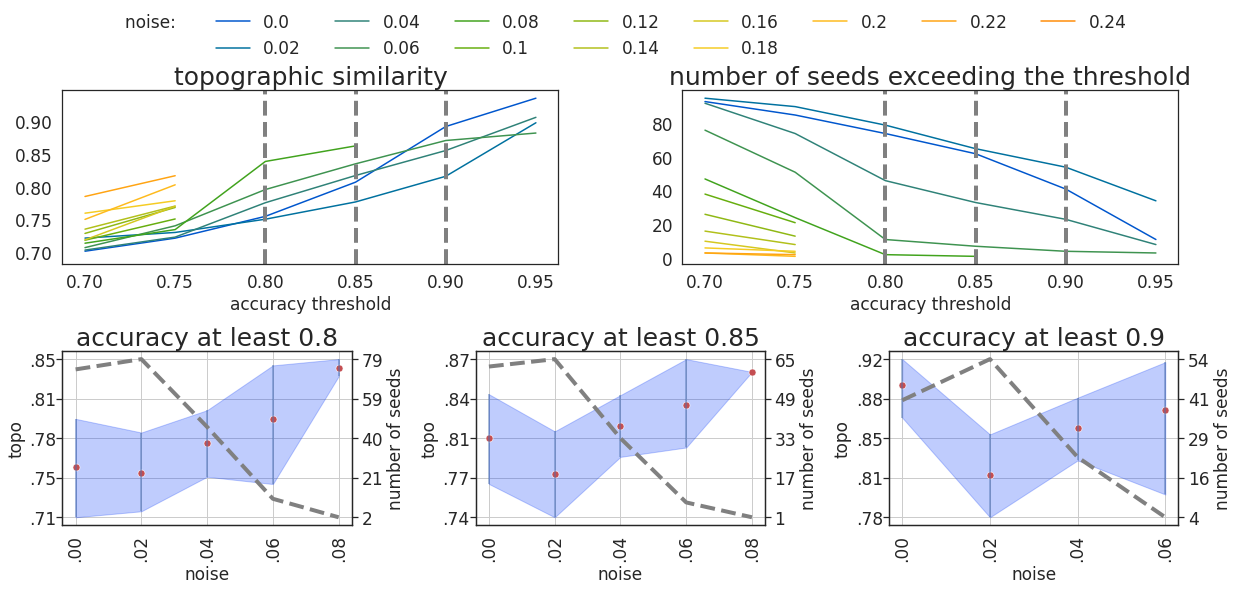}
    \caption{\small Tile size 16. Top left: The values of topo computed for each noise level (hues) and on seeds exceeding a certain accuracy threshold ($x$-axis). Vertical dashed lines represent three cross-sections, visualized in the bottom panel.
    Top right: Similar to the left panel, but instead of topo we visualize the number of seeds with accuracy at least as a given threshold ($x$-axis). Vertical dashed lines represent three cross-sections, visualized in the bottom panel.
    Bottom: Each of the plots represents a cross-section of the plots in the top panel, taken at points $0.80$, $0.85$, and $0.90$, respectively. 
    On the left axis of each figure is the range of topo, whereas on the right axis is the number of seeds with accuracy exceed the corresponding level. On the $x$-axis are the noise levels. The scatter plot with 95\%-confidence intervals represents the values of topo. The gray dashed line represents the number of seeds with accuracy exceeding a given threshold, for each of the noise levels.}
    \label{fig:scramble_16_filter}
\end{figure}

\begin{table}
    \centering
    \begin{tabular}{l|ccccc}
\toprule
noise &                                              topo &                                                conf &                                                 cont &                                               pos &                                               acc \\ \midrule
0.00  &  \makecell{$0.75$ \\ \scriptsize{$[0.71, 0.78]$}} &    \makecell{$5.20$ \\ \scriptsize{$[4.21, 6.27]$}} &  \makecell{$0.160$ \\ \scriptsize{$[0.154, 0.168]$}} &  \makecell{$0.52$ \\ \scriptsize{$[0.45, 0.59]$}} &  \makecell{$0.95$ \\ \scriptsize{$[0.93, 0.96]$}} \\ \midrule
0.02  &  \makecell{$0.77$ \\ \scriptsize{$[0.74, 0.79]$}} &    \makecell{$5.47$ \\ \scriptsize{$[4.61, 6.35]$}} &  \makecell{$0.158$ \\ \scriptsize{$[0.152, 0.164]$}} &  \makecell{$0.45$ \\ \scriptsize{$[0.39, 0.51]$}} &  \makecell{$0.98$ \\ \scriptsize{$[0.97, 0.98]$}} \\ \midrule
0.04  &  \makecell{$0.76$ \\ \scriptsize{$[0.73, 0.79]$}} &    \makecell{$5.69$ \\ \scriptsize{$[4.87, 6.55]$}} &  \makecell{$0.158$ \\ \scriptsize{$[0.153, 0.164]$}} &  \makecell{$0.45$ \\ \scriptsize{$[0.40, 0.51]$}} &  \makecell{$0.92$ \\ \scriptsize{$[0.90, 0.94]$}} \\ \midrule
0.06  &  \makecell{$0.79$ \\ \scriptsize{$[0.77, 0.82]$}} &    \makecell{$4.81$ \\ \scriptsize{$[4.01, 5.66]$}} &  \makecell{$0.169$ \\ \scriptsize{$[0.164, 0.174]$}} &  \makecell{$0.52$ \\ \scriptsize{$[0.47, 0.58]$}} &  \makecell{$0.89$ \\ \scriptsize{$[0.87, 0.91]$}} \\ \midrule
0.08  &  \makecell{$0.80$ \\ \scriptsize{$[0.78, 0.83]$}} &    \makecell{$4.68$ \\ \scriptsize{$[3.84, 5.57]$}} &  \makecell{$0.170$ \\ \scriptsize{$[0.165, 0.176]$}} &  \makecell{$0.56$ \\ \scriptsize{$[0.50, 0.62]$}} &  \makecell{$0.86$ \\ \scriptsize{$[0.83, 0.88]$}} \\ \midrule
0.10  &  \makecell{$0.80$ \\ \scriptsize{$[0.77, 0.83]$}} &    \makecell{$4.82$ \\ \scriptsize{$[4.00, 5.68]$}} &  \makecell{$0.171$ \\ \scriptsize{$[0.166, 0.177]$}} &  \makecell{$0.59$ \\ \scriptsize{$[0.53, 0.65]$}} &  \makecell{$0.83$ \\ \scriptsize{$[0.81, 0.86]$}} \\ \midrule
0.12  &  \makecell{$0.76$ \\ \scriptsize{$[0.74, 0.79]$}} &    \makecell{$6.29$ \\ \scriptsize{$[5.44, 7.14]$}} &  \makecell{$0.165$ \\ \scriptsize{$[0.160, 0.170]$}} &  \makecell{$0.53$ \\ \scriptsize{$[0.48, 0.58]$}} &  \makecell{$0.76$ \\ \scriptsize{$[0.74, 0.79]$}} \\ \midrule
0.14  &  \makecell{$0.75$ \\ \scriptsize{$[0.73, 0.77]$}} &    \makecell{$6.68$ \\ \scriptsize{$[5.90, 7.44]$}} &  \makecell{$0.163$ \\ \scriptsize{$[0.157, 0.168]$}} &  \makecell{$0.53$ \\ \scriptsize{$[0.48, 0.58]$}} &  \makecell{$0.72$ \\ \scriptsize{$[0.70, 0.75]$}} \\ \midrule
0.16  &  \makecell{$0.73$ \\ \scriptsize{$[0.72, 0.75]$}} &    \makecell{$7.57$ \\ \scriptsize{$[6.82, 8.27]$}} &  \makecell{$0.161$ \\ \scriptsize{$[0.155, 0.167]$}} &  \makecell{$0.52$ \\ \scriptsize{$[0.47, 0.56]$}} &  \makecell{$0.67$ \\ \scriptsize{$[0.65, 0.70]$}} \\ \midrule
0.18  &  \makecell{$0.73$ \\ \scriptsize{$[0.71, 0.74]$}} &    \makecell{$7.86$ \\ \scriptsize{$[7.19, 8.51]$}} &  \makecell{$0.157$ \\ \scriptsize{$[0.152, 0.163]$}} &  \makecell{$0.54$ \\ \scriptsize{$[0.50, 0.59]$}} &  \makecell{$0.63$ \\ \scriptsize{$[0.60, 0.65]$}} \\ \midrule
0.20  &  \makecell{$0.72$ \\ \scriptsize{$[0.70, 0.73]$}} &    \makecell{$8.57$ \\ \scriptsize{$[7.87, 9.22]$}} &  \makecell{$0.157$ \\ \scriptsize{$[0.151, 0.164]$}} &  \makecell{$0.57$ \\ \scriptsize{$[0.52, 0.62]$}} &  \makecell{$0.58$ \\ \scriptsize{$[0.55, 0.61]$}} \\ \midrule
0.22  &  \makecell{$0.71$ \\ \scriptsize{$[0.70, 0.73]$}} &    \makecell{$8.98$ \\ \scriptsize{$[8.34, 9.59]$}} &  \makecell{$0.157$ \\ \scriptsize{$[0.150, 0.164]$}} &  \makecell{$0.62$ \\ \scriptsize{$[0.57, 0.67]$}} &  \makecell{$0.54$ \\ \scriptsize{$[0.51, 0.57]$}} \\ \midrule
0.24  &  \makecell{$0.70$ \\ \scriptsize{$[0.69, 0.71]$}} &   \makecell{$9.54$ \\ \scriptsize{$[8.98, 10.09]$}} &  \makecell{$0.153$ \\ \scriptsize{$[0.145, 0.160]$}} &  \makecell{$0.64$ \\ \scriptsize{$[0.59, 0.70]$}} &  \makecell{$0.50$ \\ \scriptsize{$[0.47, 0.52]$}} \\ \midrule
0.26  &  \makecell{$0.69$ \\ \scriptsize{$[0.68, 0.70]$}} &  \makecell{$10.04$ \\ \scriptsize{$[9.51, 10.54]$}} &  \makecell{$0.149$ \\ \scriptsize{$[0.142, 0.157]$}} &  \makecell{$0.66$ \\ \scriptsize{$[0.61, 0.72]$}} &  \makecell{$0.47$ \\ \scriptsize{$[0.44, 0.49]$}} \\ \midrule
0.28  &  \makecell{$0.68$ \\ \scriptsize{$[0.67, 0.70]$}} &  \makecell{$10.18$ \\ \scriptsize{$[9.69, 10.65]$}} &  \makecell{$0.148$ \\ \scriptsize{$[0.142, 0.155]$}} &  \makecell{$0.66$ \\ \scriptsize{$[0.60, 0.71]$}} &  \makecell{$0.45$ \\ \scriptsize{$[0.42, 0.47]$}} \\ \midrule
0.30  &  \makecell{$0.69$ \\ \scriptsize{$[0.67, 0.70]$}} &  \makecell{$10.23$ \\ \scriptsize{$[9.68, 10.75]$}} &  \makecell{$0.150$ \\ \scriptsize{$[0.143, 0.157]$}} &  \makecell{$0.72$ \\ \scriptsize{$[0.66, 0.77]$}} &  \makecell{$0.43$ \\ \scriptsize{$[0.40, 0.46]$}} \\ 
\bottomrule
\end{tabular}

    \caption{Scramble with tile 32. Results for the metrics for selected noise levels. Shown in square brackets are bootstrapped $95\%$-confidence intervals.
    Here \emph{topo} stands for topographic similarity, \emph{conf} for conflict count, \emph{cont} for context independence, \emph{pos} for positional disentanglement and \emph{acc} for accuracy.
    }
    \label{app:tab:scramble_32}
\end{table}

\begin{table}
    \centering
    \begin{tabular}{l|ccccc}
\toprule
noise &                                              topo &                                                 conf &                                                 cont &                                               pos &                                               acc \\ \midrule
0.00  &  \makecell{$0.67$ \\ \scriptsize{$[0.64, 0.71]$}} &     \makecell{$5.32$ \\ \scriptsize{$[4.36, 6.26]$}} &  \makecell{$0.146$ \\ \scriptsize{$[0.139, 0.153]$}} &  \makecell{$0.49$ \\ \scriptsize{$[0.43, 0.55]$}} &  \makecell{$0.79$ \\ \scriptsize{$[0.77, 0.82]$}} \\ \midrule
0.02  &  \makecell{$0.71$ \\ \scriptsize{$[0.67, 0.74]$}} &     \makecell{$5.97$ \\ \scriptsize{$[5.01, 6.89]$}} &  \makecell{$0.150$ \\ \scriptsize{$[0.144, 0.157]$}} &  \makecell{$0.44$ \\ \scriptsize{$[0.39, 0.50]$}} &  \makecell{$0.84$ \\ \scriptsize{$[0.81, 0.86]$}} \\ \midrule
0.04  &  \makecell{$0.69$ \\ \scriptsize{$[0.67, 0.71]$}} &     \makecell{$6.87$ \\ \scriptsize{$[6.07, 7.66]$}} &  \makecell{$0.147$ \\ \scriptsize{$[0.141, 0.152]$}} &  \makecell{$0.41$ \\ \scriptsize{$[0.37, 0.46]$}} &  \makecell{$0.77$ \\ \scriptsize{$[0.75, 0.79]$}} \\ \midrule
0.06  &  \makecell{$0.68$ \\ \scriptsize{$[0.67, 0.70]$}} &     \makecell{$7.79$ \\ \scriptsize{$[7.13, 8.43]$}} &  \makecell{$0.147$ \\ \scriptsize{$[0.141, 0.152]$}} &  \makecell{$0.39$ \\ \scriptsize{$[0.35, 0.43]$}} &  \makecell{$0.70$ \\ \scriptsize{$[0.69, 0.72]$}} \\ \midrule
0.08  &  \makecell{$0.67$ \\ \scriptsize{$[0.66, 0.69]$}} &     \makecell{$8.05$ \\ \scriptsize{$[7.50, 8.60]$}} &  \makecell{$0.144$ \\ \scriptsize{$[0.139, 0.149]$}} &  \makecell{$0.39$ \\ \scriptsize{$[0.36, 0.42]$}} &  \makecell{$0.64$ \\ \scriptsize{$[0.62, 0.66]$}} \\ \midrule
0.10  &  \makecell{$0.66$ \\ \scriptsize{$[0.65, 0.68]$}} &     \makecell{$8.81$ \\ \scriptsize{$[8.27, 9.35]$}} &  \makecell{$0.142$ \\ \scriptsize{$[0.137, 0.148]$}} &  \makecell{$0.40$ \\ \scriptsize{$[0.36, 0.44]$}} &  \makecell{$0.60$ \\ \scriptsize{$[0.58, 0.62]$}} \\ \midrule
0.12  &  \makecell{$0.65$ \\ \scriptsize{$[0.64, 0.67]$}} &     \makecell{$9.34$ \\ \scriptsize{$[8.75, 9.91]$}} &  \makecell{$0.138$ \\ \scriptsize{$[0.132, 0.144]$}} &  \makecell{$0.41$ \\ \scriptsize{$[0.37, 0.45]$}} &  \makecell{$0.57$ \\ \scriptsize{$[0.54, 0.59]$}} \\ \midrule
0.14  &  \makecell{$0.64$ \\ \scriptsize{$[0.62, 0.65]$}} &   \makecell{$10.30$ \\ \scriptsize{$[9.76, 10.82]$}} &  \makecell{$0.136$ \\ \scriptsize{$[0.130, 0.142]$}} &  \makecell{$0.43$ \\ \scriptsize{$[0.39, 0.49]$}} &  \makecell{$0.52$ \\ \scriptsize{$[0.50, 0.55]$}} \\ \midrule
0.16  &  \makecell{$0.63$ \\ \scriptsize{$[0.62, 0.64]$}} &  \makecell{$10.54$ \\ \scriptsize{$[10.03, 11.04]$}} &  \makecell{$0.137$ \\ \scriptsize{$[0.131, 0.142]$}} &  \makecell{$0.42$ \\ \scriptsize{$[0.38, 0.47]$}} &  \makecell{$0.50$ \\ \scriptsize{$[0.48, 0.53]$}} \\ \midrule
0.18  &  \makecell{$0.62$ \\ \scriptsize{$[0.61, 0.64]$}} &  \makecell{$10.95$ \\ \scriptsize{$[10.47, 11.43]$}} &  \makecell{$0.138$ \\ \scriptsize{$[0.132, 0.145]$}} &  \makecell{$0.44$ \\ \scriptsize{$[0.39, 0.49]$}} &  \makecell{$0.48$ \\ \scriptsize{$[0.45, 0.50]$}} \\ \midrule
0.20  &  \makecell{$0.62$ \\ \scriptsize{$[0.60, 0.63]$}} &  \makecell{$11.55$ \\ \scriptsize{$[11.10, 11.99]$}} &  \makecell{$0.136$ \\ \scriptsize{$[0.131, 0.142]$}} &  \makecell{$0.45$ \\ \scriptsize{$[0.40, 0.50]$}} &  \makecell{$0.44$ \\ \scriptsize{$[0.42, 0.47]$}} \\ \midrule
0.22  &  \makecell{$0.61$ \\ \scriptsize{$[0.59, 0.62]$}} &  \makecell{$11.95$ \\ \scriptsize{$[11.47, 12.43]$}} &  \makecell{$0.133$ \\ \scriptsize{$[0.126, 0.139]$}} &  \makecell{$0.47$ \\ \scriptsize{$[0.41, 0.53]$}} &  \makecell{$0.42$ \\ \scriptsize{$[0.40, 0.45]$}} \\ \midrule
0.24  &  \makecell{$0.59$ \\ \scriptsize{$[0.58, 0.61]$}} &  \makecell{$12.38$ \\ \scriptsize{$[11.94, 12.84]$}} &  \makecell{$0.131$ \\ \scriptsize{$[0.125, 0.138]$}} &  \makecell{$0.47$ \\ \scriptsize{$[0.41, 0.53]$}} &  \makecell{$0.40$ \\ \scriptsize{$[0.38, 0.42]$}} \\ \midrule
0.26  &  \makecell{$0.58$ \\ \scriptsize{$[0.57, 0.60]$}} &  \makecell{$12.67$ \\ \scriptsize{$[12.26, 13.09]$}} &  \makecell{$0.130$ \\ \scriptsize{$[0.124, 0.136]$}} &  \makecell{$0.48$ \\ \scriptsize{$[0.42, 0.54]$}} &  \makecell{$0.38$ \\ \scriptsize{$[0.36, 0.40]$}} \\ \midrule
0.28  &  \makecell{$0.58$ \\ \scriptsize{$[0.57, 0.59]$}} &  \makecell{$12.97$ \\ \scriptsize{$[12.55, 13.41]$}} &  \makecell{$0.132$ \\ \scriptsize{$[0.126, 0.139]$}} &  \makecell{$0.49$ \\ \scriptsize{$[0.43, 0.55]$}} &  \makecell{$0.36$ \\ \scriptsize{$[0.34, 0.39]$}} \\ \midrule
0.30  &  \makecell{$0.58$ \\ \scriptsize{$[0.56, 0.59]$}} &  \makecell{$13.25$ \\ \scriptsize{$[12.79, 13.73]$}} &  \makecell{$0.129$ \\ \scriptsize{$[0.123, 0.136]$}} &  \makecell{$0.49$ \\ \scriptsize{$[0.43, 0.55]$}} &  \makecell{$0.35$ \\ \scriptsize{$[0.33, 0.37]$}} \\ 
\bottomrule
\end{tabular}

    \caption{Scramble with tile 16. Results for the metrics for selected noise levels. Shown in square brackets are bootstrapped $95\%$-confidence intervals.
    Here \emph{topo} stands for topographic similarity, \emph{conf} for conflict count, \emph{cont} for context independence, \emph{pos} for positional disentanglement and \emph{acc} for accuracy.
    }
    \label{app:tab:scramble_16}
\end{table}

\begin{table}
    \centering
    \begin{tabular}{l|ccccc}
\toprule
noise &                                              topo &                                                 conf &                                                 cont &                                               pos &                                               acc \\ \midrule
0.00  &  \makecell{$0.59$ \\ \scriptsize{$[0.57, 0.60]$}} &     \makecell{$8.51$ \\ \scriptsize{$[7.86, 9.13]$}} &  \makecell{$0.134$ \\ \scriptsize{$[0.130, 0.139]$}} &  \makecell{$0.44$ \\ \scriptsize{$[0.40, 0.47]$}} &  \makecell{$0.53$ \\ \scriptsize{$[0.52, 0.54]$}} \\ \midrule
0.02  &  \makecell{$0.64$ \\ \scriptsize{$[0.63, 0.64]$}} &     \makecell{$8.48$ \\ \scriptsize{$[8.07, 8.92]$}} &  \makecell{$0.144$ \\ \scriptsize{$[0.140, 0.149]$}} &  \makecell{$0.54$ \\ \scriptsize{$[0.52, 0.56]$}} &  \makecell{$0.49$ \\ \scriptsize{$[0.48, 0.50]$}} \\ \midrule
0.04  &  \makecell{$0.65$ \\ \scriptsize{$[0.65, 0.66]$}} &    \makecell{$9.82$ \\ \scriptsize{$[9.46, 10.18]$}} &  \makecell{$0.138$ \\ \scriptsize{$[0.132, 0.144]$}} &  \makecell{$0.70$ \\ \scriptsize{$[0.66, 0.73]$}} &  \makecell{$0.38$ \\ \scriptsize{$[0.37, 0.40]$}} \\ \midrule
0.06  &  \makecell{$0.68$ \\ \scriptsize{$[0.67, 0.68]$}} &  \makecell{$11.33$ \\ \scriptsize{$[11.04, 11.58]$}} &  \makecell{$0.130$ \\ \scriptsize{$[0.123, 0.138]$}} &  \makecell{$0.90$ \\ \scriptsize{$[0.87, 0.92]$}} &  \makecell{$0.29$ \\ \scriptsize{$[0.27, 0.30]$}} \\ \midrule
0.08  &  \makecell{$0.68$ \\ \scriptsize{$[0.68, 0.69]$}} &  \makecell{$11.86$ \\ \scriptsize{$[11.73, 11.96]$}} &  \makecell{$0.127$ \\ \scriptsize{$[0.119, 0.135]$}} &  \makecell{$0.96$ \\ \scriptsize{$[0.94, 0.97]$}} &  \makecell{$0.26$ \\ \scriptsize{$[0.25, 0.27]$}} \\ \midrule
0.10  &  \makecell{$0.68$ \\ \scriptsize{$[0.68, 0.69]$}} &  \makecell{$11.95$ \\ \scriptsize{$[11.87, 12.00]$}} &  \makecell{$0.126$ \\ \scriptsize{$[0.118, 0.135]$}} &  \makecell{$0.97$ \\ \scriptsize{$[0.96, 0.98]$}} &  \makecell{$0.25$ \\ \scriptsize{$[0.24, 0.25]$}} \\ \midrule
0.12  &  \makecell{$0.68$ \\ \scriptsize{$[0.67, 0.68]$}} &  \makecell{$12.04$ \\ \scriptsize{$[12.00, 12.12]$}} &  \makecell{$0.124$ \\ \scriptsize{$[0.116, 0.132]$}} &  \makecell{$0.97$ \\ \scriptsize{$[0.96, 0.98]$}} &  \makecell{$0.25$ \\ \scriptsize{$[0.24, 0.25]$}} \\ \midrule
0.14  &  \makecell{$0.67$ \\ \scriptsize{$[0.67, 0.68]$}} &  \makecell{$12.00$ \\ \scriptsize{$[12.00, 12.00]$}} &  \makecell{$0.128$ \\ \scriptsize{$[0.121, 0.136]$}} &  \makecell{$0.97$ \\ \scriptsize{$[0.96, 0.98]$}} &  \makecell{$0.25$ \\ \scriptsize{$[0.24, 0.25]$}} \\ \midrule
0.16  &  \makecell{$0.66$ \\ \scriptsize{$[0.65, 0.67]$}} &  \makecell{$12.03$ \\ \scriptsize{$[11.98, 12.12]$}} &  \makecell{$0.123$ \\ \scriptsize{$[0.115, 0.132]$}} &  \makecell{$0.98$ \\ \scriptsize{$[0.97, 0.98]$}} &  \makecell{$0.25$ \\ \scriptsize{$[0.24, 0.25]$}} \\ \midrule
0.18  &  \makecell{$0.63$ \\ \scriptsize{$[0.62, 0.65]$}} &  \makecell{$12.36$ \\ \scriptsize{$[12.16, 12.60]$}} &  \makecell{$0.125$ \\ \scriptsize{$[0.118, 0.133]$}} &  \makecell{$0.98$ \\ \scriptsize{$[0.97, 0.98]$}} &  \makecell{$0.24$ \\ \scriptsize{$[0.24, 0.25]$}} \\ \midrule
0.20  &  \makecell{$0.61$ \\ \scriptsize{$[0.60, 0.62]$}} &  \makecell{$12.88$ \\ \scriptsize{$[12.56, 13.24]$}} &  \makecell{$0.121$ \\ \scriptsize{$[0.114, 0.129]$}} &  \makecell{$0.98$ \\ \scriptsize{$[0.97, 0.98]$}} &  \makecell{$0.24$ \\ \scriptsize{$[0.23, 0.24]$}} \\ \midrule
0.22  &  \makecell{$0.59$ \\ \scriptsize{$[0.58, 0.61]$}} &  \makecell{$13.24$ \\ \scriptsize{$[12.88, 13.64]$}} &  \makecell{$0.119$ \\ \scriptsize{$[0.112, 0.127]$}} &  \makecell{$0.98$ \\ \scriptsize{$[0.98, 0.98]$}} &  \makecell{$0.23$ \\ \scriptsize{$[0.22, 0.24]$}} \\ \midrule
0.24  &  \makecell{$0.57$ \\ \scriptsize{$[0.56, 0.59]$}} &  \makecell{$13.93$ \\ \scriptsize{$[13.52, 14.34]$}} &  \makecell{$0.120$ \\ \scriptsize{$[0.113, 0.127]$}} &  \makecell{$0.98$ \\ \scriptsize{$[0.97, 0.98]$}} &  \makecell{$0.22$ \\ \scriptsize{$[0.21, 0.23]$}} \\ \midrule
0.26  &  \makecell{$0.55$ \\ \scriptsize{$[0.53, 0.56]$}} &  \makecell{$14.65$ \\ \scriptsize{$[14.16, 15.13]$}} &  \makecell{$0.114$ \\ \scriptsize{$[0.107, 0.121]$}} &  \makecell{$0.98$ \\ \scriptsize{$[0.97, 0.98]$}} &  \makecell{$0.21$ \\ \scriptsize{$[0.20, 0.22]$}} \\ \midrule
0.28  &  \makecell{$0.54$ \\ \scriptsize{$[0.52, 0.55]$}} &  \makecell{$14.92$ \\ \scriptsize{$[14.44, 15.44]$}} &  \makecell{$0.111$ \\ \scriptsize{$[0.104, 0.117]$}} &  \makecell{$0.98$ \\ \scriptsize{$[0.98, 0.98]$}} &  \makecell{$0.20$ \\ \scriptsize{$[0.20, 0.21]$}} \\ \midrule
0.30  &  \makecell{$0.53$ \\ \scriptsize{$[0.52, 0.55]$}} &  \makecell{$15.08$ \\ \scriptsize{$[14.60, 15.56]$}} &  \makecell{$0.111$ \\ \scriptsize{$[0.105, 0.118]$}} &  \makecell{$0.98$ \\ \scriptsize{$[0.98, 0.98]$}} &  \makecell{$0.20$ \\ \scriptsize{$[0.19, 0.21]$}} \\ 
\bottomrule
\end{tabular}

    \caption{Scramble with tile 8. Results for the metrics for selected noise levels. Shown in square brackets are bootstrapped $95\%$-confidence intervals.
    Here \emph{topo} stands for topographic similarity, \emph{conf} for conflict count, \emph{cont} for context independence, \emph{pos} for positional disentanglement and \emph{acc} for accuracy.
    }
    \label{app:tab:scramble_8}
\end{table}

\clearpage
\newpage
\subsection{Scrambled features}\label{sec:scrambled_features_appendix}
In these experiments, we aimed to understand if the architecture of the output is important. Our architecture has a separate output for 'color' ($\mathcal F_1$)  and for 'shape' ($\mathcal F_2$) (for the architecture details, see Figure \ref{fig:nn_arch}). We permute these features as follows. Let $k:\mathcal F_1\times\mathcal F_2 \mapsto \{0, \ldots, |\mathcal F_1\times\mathcal F_2| - 1\}=:\tilde{\mathcal F}$ be any bijection from the color features and the shape features. 
Let $\pi:\tilde{\mathcal F}\mapsto \tilde{\mathcal F}$ be a random permutation, sampled at the beginning of the experiment. We assign new features via mapping:
$k^{-1}(\pi(k((f_c, f_s))))$.
Clearly, these are no longer colors and shapes (unless $\pi$ is an identity). However, they are still factorized along two dimensions and we still use the factorized output. We can see that behavior is similar to the 'standard' features (high levels of compositionality with the characteristic extremum point). This poses strong evidence that output architecture is a strong inductive bias for compositionality.

\begin{figure}[H]
    \centering
    \includegraphics[width=\textwidth]{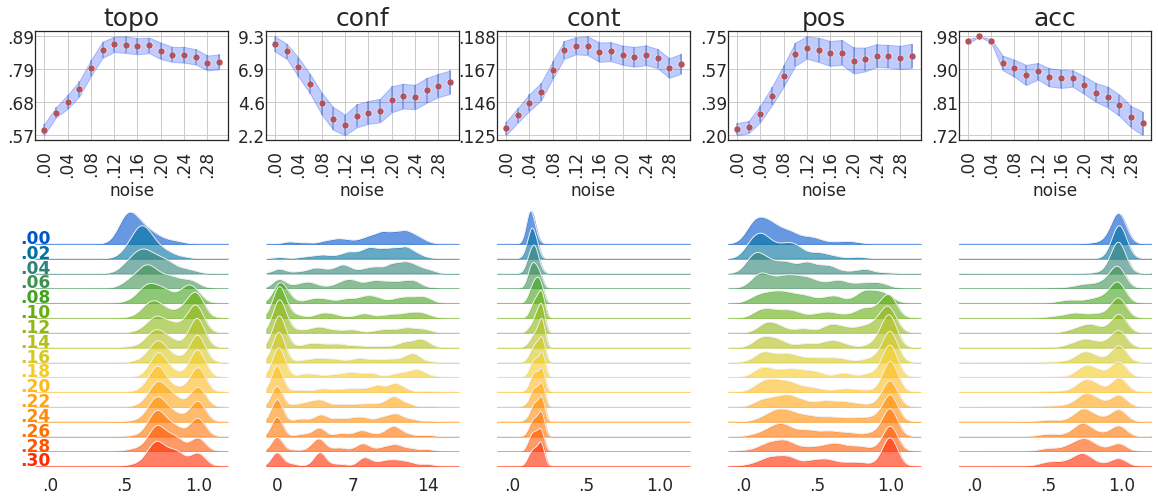}
    \caption{Scrambled features. Top panel: average value of metrics for various noise levels. The shaded area corresponds to bootstrapped $95\%$-confidence intervals for this estimator. Bottom panel: kernel density estimators for metrics and noise levels across seeds.
    Here \emph{topo} stands for topographic similarity, \emph{conf} for conflict count, \emph{cont} for context independence, \emph{pos} for positional disentanglement and \emph{acc} for accuracy.}
    \label{fig:app:scrambled_labels}
\end{figure}

\begin{table}
    \centering
    \begin{tabular}{l|ccccc}
\toprule
noise &                                              topo &                                              conf &                                                 cont &                                               pos &                                               acc \\ \midrule
0.00  &  \makecell{$0.59$ \\ \scriptsize{$[0.57, 0.61]$}} &  \makecell{$8.72$ \\ \scriptsize{$[8.17, 9.26]$}} &  \makecell{$0.129$ \\ \scriptsize{$[0.125, 0.133]$}} &  \makecell{$0.24$ \\ \scriptsize{$[0.20, 0.27]$}} &  \makecell{$0.97$ \\ \scriptsize{$[0.96, 0.98]$}} \\ \midrule
0.02  &  \makecell{$0.64$ \\ \scriptsize{$[0.63, 0.66]$}} &  \makecell{$8.21$ \\ \scriptsize{$[7.65, 8.73]$}} &  \makecell{$0.138$ \\ \scriptsize{$[0.133, 0.142]$}} &  \makecell{$0.25$ \\ \scriptsize{$[0.21, 0.28]$}} &  \makecell{$0.98$ \\ \scriptsize{$[0.98, 0.98]$}} \\ \midrule
0.04  &  \makecell{$0.68$ \\ \scriptsize{$[0.66, 0.70]$}} &  \makecell{$7.07$ \\ \scriptsize{$[6.39, 7.75]$}} &  \makecell{$0.145$ \\ \scriptsize{$[0.140, 0.151]$}} &  \makecell{$0.32$ \\ \scriptsize{$[0.28, 0.37]$}} &  \makecell{$0.97$ \\ \scriptsize{$[0.97, 0.98]$}} \\ \midrule
0.06  &  \makecell{$0.72$ \\ \scriptsize{$[0.70, 0.75]$}} &  \makecell{$5.88$ \\ \scriptsize{$[5.18, 6.58]$}} &  \makecell{$0.153$ \\ \scriptsize{$[0.147, 0.158]$}} &  \makecell{$0.42$ \\ \scriptsize{$[0.37, 0.47]$}} &  \makecell{$0.91$ \\ \scriptsize{$[0.89, 0.93]$}} \\ \midrule
0.08  &  \makecell{$0.79$ \\ \scriptsize{$[0.76, 0.82]$}} &  \makecell{$4.48$ \\ \scriptsize{$[3.69, 5.24]$}} &  \makecell{$0.166$ \\ \scriptsize{$[0.161, 0.172]$}} &  \makecell{$0.53$ \\ \scriptsize{$[0.47, 0.59]$}} &  \makecell{$0.90$ \\ \scriptsize{$[0.88, 0.92]$}} \\ \midrule
0.10  &  \makecell{$0.85$ \\ \scriptsize{$[0.82, 0.88]$}} &  \makecell{$3.40$ \\ \scriptsize{$[2.58, 4.21]$}} &  \makecell{$0.179$ \\ \scriptsize{$[0.173, 0.185]$}} &  \makecell{$0.65$ \\ \scriptsize{$[0.59, 0.72]$}} &  \makecell{$0.88$ \\ \scriptsize{$[0.85, 0.91]$}} \\ \midrule
0.12  &  \makecell{$0.87$ \\ \scriptsize{$[0.84, 0.89]$}} &  \makecell{$2.93$ \\ \scriptsize{$[2.21, 3.69]$}} &  \makecell{$0.182$ \\ \scriptsize{$[0.176, 0.187]$}} &  \makecell{$0.69$ \\ \scriptsize{$[0.63, 0.75]$}} &  \makecell{$0.89$ \\ \scriptsize{$[0.87, 0.91]$}} \\ \midrule
0.14  &  \makecell{$0.87$ \\ \scriptsize{$[0.84, 0.89]$}} &  \makecell{$3.57$ \\ \scriptsize{$[2.77, 4.42]$}} &  \makecell{$0.182$ \\ \scriptsize{$[0.176, 0.188]$}} &  \makecell{$0.68$ \\ \scriptsize{$[0.61, 0.74]$}} &  \makecell{$0.88$ \\ \scriptsize{$[0.85, 0.90]$}} \\ \midrule
0.16  &  \makecell{$0.86$ \\ \scriptsize{$[0.83, 0.89]$}} &  \makecell{$3.81$ \\ \scriptsize{$[3.02, 4.65]$}} &  \makecell{$0.178$ \\ \scriptsize{$[0.171, 0.184]$}} &  \makecell{$0.66$ \\ \scriptsize{$[0.59, 0.73]$}} &  \makecell{$0.87$ \\ \scriptsize{$[0.85, 0.90]$}} \\ \midrule
0.18  &  \makecell{$0.86$ \\ \scriptsize{$[0.84, 0.89]$}} &  \makecell{$3.95$ \\ \scriptsize{$[3.13, 4.82]$}} &  \makecell{$0.178$ \\ \scriptsize{$[0.172, 0.184]$}} &  \makecell{$0.66$ \\ \scriptsize{$[0.59, 0.73]$}} &  \makecell{$0.87$ \\ \scriptsize{$[0.85, 0.89]$}} \\ \midrule
0.20  &  \makecell{$0.85$ \\ \scriptsize{$[0.82, 0.87]$}} &  \makecell{$4.75$ \\ \scriptsize{$[3.84, 5.71]$}} &  \makecell{$0.176$ \\ \scriptsize{$[0.170, 0.182]$}} &  \makecell{$0.62$ \\ \scriptsize{$[0.55, 0.69]$}} &  \makecell{$0.85$ \\ \scriptsize{$[0.83, 0.88]$}} \\ \midrule
0.22  &  \makecell{$0.83$ \\ \scriptsize{$[0.81, 0.86]$}} &  \makecell{$4.99$ \\ \scriptsize{$[4.12, 5.88]$}} &  \makecell{$0.175$ \\ \scriptsize{$[0.168, 0.181]$}} &  \makecell{$0.62$ \\ \scriptsize{$[0.56, 0.69]$}} &  \makecell{$0.83$ \\ \scriptsize{$[0.81, 0.86]$}} \\ \midrule
0.24  &  \makecell{$0.83$ \\ \scriptsize{$[0.81, 0.86]$}} &  \makecell{$4.92$ \\ \scriptsize{$[4.06, 5.81]$}} &  \makecell{$0.176$ \\ \scriptsize{$[0.169, 0.182]$}} &  \makecell{$0.64$ \\ \scriptsize{$[0.58, 0.71]$}} &  \makecell{$0.82$ \\ \scriptsize{$[0.80, 0.85]$}} \\ \midrule
0.26  &  \makecell{$0.83$ \\ \scriptsize{$[0.80, 0.85]$}} &  \makecell{$5.41$ \\ \scriptsize{$[4.53, 6.30]$}} &  \makecell{$0.174$ \\ \scriptsize{$[0.167, 0.180]$}} &  \makecell{$0.64$ \\ \scriptsize{$[0.58, 0.70]$}} &  \makecell{$0.80$ \\ \scriptsize{$[0.77, 0.83]$}} \\ \midrule
0.28  &  \makecell{$0.81$ \\ \scriptsize{$[0.78, 0.83]$}} &  \makecell{$5.73$ \\ \scriptsize{$[4.93, 6.58]$}} &  \makecell{$0.168$ \\ \scriptsize{$[0.161, 0.174]$}} &  \makecell{$0.63$ \\ \scriptsize{$[0.57, 0.70]$}} &  \makecell{$0.77$ \\ \scriptsize{$[0.74, 0.80]$}} \\ \midrule
0.30  &  \makecell{$0.81$ \\ \scriptsize{$[0.79, 0.83]$}} &  \makecell{$5.98$ \\ \scriptsize{$[5.15, 6.83]$}} &  \makecell{$0.170$ \\ \scriptsize{$[0.164, 0.176]$}} &  \makecell{$0.64$ \\ \scriptsize{$[0.58, 0.71]$}} &  \makecell{$0.75$ \\ \scriptsize{$[0.72, 0.78]$}} \\ 
\bottomrule
\end{tabular}

    \caption{Scramble with tile 32. Results for the metrics for selected noise levels. Shown in square brackets are bootstrapped $95\%$-confidence intervals.
    Here \emph{topo} stands for topographic similarity, \emph{conf} for conflict count, \emph{cont} for context independence, \emph{pos} for positional disentanglement and \emph{acc} for accuracy.
    }
    \label{app:tab:scrambles_labels}
\end{table}

\clearpage
\newpage
\section{Average topographical similarity for random languages}\label{app:sec:topo}
In this section, we compute the average performance of the topographic similarity metric for a random language when a message is of length $K=2$. 
For simplicity, we assume that the feature space and the alphabet space are the same, and equal $\mathcal F=\{1, \ldots, n\}^2$. 
Then the topographic similarity for language $\ell\colon \mathcal F\to\mathcal F$ is defined as 
\[
\text{topo}(\ell) = corr(R(\rho(F_0, F_1)), R(\rho(\ell(F_0), \ell(F_1))), 
\]
where $F_0, F_1$ are uniform random variables on $\mathcal F$ and $R$ is the rank function. The random variable $\rho(F_0, F_1)$ takes values $\{0, 1, 2\}$ with probabilities $p_0, p_1, p_2$. Since $\rho(F_0, F_1)$ is discrete, there are different conventions for defining function $R$. 
Typical choices for ranks are: (i) "min-ranks" $R(0)=0, R(1)=p_0, R(2)=p_0+p_1$;
(ii) "max-ranks" $R(0)=p_0, R(1)=p_0+p_1, R(2)=1$; (iii) "average-ranks" $R(0)=p_0/2, R(1)=(p_0+p_1)/2, R(2)=(p_0+p_1+1)/2$. 
Lemma \ref{lem:app:topo} gives the formula for $\mathbb E_{\ell\sim U}[\text{topo}(\ell)]$. 

\begin{lem}\label{lem:app:topo}
Let $R$ be a rank function. Then 
\begin{align*}
\mathbb E_{\ell\sim U}\left[\text{topo}(\ell)\right] 
= 
\frac{p_0 R(0)^2 + \frac{2p_1}{n+1}R(1)^2 + \frac{p_2(n-1)}{n+1}R(2)^2 + \frac{4p_2}{n+1}R(1)R(2)- (\sum_{i=0}^2 p_iR(i))^2}{\sum_{i=0}^2p_iR(i)^2 - (\sum_{i=0}^2 p_iR(i))^2},
\end{align*}
where $U$ is a uniform distribution among all bijective $\ell\colon \mathcal F\to\mathcal F$, 
$n_0 = n^2$, $n_1=2n^2(n-1)$, $n_2=n^2(n-1)^2$, and $p_i = n_i/n^4$. 
\end{lem}

\begin{proof}
Notice that $\rho(F_0, F_1)$ and $\rho(\ell(F_0), \ell(F_1))$ have the same distribution described by $p_i$. 
Define $\alpha=R(\rho(F_0, F_1))$ and $\alpha_\ell=R(\rho(\ell(F_0), \ell(F_1)))$. 
Consequently, 
\[
\mathbb E_{\ell\sim U}\left[\text{topo}(\ell)\right] 
= \frac{\mathbb E_{\ell\sim U}\mathbb E[\alpha\alpha_\ell]-(\mathbb E[\alpha])^2}{Var(\alpha)}.
\]
Since, $\mathbb E[\alpha] = \sum_{i=0}^2p_iR(i)$ and  
$Var(\alpha) =\sum_{i=0}^2p_iR(i)^2 - (\mathbb E[\alpha])^2$, 
it remains to compute $\mathbb E_{\ell\sim U}\mathbb E[\alpha\alpha_\ell]$:
\begin{align*}
&\mathbb E_{\ell\sim U}\mathbb E[\alpha\alpha_\ell] =
\mathbb E \mathbb E_{\ell\sim U}[\alpha\alpha_\ell]
=\frac{1}{n^4}\sum_{f_0, f_1\in \mathcal F}R(\rho(f_0, f_1)) E_{\ell\sim U}[R(\rho(\ell(f_0), \ell(f_1)))]\\
&=\frac{1}{n^4}\sum_{i=0}^2\sum_{f_0, f_1, \rho(f_0, f_1)=i}R(i) \mathbb E_{\ell\sim U}[R(\rho(\ell(f_0), \ell(f_1)))]\\
&=p_0R(0)^2 + \frac{1}{n^4}\sum_{i=1}^2\sum_{f_0, f_1, \rho(f_0, f_1)=i}R(i) E_{\ell\sim U}[R(\rho(\ell(f_0), \ell(f_1)))]\\
&=p_0R(0)^2 + \frac{1}{n^4}\sum_{i=1}^2\sum_{f_0, f_1, \rho(f_0, f_1)=i}R(i) \left(R(1)\frac{2}{n+1} + R(2)\frac{n-1}{n+1}\right)\\
&=p_0R(0)^2 + \left(R(1)p_1 + R(2)p_2\right) \left(R(1)\frac{2}{n+1} + R(2)\frac{n-1}{n+1}\right).
\end{align*}
The second to last equality follows from the fact that there $n_i(n^2-2)!$ bijections $\ell$ that map $f_0\ne f_1$ to $\ell(f_0), \ell(f_1)$ such that $\rho(\ell(f_0), \ell(f_1))=i$, for $i=1,2$. 
\end{proof}

\section{Optimality of compositional communication}\label{app:sec:optimality}

When features are explicitly stated, we can write $\mathcal F=\prod_{i=1}^K \mathcal F_i$, where $\mathcal F_i$ is the space of values of $i$-th feature and $K$ is the number of features and a message length. We assume that $|\mathcal F_i|=|\mathcal A_s|$ and that $|\mathcal A_s|\ge 2$. We will assume that a language is a mapping $\ell : \mathcal F\to \mathcal A_s^K$. 
We said that the language is compositional if a change in one feature only impacts a corresponding index of the message.
Formally, we say that $\ell$ is compositional if and only if for every $k=0,\ldots, K$, and $\ff_0, \ff_1\in \mathcal F$,  $\rho(\ff_0, \ff_1) = k\iff \rho(\ell(\ff_0),\ell(\ff_1))=k$. Here $\rho$ stands for the Hamming distance.

Recall that the corrupted message corresponding to $\mathbf{f}\in\mathcal F$ is denoted by $\mathbf{f}' := \ell^{-1}(\ell(\mathbf{f})')$, and
the two loss function are defined as:
\begin{equation*}
\begin{split}
	J_1(\ell, \mathbf{f}) &= F(e_1, \ldots, e_K), \quad e_j := \mathbb P(\mathbf{f}'_j\ne \mathbf{f}_j),\\
    J_2(\ell, \mathbf{f}) &= \mathbb E[H(\rho(\ff', \ff))],
\end{split}
\end{equation*}
where $F$ is a non-negative function such that $F(\varepsilon, \ldots, \varepsilon)=0$, and $H$ is a non-negative,  increasing function. 

\begin{figure}
    \centering
\begin{tikzpicture}[scale=2]
    \draw (-1.5,0) circle (1cm);
    \draw (1.5,0) circle (1cm);

    \node[label=$\ff_0$] (f0) at (-1.5, 0) {};
    \node[label=$\ff_1$] (f1) at (-1, .866) {};
    \node[label=$\ff_2$] (f2) at (-.8, .1) {};
    \node[label=$\ss_0$] (s0) at (1.5, 0) {};
    \node[label=$\ss_2$] (s2) at (1, .866) {};
    \node[label=$\ss_1$] (s1) at (.8, .1) {};

    \fill[black] (f0) circle (1pt);
    \fill[black] (f1) circle (1pt);
    \fill[black] (f2) circle (1pt);
    \fill[black] (s0) circle (1pt);
    \fill[black] (s1) circle (1pt);
    \fill[black] (s2) circle (1pt);
    
    \node at (-1.5, -1.2) {$\{\ff\in \mathcal F:\rho(\ff, \ff_0) = k^*\}$};
    \node at (1.5, -1.2) {$\{\ss\in \mathcal A_s^K:\rho(\ss, \ss_0) = k^*\}$};

    \draw[-{>[scale=2, width=3]}] (f0) to [bend right=45]  node[midway, above] {$\ell$} (s0);
    \draw[-{>[scale=2, width=3]}] (f1) to [bend right=-25] node[near start, above] {$\ell$} (s1);
    \draw[-{>[scale=2, width=3]}] (s2) to [bend right=25] node[near start, above] {$\ell^{-1}$} (f2);
    \draw[-{>[scale=2, width=3]}, dashed] (f1) to [bend right=-45] node[midway, above] {$\tilde{\ell}$} (s2);
    \draw[-{>[scale=2, width=3]}, dashed] (f2) to [bend right=25] node[midway, above] {$\tilde{\ell}$}  (s1);
    \draw (s0) -- node[near end, right] {$k^*$} (s2);
    \draw (f0) -- node[near end, left] {$k^*$} (f1);
\end{tikzpicture}
    \caption{Illustration of the construction given in the proof of Theorem \ref{app:thm:compostionality}. 
    Two circles represent the ball of radius $k^*$ given by $\rho$ (in $\mathcal F$ and $\mathcal A_s^K$, respectively). 
    The solid arrow represents the language $\ell$, and the dashed arrows show the swap that is performed when defining $\tilde{\ell}$. $\tilde{\ell}$ improves upon $\ell$ by reducing the number of compositionality violations.  
    }
    \label{fig:app:tikz}
\end{figure}
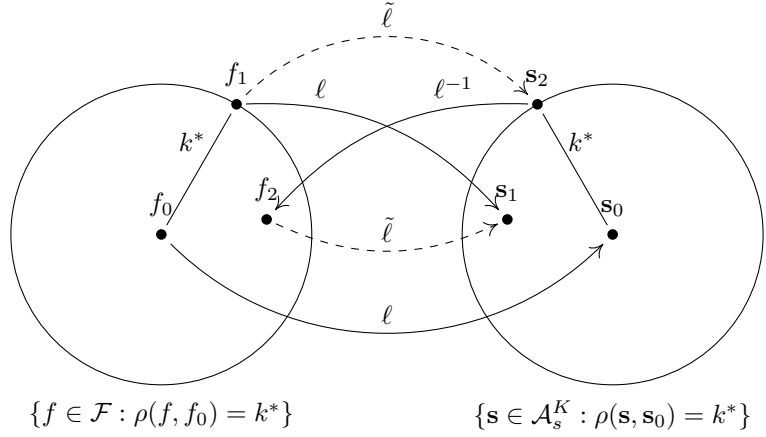

The noisy channel transforms a message $\mathbf{s}\in\mathcal A^K_s$, into a corrupted message  
$\mathbf{s'}\in\mathcal A_s^K$, by replacing each symbol with a different symbol, independently and with probability $\varepsilon \in (0,1)$.
More formally, the conditional distribution of $\ssp$, given $\mathbf{s}$, is expressed by the following formula: 
\begin{equation}\label{eq:app:prob}
\mathbb P(\mathbf{s'}=\hat{\mathbf{s}}|\mathbf{s})=(1-\varepsilon)^{K-\rho(\hat{\mathbf{s}}, \mathbf{s})}\varepsilon^{\rho(\mathbf{\hat{s}}, \mathbf{s})}\left(\frac{1}{|\mathcal A_s|-1}\right)^{\rho(\mathbf{\hat{s}}, \mathbf{s})}, \qquad \text{ for any } \hat{\ss}\in \mathcal A_s^K.
\end{equation}
Indeed, there has to be $\rho(\hat{\mathbf{s}}, \mathbf{s})$ noise flips (hence the first two terms on the right-hand side \eqref{eq:app:prob}), and each flip changes one coordinate of $\mathbf{s}$ to the corresponding coordinate of $\hat{\mathbf{s}}$ with probability $1/(|\mathcal A_s|-1)$. 

The following result is a more detailed version of Theorem \ref{thm:optimality}.

\begin{thm}\label{app:thm:compostionality}
Assume $\mathcal F=\mathcal X$, and 
$\varepsilon < (|\mathcal A_s|-1)/|\mathcal A_s|$. A compositional language minimizes $J_1$ and $J_2$ over all one-to-one languages. Furthermore, for arbitrary $\ff\in\mathcal F$, $\min J_2(\ff, \ell) = \mathbb E[H(B_\varepsilon)]$, where $B_\varepsilon$ is a Binomial distribution with success probability $\varepsilon$.
Moreover, $\ell$ is optimal for $J_2$ if and only if $\ell$ is compositional.
\end{thm}

Notice that, since the assertion holds for arbitrary $\ff\in\mathcal F$, the language $\ell$ is compositional if and only if it is optimal for $\mathbb E_{\ff\sim\nu}[J_2(\ff, \ell)]$, where $\nu\in\mathcal P(\mathcal F)$ is any distribution such that $supp(\nu)=\mathcal F$.

\begin{proof}
We start by proving the claim for $J_2$. 
Fix $\ff_0\in\mathcal F$ and denote $\ss_0=\ell(\ff_0)$. 
Suppose that $\ell$ is not compositional. 
Then, there exists $k>0, \ff_1\in\mathcal F, \ss_2\in\mathcal A_s^K$ such that $\rho(\ff_0, \ff_1)=k$, $\rho(\ss_0, \ss_2)=k$, $\rho(\ff_0, \ell^{-1}(\ss_2))\ne k$, and $\rho(\ss_0, \ell(\ff_1))\ne k$.
Let $k^*$ be the biggest among mentioned $k$'s and denote $\ss_1=\ell(\ff_1), \ff_2=\ell^{-1}(\ss_2)$. %
By the definition of $k^*$, we have
\[
\rho(\ss_0, \ss_1)< k^*, \qquad \rho(\ff_0, \ff_2)<k^*.
\]
Since $\varepsilon<(|\mathcal A_s|-1)/\mathcal |A_s|$, the probability $\mathbb P(\ss'=\hat{\ss}|\ss)$  is a decreasing function of $\rho(\hat{\ss},\ss)$ (see equation \eqref{eq:app:prob}). It follows that 
\[
\mathbb P(\ss'=\ss_1|\ss_0) > \mathbb P(\ss'=\ss_2|\ss_0).
\]
We construct a new language $\tilde{\ell}$ (see Figure \ref{fig:app:tikz})
\[
\tilde{\ell}(\ff)=\begin{cases}
\ss_{1} & \ff=\ff_2,\\
\ss_{2} & \ff=\ff_{1},\\
\ell(\ff) & \text{otherwise}.
\end{cases}
\]
As a result, 
\begin{align*}
\rho\left(\ff_{0},\tilde{\ell}^{-1}(\ss_{1})\right)-\rho\left(\ff_{0},\ell^{-1}(\ss_{1})\right)
&=\rho\left(\ff_{0},\ff_2\right)-\rho\left(\ff_{0},\ff_1\right)<0,\\
\rho\left(\ff_{0},\tilde{\ell}^{-1}(\ss_{2})\right)-\rho\left(\ff_{0},\ell^{-1}(\ss_{2})\right)
&=\rho\left(\ff_{0},\ff_1\right)-\rho\left(\ff_{0},\ff_2\right)>0.
\end{align*}
Putting things together and using the fact that $H$ is increasing, we get
\begin{align*}
&J_2(\ff_0, \tilde{\ell}) - J_2(\ff_0, \ell) 
= \mathbb E\left[H\left(\rho(\tilde{\ell}^{-1}(\tilde{\ell}(\ff_0)'), \ff_0)\right)\right] - 
 \mathbb E\left[H\left(\rho(\ell^{-1}(\ell(\ff_0)'), \ff_0)\right)\right]\\
&=
\mathbb P(\ssp=\ss_1|\ss_0))\left\{H\left(\rho\left(\ff_{0},\tilde{\ell}^{-1}(\ss_{1})\right)\right)-H\left(\rho\left(\ff_{0},\ell^{-1}(\ss_{1})\right)\right)\right\}\\
&+\mathbb P(\ssp=\ss_2|\ss_0)\left\{H\left(\rho\left(\ff_{0},\tilde{\ell}^{-1}(\ss_{2})\right)\right)-H\left(\rho\left(\ff_{0},\ell^{-1}(\ss_{2})\right)\right)\right\}\\
&=\left\{H\left(\rho(\ff_0,\ff_1)\right)-H\left(\rho(\ff_0,\ff_2)\right)\right\}
\left(\mathbb P(\ssp=\ss_2|\ss_0)-\mathbb P(\ssp=\ss_1|\ss_0)\right)<0.
\end{align*}
Consequently, for any non-compositional $\ell$ we can strictly decrease its loss value $J_2(\ff_0, \ell)$. Since the loss is nonnegative and there is a finite number of languages, optimal $\ell$ has to be compositional. 
For compositional language $\ell$, 
\begin{equation}\label{eq:pequalsepsilon}
\begin{split}
J_2(\ff_0, \ell) &= \mathbb E\left[H(\rho(\ss_0', \ss_0))\right]=
\sum_{\hat{\ss}\in\mathcal A_s^K} H(\rho(\hat{\ss},\ss_0)) \mathbb P(\ss_0'=\hat{\ss}|\ss_0)\\
&=\sum_{k=0}^{K}\sum_{\stackrel{\hat{\ss}\in\mathcal A_s^K}{\rho(\hat{\ss}, \ss_0)=k}} H(k)\mathbb{P}\left(\ss'=\hat{\ss}|\ss_{0}\right)
=\sum_{k=0}^K H(k){K \choose k}\varepsilon^k (1-\varepsilon)^{K-k}= \mathbb E[H(B_\varepsilon)],
\end{split}
\end{equation}
where $B_\varepsilon$ is a binomial random variable with success probability $\varepsilon$. 
This ends the proof for $J_2$.

To prove the assertion for $J_1$ it is enough to notice that a compositional language satisfies 
$\mathbb P(\mathbf{f}'_j\ne \mathbf{f}_j)=\varepsilon$. This follows by \eqref{eq:pequalsepsilon} and the fact that, by symmetry, 
$\mathbb P(\mathbf{f}'_j\ne \mathbf{f}_j)$ does not depend on $j$ (for compositional $\ell$).
\end{proof}

As a corollary, we see that if $H(x)=x/K$, $J_2(\ff_0, \ell)=\varepsilon$ for any compositional language $\ell$.

Notice that Theorem \ref{thm:all_is_good} follows from the fact that any permutation $\pi:\mathcal F\to\mathcal F$ is a bijection. 

\end{document}